\documentclass{article}

\newcommand{\Ac}{\mathcal{A}}

\newcommand{\Dc}{\mathcal{D}}

\newcommand{\Gc}{\mathcal{G}}

\newcommand{\Lc}{\mathcal{L}}
\newcommand{\Mc}{\mathcal{M}}
\newcommand{\Nc}{\mathcal{N}}

\newcommand{\Sc}{\mathcal{S}}

\newcommand{\Vc}{\mathcal{V}}

\newcommand{\E}{\mathbb{E}}
\newcommand{\R}{\mathbb{R}}

\newcommand{\Prb}{\mathbb{P}}

\newcommand{\T}{^\top}
\usepackage{microtype}
\usepackage{graphicx}
\usepackage{subcaption}
\usepackage{booktabs} 
\usepackage{multirow}
\usepackage{bbm}
\usepackage[table]{xcolor}   

\usepackage{enumitem}

\usepackage{hyperref}

\usepackage[accepted]{icml2025}

\usepackage{tcolorbox}
\usepackage{amsmath}
\usepackage{amssymb}
\usepackage{mathtools}
\usepackage{amsthm}
\usepackage{bm}
\usepackage[capitalize,noabbrev]{cleveref}
\tcbset{
    examplebox/.style={
        colback=gray!10,
        colframe=black,
        boxrule=0.5mm,
        arc=2mm,
        outer arc=2mm,
        fonttitle=\bfseries,
        coltitle=black,
        width=\textwidth,
        sharp corners=southwest,
        before skip=10pt,
        after skip=10pt,
    }
}

\newcommand{\attn}{\text{Attn}}

\newcommand{\sft}{\bar{\sigma}}

\newcommand{\Wlin}{W_{\text{lin}}}
\newcommand{\Wkq}{W_{\text{KQ}}}
\newcommand{\Wov}{W_{\text{OV}}}

\newcommand{\phl}{$\text{PH3}_l$\xspace}
\newcommand{\phs}{$\text{PH3}_s$\xspace}
\newcommand{\superpo}{superposition of contextual information and parametric memory\xspace}

\theoremstyle{plain}
\newtheorem{theorem}{Theorem}[section]
\newtheorem{proposition}[theorem]{Proposition}
\newtheorem{lemma}[theorem]{Lemma}
\newtheorem{corollary}[theorem]{Corollary}
\theoremstyle{definition}

\newtheorem{assumption}[theorem]{Assumption}
\theoremstyle{remark}

\usepackage{xspace}
\newcommand{\lgtnote}[1]{#1}

\definecolor{lightgreen}{RGB}{190, 230, 190}
\definecolor{lightred}{RGB}{255, 200, 200}

\newcommand{\jrt}{\textsc{JuICE}\xspace}
\newcommand{\jro}{\textsc{JuNe}\xspace}
\newcommand{\sub}{\{subject\}\xspace}

\newcommand{\act}{\{context answer\}\xspace}

\usepackage[textsize=tiny]{todonotes}

\begin{document}

\twocolumn[

\icmltitle{Taming Knowledge Conflicts in Language Models}




\icmlsetsymbol{equal}{*}

\begin{icmlauthorlist}
\icmlauthor{Gaotang Li}{uiuc}
\icmlauthor{Yuzhong Chen}{visa}
\icmlauthor{Hanghang Tong}{uiuc}
\end{icmlauthorlist}

\icmlaffiliation{visa}{Visa Research}
\icmlaffiliation{uiuc}{University of Illinois Urbana-Champaign}

\icmlcorrespondingauthor{Gaotang Li}{gaotang3@illinois.edu}
\icmlcorrespondingauthor{Hanghang Tong}{htong@illinois.edu}

\icmlkeywords{Machine Learning}

\vskip 0.3in
]



\printAffiliationsAndNotice{}  

\begin{abstract}

Language Models (LMs) often encounter knowledge conflicts when parametric memory contradicts contextual knowledge. 
Previous works attribute this conflict to the interplay between ``memory heads'' and ``context heads'', attention heads assumed to promote either memory or context exclusively.
In this study, we go beyond this fundamental assumption by uncovering a critical phenomenon we term the \emph{\superpo}, where highly influential attention heads simultaneously contribute to both memory and context.
Building upon this insight, we propose Just Run Twice (\jrt), a test-time attention intervention method that steers 
LMs toward either parametric beliefs or contextual knowledge without requiring fine-tuning.
\jrt identifies a set of reliable attention heads and leverages a dual-run approach to mitigate the superposition effects. 
Extensive experiments across 11 datasets and 6 model architectures demonstrate that \jrt sets the new state-of-the-art performance and robust generalization, achieving significant and consistent improvement 
across different domains under various conflict types.
Finally, we theoretically analyze knowledge conflict and the \superpo in attention heads, which further elucidates the effectiveness of \jrt in these settings. Our code is available at \url{https://github.com/GaotangLi/JUICE}.

\end{abstract}
\addtocontents{toc}{\protect\setcounter{tocdepth}{-1}}

\section{Introduction}
\label{sec:intro}
Language Models (LMs) store vast amounts of information during pretraining as 
parametric knowlege. During inference, they leverage this parametric memory alongside 
the provided context to generate the next token. However, conflicts can arise 
when parametric memory contradicts contextual information---a phenomenon known as 
knowledge conflict~\citep{xu2024knowledge}. In such cases, the model may become uncertain about which source of knowledge to trust.
These conflicts are particularly prevalent in real-world applications, especially in context-heavy Large Language Models (LLMs)
systems like retrieval-augmented generation (RAG)~\citep{gao2023retrieval}, LLM agents~\citep{xi2023rise}, and 
tool-augmented LLMs~\citep{qu2024tool}. 
Depending on the application, user may require an LLM to either remain faithful to its parametric memory or prioritize contextual reliance for accurate and reliable outputs.

\begin{figure}[t]
    \centering
    \includegraphics[width=0.95\linewidth]{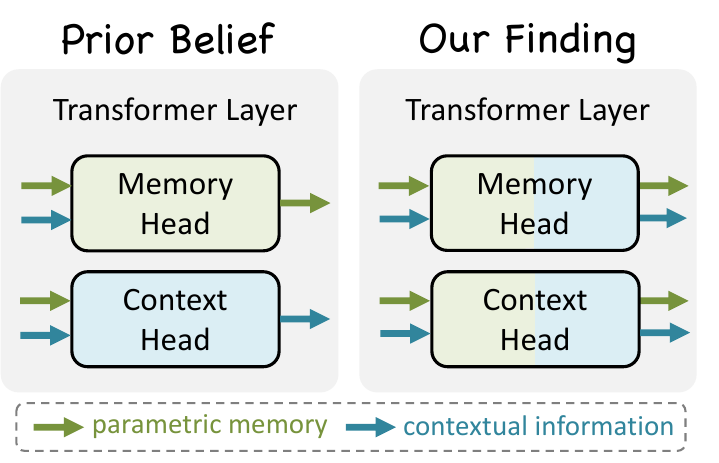}
    \vspace{-3mm}
  \caption{Our finding goes beyond the prior notion of exclusive ``memory head'' and ``context head'', where we show that memory and contexts are encoded in attention heads in superposition. 
  }
  \label{fig:teaser}
  \vspace{-4mm}
\end{figure}

Prior works have explored the behavior of LMs under knowledge conflicts, 
either by treating the model as an oracle to analyze how different contexts influence its predictions~\citep{xie2023adaptive} 
or by treating the context as an oracle to evaluate how effectively the model follows it~\citep{longpre2021entity}. 
While these studies provide valuable insights into knowledge conflicts, the intrinsic mechanisms underlying these conflicts and corresponding 
mitigation strategies largely remain unexplored.
Some studies have taken important steps to characterize~\citep{yu2023characterizing} and intervene~\citep{jin2024cutting} in knowledge conflicts, 
primarily focusing on a single conflict type (\emph{e.g.}, substitution-based conflicts). 
While pioneering, these efforts leave opportunities for more comprehensive understanding of diverse conflict types and the development of fine-grained approaches to address knowledge conflicts.
In addition, much of the existing literature predominantly adopts a single-sided perspective on knowledge conflict, focusing on enhancing contextual 
reliance and addressing issues commonly referred to as ``RAG hallucination''~\citep{goyal2024context,huang2023survey,shi2023trusting}.
In contrast, we advocate for a unified method capable of flexibly steering the model toward either parametric or contextual knowledge, offering broader utility.

\begin{figure}[t]
    \centering
    \includegraphics[width=\linewidth]{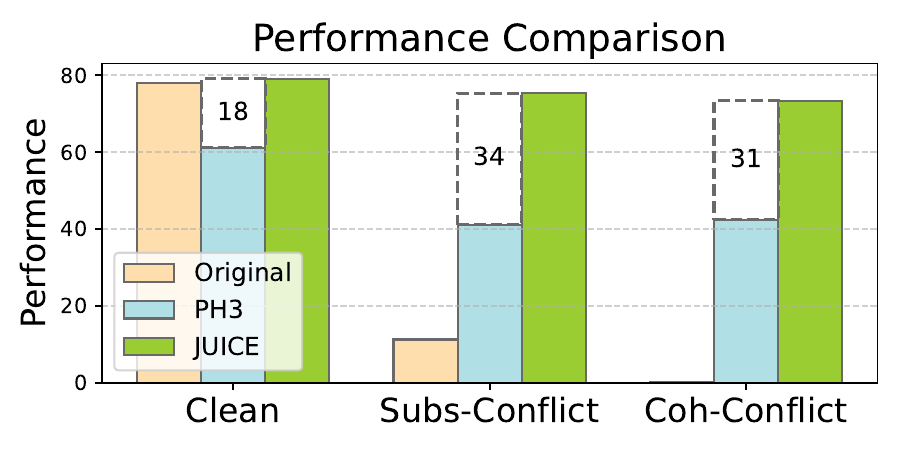}
    \vspace{-6mm}
  \caption{Performance of different methods with Gemma-2b under various conflict types. \jrt achieves consistently high performance in facing challenging knowledge conflicts. 
  }
  \label{fig:intro}
  \vspace{-4mm}
\end{figure}

In this paper, we begin by treating LMs as an oracle and considering the setting of factual recall, 
a task requiring pure memorization. 
We then treat contexts as providing misleading information~\citep{shi2023large} and systematically explore various types of knowledge conflicts over diverse domains, including sentence-level (substitution), and paragraph-level (coherent) conflicts~(Sec.~\ref{sec:experiment_setup}), to uncover their underlying mechanisms and design effective intervention strategies. 
Starting with empirical analysis,
our findings go beyond the hypothesis posited in~\citep{jin2024cutting} that model components exclusively contribute to either parametric or contextual knowledge, uncovering the phenomenon of ``\textbf{\superpo}'' (CP superposition), as shown in Fig.~\ref{fig:teaser}. \lgtnote{We revealed the inconsistent behaviors of model components under different degrees of knowledge conflicts and the counteracting effects of multiple individually effective interventions.}

Building on these insights, we propose Just Run Twice (\jrt), a simple yet effective method for steering LMs towards either parametric or contextual knowledge without finetuning. 
\jrt operates in two stages: (1) a \emph{head identification} stage, where two sets of attention heads that yield consistent improvements with positive or negative scaling are identified using a minimal number of samples, 
and (2) an \emph{dual-run inference} stage, where the model runs twice: first saving the outputs of the identified heads, and then using scaled versions of these saved outputs to intervene during the second run. 
Intuitively, this approach ensures that the identified components are consistently effective, mitigating the superposition effects, and therefore provide 
more accurate steering directions through residual head activations.

We evaluate \jrt in two distinct settings: enhancing parametric beliefs and enhancing contextual reliance. For the first setting, we use six factual association datasets covering diverse domains, each tested under three levels of knowledge conflict. In the second setting, we evaluate five datasets spanning diverse fields and formats, including open-domain question answering and sentence completion.
 Extensive experimental results demonstrate the consistent state-of-the-art performance 
of \jrt. 
Fig.~\ref{fig:intro} illustrates the strong performance of \jrt under the Gemma-2b model, with detailed results provided in Tab.~\ref{tab:main_intervention_unfiltered}. 
We also show the robustness of \jrt against key hyperparameters and paraphrased input. 

Finally, we analyze our empirical observations from a theoretical perspective, conceptualizing knowledge conflict as the result of conflicting tasks at inference, which arise from distinct tasks during training.
In a succinct
setup, we demonstrate the existence of attention heads that simultaneously contribute to both parametric and contextual knowledge and show how standard training encourages the formation of such heads.
We further provide theoretic justifications for the effectiveness of \jrt under these settings.

Our main contributions can be summarize as follows:
\begin{itemize}
    \setlength\itemsep{0em}
    \item {\bf Problem.} We conduct a systematic and principled study of knowledge conflicts in LMs, considering both parametric and contextual perspectives and covering various types of datasets over diverse domains.

    \item {\bf Mechanism.} We reveal the limitations of naive intervention methods by uncovering a critical phenomenon we term the ``\superpo'', where the relative role of a model component in parametric versus contextual knowledge is not exclusive.
    \item {\bf Algorithm.} We propose \jrt, a simple yet effective method to steer an LM toward parametric or contextual knowledge 
    without finetuning, leveraging a dual-run approach to mitigate the superposition effects. 
    \item {\bf Experiment.} Through extensive experiments across 11 datasets and 6 architectures, we set the new state-of-the-art performance 
    and robust generalization, achieving significant and consistent improvements.
    \item {\bf Theory.} We provide a theoretical analysis of knowledge conflicts, conceptualizing the \superpo. This analysis further justifies the effectiveness of \jrt under these conditions.
\end{itemize}

\section{Problem Setup}
\label{sec:experiment_setup}

In this paper, we study how language models respond to varying degrees 
of knowledge conflict and propose methods to regulate these behaviors. We 
identify two complementary perspectives on knowledge conflict: (1) 
when the input context is irrelevant or potentially misleading, 
we treat the LM as an \emph{oracle},
aiming to enhance its reliance on 
\emph{parametric beliefs}; (2) when the input context is accurate, but 
the LM's prior knowledge may be outdated or incorrect, we aim to increase 
the model's dependence on \emph{contextual knowledge}. 
Both perspectives hold intrinsic value and merit further investigation.

\subsection{Parametric Datasets}
\label{subsec:para_setup}

In this setup, we treat the input context as potentially misleading information and the language model as an \emph{oracle}. 
For our study, we carefully \emph{curate six datasets encompassing distinct types of knowledge conflicts} in factual recalls. Below, we detail the specific design choices differing from prior studies and the underlying rationales:

\textbf{Diverse Factual Domains:} We create six datasets spanning various domains of factual knowledge: World Capital, Athlete Sport, Book Author, Official Language, Company Headquarter, and Company Founder. This setting will allow us to investigate the transferrability across unrelated domains of intervention methods, a critical aspect that is missing in the prior work~\cite{jin2024cutting,yu2023characterizing}.

\textbf{Sentence-level Conflict (Substitution-based):} This is the exclusive approach adopted in prior works~\citep{yu2023characterizing,jin2024cutting}. 
    A typical input takes the form (\emph{e.g.,} ``The name of the capital city of \{\(s\)\} is \{\(a_c\)\}. The name of the capital city of \{\(s\)\} is''), where \(a_c\) represents the substituted contextual answer that conflicts with the parametric answer \(a_p\). 
    In our experiment, we aim to enhance the model's ability to output \(a_p\), despite the conflicting presence of \(a_c\).

\textbf{Paragraph-level Conflict (Coherent Counterfactual):}
Recent work~\citep{xie2023adaptive} demonstrates that language models rely more on context when it is coherent.
    In this scenario, the context extends beyond a single substitution, reinforced by coherent and persuasive evidence, often generated by advanced models like GPT-4.    
    This presents a highly challenging case, as models almost inevitably output the contextual answer \(a_c\) over the parametric answer \(a_p\).
    In our experiment, we focus on enhancing the model's ability to output \(a_p\), despite these difficult conditions.

There is also a trivial type of knowledge conflict: when no conflict is present, in which case we still expect the model to respond faithfully.
\textbf{Detailed examples} are provided in Appen.~\ref{appen:dataset_detail}.
Importantly, different from \cite{xie2023adaptive}, which focuses solely on altering the model's predictions regardless of their correctness, we explicitly ensure that conflicting contexts include factually incorrect answers. 
For evaluation, we primarily rely on the exact match (accuracy) metric with respect to the \textbf{factually correct} answer. 
Our curated dataset is available at  \url{https://huggingface.co/datasets/gaotang/ParaConfilct}.

\subsection{Contextual Datasets}
\label{subsec:context_setup}

In this setup, we treat the input context as the desired target and consider the prior knowledge of the language model as an unreliable source of information.
This approach enables a more unified and versatile evaluation of baseline methods.

Since this setup has been extensively studied, we adopt the dataset choice of a seminar work~\citep{shi2023trusting} by using two context-oriented knowledge conflict benchmarks: Memo-Trap~\citep{memotrap} and NQ-Swap~\citep{longpre2021entity}.
The details of these datasets can be found in Appen.~\ref{appen:exp_detail}.
We evaluate performance using exact match 
(accuracy) with respect to the contextual answer. 

\subsection{Models}
\label{subsec:setup_model}

We benchmark our studies using six existing open-sourced base 
language models: Gemma-2b~\citep{team2024gemma}, Llama2-7B~\citep{touvron2023llama}, 
Llama3-8B~\citep{dubey2024llama}, Phi2-2.7b~\citep{javaheripi2023phi}, StableLm2-1.6b~\citep{bellagente2024stable}, and Olmo-7b~\citep{groeneveld2024olmo}. We conduct our analysis in Sec.~\ref{sec:method} mainly using 
Gemma and evaluate the effectiveness of the intervention methods using 
all backbone models.

\section{Interpreting and Resolving Knowledge Conflicts}
\label{sec:method}
\begin{figure*}[!ht]
      \centering
      \includegraphics[width=\textwidth]{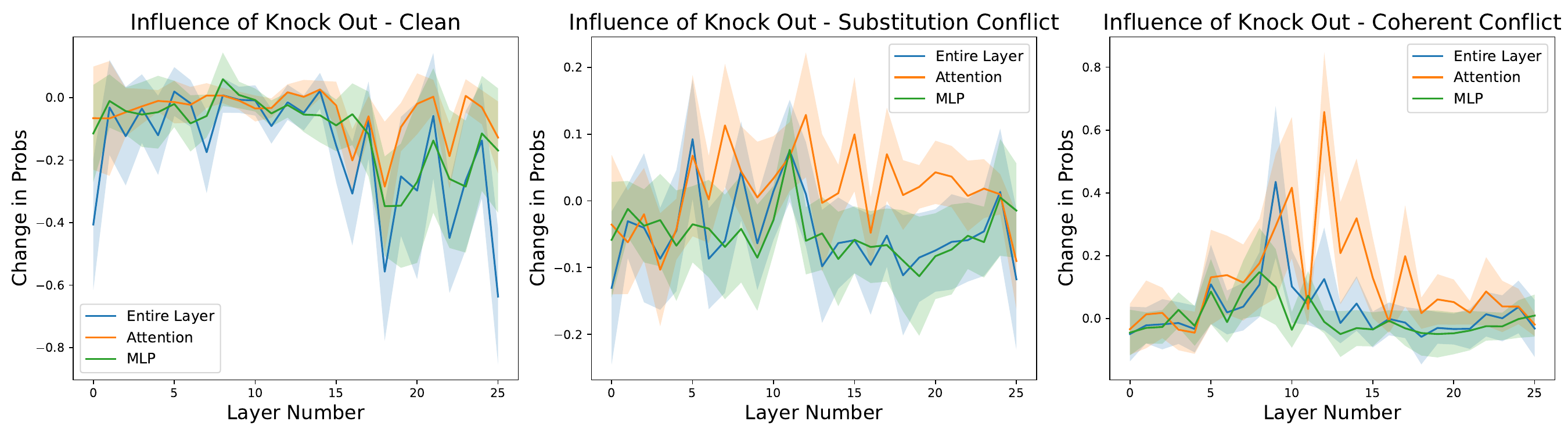}
    \caption{Influence of Knock Out (Zero Out) Model Components in changing 
    the probability of outputting the parametric answer tokens ($a_p$) on the 
    \texttt{World Capital} dataset. Three different 
    scenarios are considered: clean inputs, substitution conflict inputs, and 
    coherent conflict inputs. 
    We find that (1) removing (nearly) all components leads to decreases in 
    probability of outputting $a_p$ in clean prompts, (2) removing components leads to 
    both increase and decrease in outputting $a_p$ in substitution conflict 
    prompts, and (3) removing 
    (nearly) all components leads to increases in probability of outputting $a_p$ in coherent conflict prompts.
    }
    \label{fig:knock_out}
  \end{figure*}

In this section, we analyze how the internal structure of language models (LMs) influences their parametric versus contextual tendencies through causal analysis. We quantify these tendencies by measuring the expected change in the probability of the output token (parametric versus contextual) when perturbations are applied to specific model components. These perturbations are implemented by scaling the activation outputs. Formally, given a distribution over input triplets $(X, y_p, y_c)$, where $X \coloneqq \{x_i\}_{i=1}^n$ is the input prompt set,
encompassing various conflicting forms (\emph{e.g.}, clean input, substitution conflicts, and coherent conflicts), $y_p$ and $y_c$ represent the parametric and contextual answers, respectively, we measure:
\begin{equation}
\label{equa:intervention}
    \E_{(x, y)} \left[ \Prb \left( y\ | x, do(\Mc^{(i)} = \alpha \Mc^{(i)}) \right) 
- \Prb \left( y | x \right) \right].
\end{equation} 
Here, $\Mc^{(i)}$ refers to a specific model component with index $i$, and $y$ is set to either $y_p$ or $y_c$ upon our needs. 
While $(x, y)$ can be drawn from an arbitrary distribution, we use Gemma and World Capital as a concrete example in this section.

Previous works analyzing model internals typically adhere to two ``locate-and-edit" principles~\citep{xu2024knowledge}:

\vspace{-1mm}
\begin{itemize}

    \setlength\itemsep{0em}
    
    \item Identify a circuit (specific model components) that is exclusively responsible for a particular functionality.

    \item Apply targeted interventions to these circuits to achieve the desired control or behavior.

\end{itemize}
\vspace{-1mm}

In our motivating experiments, we demonstrate the need for additional criteria when performing interventions to address the complexities of model internals and knowledge conflicts.

\subsection{Analysis}

\paragraph{Observation 1: Inconsistent Behaviors of Model Components Under Different Degrees of Knowledge Conflict.} In our first set of experiments, we examine how model components exhibit significantly different functionalities when faced with varying degrees of knowledge conflict. We set \(\Mc^{(i)}\) to represent either the entire MLP, attention module, or both within layer $i$. 
For the intervention method, we fix it to be knocking out (\emph{i.e.,} zero-ablating). The goal is to promote parametric knowledge, setting $y = y_p$ in Eq.~\ref{equa:intervention}.  Fig.~\ref{fig:knock_out} illustrates these findings, revealing the following trends:
(1) removing (nearly) all components \emph{decreases} the probability of outputting parametric answers for clean prompts; (2) removing components leads to both increase and decrease in outputting parametric answers for substitution conflicts; and (3) removing (nearly) all components \emph{increases} the probability of outputting parametric answers for coherent conflict prompts.
Quantitatively, the number of components yielding consistent parametric gains across all three conflict types is 0 for the entire layer, 1 for the MLP module, and 6 for the Attention module (out of 26 layers in Gemma). These results suggest that \emph{the same model component may exhibit different influences on parametric and contextual knowledge depending on residual streams received from prior layers}.

Prior work~\citep{jin2024cutting} introduces the notion of ``memory heads" and ``context heads", positing that there are attention heads exclusively responsible for promoting parametric or contextual knowledge. Specifically, promoting contextual knowledge involves knocking out parametric heads, and vice versa. While this approach achieves success in single-typed conflicts, we find its limitations when extended to multiple kinds of conflicts. Tab.~\ref{tab:superposition_1} ranks the top-4 memory heads based on their effectiveness in substitution conflicts and evaluates their influence in coherent conflicts. Surprisingly, half of the top-performing ``memory heads'' in substitution conflicts become ``context heads'' in coherent conflicts. This shows that even the most influential model component could have completely opposite functionality.

\begin{figure*}[ht!]
      \centering
      \includegraphics[width=\textwidth]{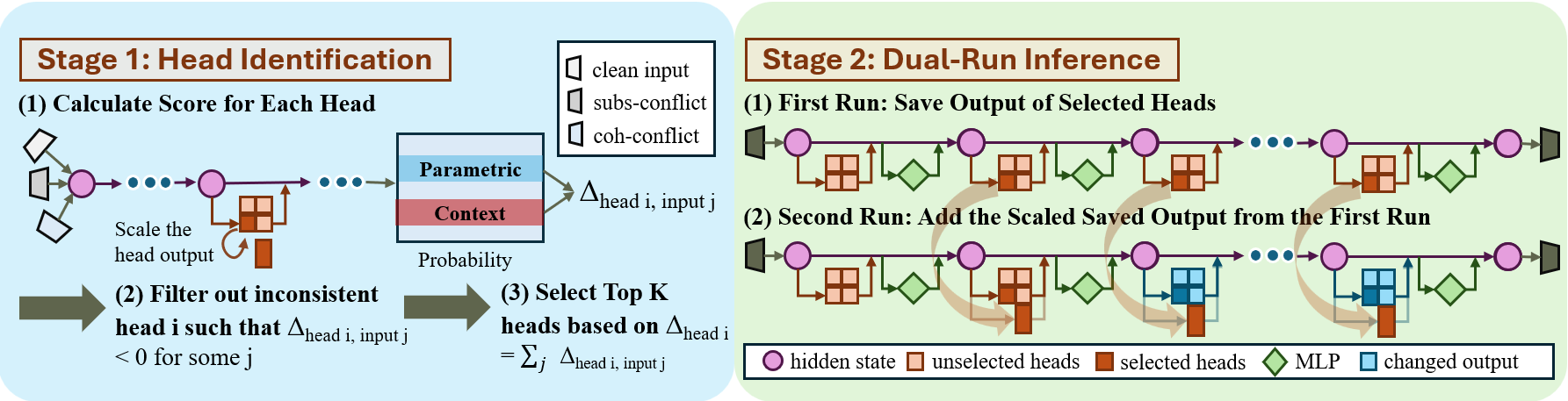}
    \caption{Overview of \jrt. In the first head identification stage (left), \jrt identifies a set of attention heads that could consistently achieve the desired effect. In the second inference stage (right), \jrt first saves the outputs of the identified heads, and then adds the scaled version of those outputs to the corresponding modules.
    }
    \label{fig:method}
    \vspace{-3mm}
  \end{figure*}

\begin{table}[h]
  \centering
  \caption{
  The top 4 heads ranked by the average prob increase of contextual knowledge in substitution-based conflicts via knocking out. We find that half of the top-influential memory heads in substitution conflict 
  lead to contrary effects in coherent conflict. Green denotes the desired behavior ($\uparrow$ context and $\downarrow$ parametric) and red denotes the undesired behavior ($\downarrow$ context and $\uparrow$ parametric).}
  \label{tab:superposition_1}
  \resizebox{0.80\linewidth}{!}{
    \begin{tabular}{@{}ccccc@{}}
\toprule
\multirow{2}{*}{\textbf{Head}} & \multicolumn{2}{c}{\textbf{Subs-Conflict}} & \multicolumn{2}{c}{\textbf{Coh-Conflict}} \\ \cmidrule(l){2-5} 
 & \textbf{$\triangle$Context Prob} & \textbf{$\triangle$Para Prob} & \textbf{$\triangle$Context Prob} & \textbf{$\triangle$Para Prob} \\ \midrule
(8, 0) & \cellcolor{lightgreen!85}+0.18 & \cellcolor{lightgreen!85}-0.03 & \cellcolor{lightgreen!85}+0.04 & \cellcolor{lightgreen!85}-0.03 \\
(15, 6) & \cellcolor{lightgreen!85}+0.16 & \cellcolor{lightgreen!85}-0.04 & \cellcolor{lightgreen!85}+0.08 & \cellcolor{lightgreen!85}-0.04 \\
(9, 3) & \cellcolor{lightgreen!85}+0.13 & \cellcolor{lightgreen!85}-0.08 & \cellcolor{lightred!85}-0.17 & \cellcolor{lightred!85}+0.09 \\
(13, 5) & \cellcolor{lightgreen!85}+0.11 & \cellcolor{lightgreen!85}-0.03 & \cellcolor{lightred!85}-0.13 & \cellcolor{lightred!85}+0.07 \\ \bottomrule
\end{tabular}
  }
\end{table}

\paragraph{Observation 2: Counteracting Effects of Multiple Interventions.} 
Expanding on prior observations, we evaluate the impact of multiple interventions on parametric knowledge. We first identify attention heads that consistently increase parametric logits when individually knocked out, ranking them by their average contribution. A natural approach is to apply these effective individual interventions simultaneously, as proposed by~\citet{jin2024cutting}. However, Tab.~\ref{tab:multiple_influence} reveals that combining individually helpful interventions does not always yield additive benefits and can even reduce performance. This behavior likely arises from the dependence of a model component’s functionality on input residual streams, as highlighted in Observation 1. Modified activations from earlier layers may alter downstream behavior, leading to counteracting effects.

\vspace{-3.75mm}
\begin{table}[h]
  \centering
  \caption{Target probability value using multiple interventions under coherent conflicts. Top-$i$ denotes combining $0$ to $i$-th ranked individual intervention performances. This shows that different modules can ``counteract'' each other, even though individual intervention contributes to substantial performance gains.}
  \label{tab:multiple_influence}

  \resizebox{0.75\linewidth}{!}{
    \begin{tabular}{@{}cc@{}}
      \toprule
      Number of Intervened Components    & Target Prob Value \\ \midrule
      None (Original Model)      & 0.03                \\
      Top 1    & 0.12                \\
      Top 3  & 0.24 \\ 
      Top 10 & 0.14                \\ \bottomrule
      \end{tabular}
  }
  \vspace{-0.7mm}
\end{table}

  
Our findings collectively suggest a phenomenon we term the ``\superpo'' (CP Superposition), where the roles of ``context'' 
or ``memory'' of model components depend on the inputs they receive. Next, we discuss how we could propose effective methods while acknowledging such superpositions.

\subsection{Our Approach: Just Run Twice (\jrt)}

We introduce Just Run Twice (\jrt), a test-time intervention method for addressing knowledge conflicts. Fig.~\ref{fig:method} illustrates the core idea, and Alg.~\ref{alg:example} provides the detailed algorithm. 
\jrt operates in two stages.

\textbf{Stage 1 (Head Identification).} This stage identifies two sets of attention heads that consistently achieve the desired effect with either positive or negative scaling. Each head is assigned a score based on the expected change in the desired probability value under individual scaling, computed across a small head selection dataset spanning multiple conflict types. To ensure consistency, only heads with non-negative scores across all conflict types are selected. The top $K$, based on aggregated scores, are retained. This process ensures reliability for individual head activations.

\textbf{Stage 2 (Dual-run Inference).} To mitigate counteracting effects from multiple interventions, the model runs twice. In the first run, the outputs of the identified heads are saved. In the second run, scaled versions of these saved outputs are added to the corresponding activations. Intuitively, the first-run activations serve as more reliable steering directions.
We validate this intuition through experiments in Sec.~\ref{subsec:direct_ablation} and analyses in Sec.~\ref{sec:theory_jrt}.

\textbf{Practical Implementation.} 
The key hyperparameters of \jrt include the size of the head selection dataset \(D\), 
the number of intervened heads \(K\), and the scaling factors at inference. 
In practice, we fix \(K\)  to be a constant number (\emph{e.g.}, 5) and determine the scaling factors using the validation set. 
We fix $|D|$ to be 4 for all primary experiments. 
Additionally, we test the generalizability of 
\jrt by using a head identification set from a single domain and evaluating its performance across other domains.

\section{Intervention Experiment}
\label{sec:experiment}

In this section, we analyze the intervention performance of \jrt and compare it against different baselines. Due to the page limit, we only present three models in the main paper. 
A more comprehensive experiment section 
with additional model results can be found in Appen.~\ref{appen:exp_detail}.

\begin{table*}[!ht]

    \centering
    \caption{
    Results of intervention for enhancing parametric memory. All results are in 
    accuracy (\(\%\)). \jrt consistently achieves the state-of-the-art 
    performances in most cases. \textbf{Bold} denotes the best result. Additional model results can be found in \lgtnote{Appen.~\ref{appen_sub:comprehensive_model_result}}.
    }
    \label{tab:main_intervention_unfiltered}
    \resizebox{\linewidth}{!}{
\begin{tabular}{@{}lllllllllllllllllllllll@{}}
\toprule
\multicolumn{2}{c}{\textbf{Dataset}}       
& \multicolumn{3}{c}{\textbf{\begin{tabular}[c]{@{}c@{}}Athlete \\ Sport\end{tabular}}}            
& \multicolumn{3}{c}{\textbf{\begin{tabular}[c]{@{}c@{}}Book \\ Author\end{tabular}}}              
& \multicolumn{3}{c}{\textbf{\begin{tabular}[c]{@{}c@{}}Company \\ Founder\end{tabular}}}          
& \multicolumn{3}{c}{\textbf{\begin{tabular}[c]{@{}c@{}}Company \\ Headquarter\end{tabular}}}      
& \multicolumn{3}{c}{\textbf{\begin{tabular}[c]{@{}c@{}}Official \\ Language\end{tabular}}}        
& \multicolumn{3}{c}{\textbf{\begin{tabular}[c]{@{}c@{}}World \\ Capital\end{tabular}}}            
& \multicolumn{3}{c}{\textbf{Average}}                                                             \\ 
\cmidrule(l){3-23} 
\multicolumn{2}{c}{\textbf{Conflict Type}} 
& \textbf{1} & \textbf{2} & \textbf{3} 
& \textbf{1} & \textbf{2} & \textbf{3} 
& \textbf{1} & \textbf{2} & \textbf{3} 
& \textbf{1} & \textbf{2} & \textbf{3} 
& \textbf{1} & \textbf{2} & \textbf{3} 
& \textbf{1} & \textbf{2} & \textbf{3} 
& \textbf{1} & \textbf{2} & \textbf{3} \\ 
\midrule

\multirow{6}{*}{Gemma}

& \cellcolor{gray!18}Original      
  & \cellcolor{gray!18}93.4 & \cellcolor{gray!18}18.1 & \cellcolor{gray!18}0.0 
  & \cellcolor{gray!18}73.0 & \cellcolor{gray!18}7.7  & \cellcolor{gray!18}0.0 
  & \cellcolor{gray!18}\textbf{47.0} & \cellcolor{gray!18}2.7  & \cellcolor{gray!18}0.0 
  & \cellcolor{gray!18}64.2 & \cellcolor{gray!18}0.7  & \cellcolor{gray!18}0.0 
  & \cellcolor{gray!18}\textbf{96.9} & \cellcolor{gray!18}23.5 & \cellcolor{gray!18}0.0 
  & \cellcolor{gray!18}\textbf{94.1} & \cellcolor{gray!18}15.1 & \cellcolor{gray!18}1.1 
  & \cellcolor{gray!18}78.1 & \cellcolor{gray!18}11.3 & \cellcolor{gray!18}0.2 \\

& \cellcolor{cyan!5}Prompt        
  & \cellcolor{cyan!5}93.4 & \cellcolor{cyan!5}44.5 & \cellcolor{cyan!5}0.0 
  & \cellcolor{cyan!5}73.0 & \cellcolor{cyan!5}22.4 & \cellcolor{cyan!5}1.6 
  & \cellcolor{cyan!5}\textbf{47.0} & \cellcolor{cyan!5}6.5  & \cellcolor{cyan!5}3.8 
  & \cellcolor{cyan!5}64.2 & \cellcolor{cyan!5}3.1  & \cellcolor{cyan!5}0.0 
  & \cellcolor{cyan!5}\textbf{96.9} & \cellcolor{cyan!5}50.0 & \cellcolor{cyan!5}22.2 
  & \cellcolor{cyan!5}\textbf{94.1} & \cellcolor{cyan!5}50.8 & \cellcolor{cyan!5}35.7 
  & \cellcolor{cyan!5}78.1 & \cellcolor{cyan!5}29.6 & \cellcolor{cyan!5}10.5 \\

& \cellcolor{cyan!5}\phl         
  & \cellcolor{cyan!5}86.6 & \cellcolor{cyan!5}71.6 & \cellcolor{cyan!5}33.3 
  & \cellcolor{cyan!5}33.3 & \cellcolor{cyan!5}4.8  & \cellcolor{cyan!5}0.0 
  & \cellcolor{cyan!5}28.1 & \cellcolor{cyan!5}10.8 & \cellcolor{cyan!5}19.5 
  & \cellcolor{cyan!5}44.3 & \cellcolor{cyan!5}22.4 & \cellcolor{cyan!5}30.6 
  & \cellcolor{cyan!5}90.7 & \cellcolor{cyan!5}72.8 & \cellcolor{cyan!5}82.7 
  & \cellcolor{cyan!5}84.3 & \cellcolor{cyan!5}64.3 & \cellcolor{cyan!5}88.1 
  & \cellcolor{cyan!5}61.2 & \cellcolor{cyan!5}41.1 & \cellcolor{cyan!5}42.4 \\

& \cellcolor{cyan!5}\phs         
  & \cellcolor{cyan!5}93.2 & \cellcolor{cyan!5}75.3 & \cellcolor{cyan!5}0.0 
  & \cellcolor{cyan!5}21.8 & \cellcolor{cyan!5}19.3 & \cellcolor{cyan!5}0.2 
  & \cellcolor{cyan!5}42.7 & \cellcolor{cyan!5}5.4  & \cellcolor{cyan!5}0.0 
  & \cellcolor{cyan!5}62.0 & \cellcolor{cyan!5}0.7  & \cellcolor{cyan!5}0.0 
  & \cellcolor{cyan!5}82.7 & \cellcolor{cyan!5}37.7 & \cellcolor{cyan!5}0.0 
  & \cellcolor{cyan!5}78.9 & \cellcolor{cyan!5}15.7 & \cellcolor{cyan!5}0.5 
  & \cellcolor{cyan!5}63.5 & \cellcolor{cyan!5}25.7 & \cellcolor{cyan!5}0.1 \\

& \cellcolor{cyan!14}\jro (Ours)
  & \cellcolor{cyan!14}91.2 & \cellcolor{cyan!14}63.2 & \cellcolor{cyan!14}65.9
  & \cellcolor{cyan!14}78.0 & \cellcolor{cyan!14}61.0 & \cellcolor{cyan!14}2.9
  & \cellcolor{cyan!14}46.5 & \cellcolor{cyan!14}\textbf{44.9} & \cellcolor{cyan!14}41.1
  & \cellcolor{cyan!14}57.9 & \cellcolor{cyan!14}36.2 & \cellcolor{cyan!14}38.9
  & \cellcolor{cyan!14}94.4 & \cellcolor{cyan!14}82.1 & \cellcolor{cyan!14}84.0
  & \cellcolor{cyan!14}91.9 & \cellcolor{cyan!14}69.2 & \cellcolor{cyan!14}83.2
  & \cellcolor{cyan!14}76.7 & \cellcolor{cyan!14}59.4 & \cellcolor{cyan!14}52.7 \\

& \cellcolor{cyan!25}\jrt (Ours)
  & \cellcolor{cyan!25}\textbf{96.3} & \cellcolor{cyan!25}\textbf{95.4} & \cellcolor{cyan!25}\textbf{91.9}
  & \cellcolor{cyan!25}\textbf{79.8} & \cellcolor{cyan!25}\textbf{75.5} & \cellcolor{cyan!25}\textbf{68.0}
  & \cellcolor{cyan!25}45.4         & \cellcolor{cyan!25}39.5         & \cellcolor{cyan!25}\textbf{43.2}
  & \cellcolor{cyan!25}\textbf{65.8}& \cellcolor{cyan!25}\textbf{60.0} & \cellcolor{cyan!25}\textbf{59.3}
  & \cellcolor{cyan!25}93.2         & \cellcolor{cyan!25}\textbf{86.4} & \cellcolor{cyan!25}\textbf{85.2}
  & \cellcolor{cyan!25}\textbf{94.1}& \cellcolor{cyan!25}\textbf{95.1} & \cellcolor{cyan!25}\textbf{93.0}
  & \cellcolor{cyan!25}\textbf{79.1}& \cellcolor{cyan!25}\textbf{75.3} & \cellcolor{cyan!25}\textbf{73.4} \\
\midrule

\multirow{6}{*}{Llama2}

& \cellcolor{gray!18}Original
  & \cellcolor{gray!18}90.4 & \cellcolor{gray!18}9.0  & \cellcolor{gray!18}0.7  
  & \cellcolor{gray!18}81.4 & \cellcolor{gray!18}47.0 & \cellcolor{gray!18}0.0  
  & \cellcolor{gray!18}\textbf{57.5} & \cellcolor{gray!18}29.3 & \cellcolor{gray!18}0.0  
  & \cellcolor{gray!18}\textbf{75.2} & \cellcolor{gray!18}1.1  & \cellcolor{gray!18}0.7  
  & \cellcolor{gray!18}95.7 & \cellcolor{gray!18}46.9 & \cellcolor{gray!18}0.0  
  & \cellcolor{gray!18}95.1 & \cellcolor{gray!18}22.3 & \cellcolor{gray!18}0.0  
  & \cellcolor{gray!18}\textbf{82.5} & \cellcolor{gray!18}25.9 & \cellcolor{gray!18}0.2 \\

& \cellcolor{cyan!5}Prompt        
  & \cellcolor{cyan!5}90.4 & \cellcolor{cyan!5}70.2 & \cellcolor{cyan!5}0.2  
  & \cellcolor{cyan!5}81.4 & \cellcolor{cyan!5}65.1 & \cellcolor{cyan!5}22.0 
  & \cellcolor{cyan!5}\textbf{57.5} & \cellcolor{cyan!5}16.6 & \cellcolor{cyan!5}24.3 
  & \cellcolor{cyan!5}\textbf{75.2} & \cellcolor{cyan!5}38.0 & \cellcolor{cyan!5}15.7 
  & \cellcolor{cyan!5}95.7 & \cellcolor{cyan!5}79.6 & \cellcolor{cyan!5}40.7 
  & \cellcolor{cyan!5}95.1 & \cellcolor{cyan!5}60.3 & \cellcolor{cyan!5}15.8 
  & \cellcolor{cyan!5}\textbf{82.5} & \cellcolor{cyan!5}55.0 & \cellcolor{cyan!5}19.8 \\

& \cellcolor{cyan!5}\phl
  & \cellcolor{cyan!5}91.0 & \cellcolor{cyan!5}87.4 & \cellcolor{cyan!5}37.5 
  & \cellcolor{cyan!5}77.8 & \cellcolor{cyan!5}92.0 & \cellcolor{cyan!5}70.9 
  & \cellcolor{cyan!5}53.0 & \cellcolor{cyan!5}\textbf{52.2} & \cellcolor{cyan!5}32.6 
  & \cellcolor{cyan!5}73.4 & \cellcolor{cyan!5}74.0 & \cellcolor{cyan!5}12.1 
  & \cellcolor{cyan!5}94.4 & \cellcolor{cyan!5}90.7 & \cellcolor{cyan!5}84.0 
  & \cellcolor{cyan!5}94.2 & \cellcolor{cyan!5}\textbf{95.7} & \cellcolor{cyan!5}90.2 
  & \cellcolor{cyan!5}80.6 & \cellcolor{cyan!5}82.0 & \cellcolor{cyan!5}54.5 \\

& \cellcolor{cyan!5}\phs
  & \cellcolor{cyan!5}89.0 & \cellcolor{cyan!5}88.1 & \cellcolor{cyan!5}10.5 
  & \cellcolor{cyan!5}80.2 & \cellcolor{cyan!5}86.1 & \cellcolor{cyan!5}64.5 
  & \cellcolor{cyan!5}52.7 & \cellcolor{cyan!5}50.0 & \cellcolor{cyan!5}34.0 
  & \cellcolor{cyan!5}73.4 & \cellcolor{cyan!5}72.9 & \cellcolor{cyan!5}18.5 
  & \cellcolor{cyan!5}94.4 & \cellcolor{cyan!5}85.5 & \cellcolor{cyan!5}80.7 
  & \cellcolor{cyan!5}94.0 & \cellcolor{cyan!5}91.3 & \cellcolor{cyan!5}85.3 
  & \cellcolor{cyan!5}80.6 & \cellcolor{cyan!5}79.0 & \cellcolor{cyan!5}48.9 \\

& \cellcolor{cyan!14}\jro (Ours)
  & \cellcolor{cyan!14}89.9 & \cellcolor{cyan!14}61.6 & \cellcolor{cyan!14}50.4
  & \cellcolor{cyan!14}77.1 & \cellcolor{cyan!14}85.6 & \cellcolor{cyan!14}79.8
  & \cellcolor{cyan!14}53.6 & \cellcolor{cyan!14}47.0 & \cellcolor{cyan!14}40.9
  & \cellcolor{cyan!14}72.2 & \cellcolor{cyan!14}66.3 & \cellcolor{cyan!14}64.0
  & \cellcolor{cyan!14}93.8 & \cellcolor{cyan!14}92.0 & \cellcolor{cyan!14}95.7
  & \cellcolor{cyan!14}94.6 & \cellcolor{cyan!14}94.0 & \cellcolor{cyan!14}95.7
  & \cellcolor{cyan!14}80.2 & \cellcolor{cyan!14}74.4 & \cellcolor{cyan!14}71.1 \\

& \cellcolor{cyan!25}\jrt (Ours)
  & \cellcolor{cyan!25}\textbf{91.5} & \cellcolor{cyan!25}\textbf{88.6} & \cellcolor{cyan!25}\textbf{91.0}
  & \cellcolor{cyan!25}\textbf{82.8} & \cellcolor{cyan!25}\textbf{91.1} & \cellcolor{cyan!25}\textbf{88.5}
  & \cellcolor{cyan!25}53.0          & \cellcolor{cyan!25}51.9          & \cellcolor{cyan!25}\textbf{54.1}
  & \cellcolor{cyan!25}74.3          & \cellcolor{cyan!25}\textbf{74.3} & \cellcolor{cyan!25}\textbf{73.6}
  & \cellcolor{cyan!25}\textbf{96.1} & \cellcolor{cyan!25}\textbf{93.8} & \cellcolor{cyan!25}\textbf{94.4}
  & \cellcolor{cyan!25}\textbf{95.4} & \cellcolor{cyan!25}95.4          & \cellcolor{cyan!25}\textbf{96.2}
  & \cellcolor{cyan!25}82.2          & \cellcolor{cyan!25}\textbf{82.5} & \cellcolor{cyan!25}\textbf{83.0} \\
\midrule

\multirow{6}{*}{Llama3}

& \cellcolor{gray!18}Original
  & \cellcolor{gray!18}84.1 & \cellcolor{gray!18}22.2 & \cellcolor{gray!18}0.0 
  & \cellcolor{gray!18}55.6 & \cellcolor{gray!18}2.2  & \cellcolor{gray!18}0.0 
  & \cellcolor{gray!18}61.1 & \cellcolor{gray!18}3.3  & \cellcolor{gray!18}0.0 
  & \cellcolor{gray!18}80.3 & \cellcolor{gray!18}1.4  & \cellcolor{gray!18}1.8 
  & \cellcolor{gray!18}96.3 & \cellcolor{gray!18}20.4 & \cellcolor{gray!18}0.6 
  & \cellcolor{gray!18}94.6 & \cellcolor{gray!18}16.8 & \cellcolor{gray!18}0.0 
  & \cellcolor{gray!18}78.7 & \cellcolor{gray!18}11.0 & \cellcolor{gray!18}0.4 \\

& \cellcolor{cyan!5}Prompt
  & \cellcolor{cyan!5}84.1 & \cellcolor{cyan!5}87.4 & \cellcolor{cyan!5}4.1 
  & \cellcolor{cyan!5}55.6 & \cellcolor{cyan!5}77.7 & \cellcolor{cyan!5}0.0 
  & \cellcolor{cyan!5}61.1 & \cellcolor{cyan!5}38.3 & \cellcolor{cyan!5}0.6 
  & \cellcolor{cyan!5}80.3 & \cellcolor{cyan!5}48.2 & \cellcolor{cyan!5}0.0 
  & \cellcolor{cyan!5}96.3 & \cellcolor{cyan!5}85.2 & \cellcolor{cyan!5}5.6 
  & \cellcolor{cyan!5}94.6 & \cellcolor{cyan!5}83.8 & \cellcolor{cyan!5}11.9 
  & \cellcolor{cyan!5}78.7 & \cellcolor{cyan!5}70.1 & \cellcolor{cyan!5}3.7 \\

& \cellcolor{cyan!5}\phl
  & \cellcolor{cyan!5}86.4 & \cellcolor{cyan!5}86.5 & \cellcolor{cyan!5}14.1 
  & \cellcolor{cyan!5}75.3 & \cellcolor{cyan!5}87.4 & \cellcolor{cyan!5}4.9 
  & \cellcolor{cyan!5}55.6 & \cellcolor{cyan!5}48.9 & \cellcolor{cyan!5}30.6 
  & \cellcolor{cyan!5}78.0 & \cellcolor{cyan!5}55.3 & \cellcolor{cyan!5}9.4 
  & \cellcolor{cyan!5}96.3 & \cellcolor{cyan!5}\textbf{96.3} & \cellcolor{cyan!5}84.0 
  & \cellcolor{cyan!5}93.0 & \cellcolor{cyan!5}94.1 & \cellcolor{cyan!5}92.4 
  & \cellcolor{cyan!5}80.7 & \cellcolor{cyan!5}78.1 & \cellcolor{cyan!5}39.2 \\

& \cellcolor{cyan!5}\phs
  & \cellcolor{cyan!5}86.5 & \cellcolor{cyan!5}86.3 & \cellcolor{cyan!5}12.5 
  & \cellcolor{cyan!5}61.1 & \cellcolor{cyan!5}84.8 & \cellcolor{cyan!5}6.8 
  & \cellcolor{cyan!5}58.3 & \cellcolor{cyan!5}51.7 & \cellcolor{cyan!5}27.8 
  & \cellcolor{cyan!5}70.0 & \cellcolor{cyan!5}56.2 & \cellcolor{cyan!5}26.8 
  & \cellcolor{cyan!5}96.3 & \cellcolor{cyan!5}95.8 & \cellcolor{cyan!5}87.0 
  & \cellcolor{cyan!5}91.4 & \cellcolor{cyan!5}87.6 & \cellcolor{cyan!5}90.3 
  & \cellcolor{cyan!5}77.3 & \cellcolor{cyan!5}77.1 & \cellcolor{cyan!5}41.9 \\

& \cellcolor{cyan!14}\jro (Ours)
  & \cellcolor{cyan!14}82.8 & \cellcolor{cyan!14}72.8 & \cellcolor{cyan!14}58.7
  & \cellcolor{cyan!14}66.2 & \cellcolor{cyan!14}92.1 & \cellcolor{cyan!14}83.0
  & \cellcolor{cyan!14}\textbf{61.7} & \cellcolor{cyan!14}51.1 & \cellcolor{cyan!14}54.4
  & \cellcolor{cyan!14}\textbf{80.5} & \cellcolor{cyan!14}56.9 & \cellcolor{cyan!14}56.0
  & \cellcolor{cyan!14}95.7 & \cellcolor{cyan!14}95.7 & \cellcolor{cyan!14}93.2
  & \cellcolor{cyan!14}94.1 & \cellcolor{cyan!14}95.7 & \cellcolor{cyan!14}96.8
  & \cellcolor{cyan!14}80.2 & \cellcolor{cyan!14}77.4 & \cellcolor{cyan!14}73.7 \\

& \cellcolor{cyan!25}\jrt (Ours)
  & \cellcolor{cyan!25}\textbf{87.0} & \cellcolor{cyan!25}\textbf{87.8} & \cellcolor{cyan!25}\textbf{95.9}
  & \cellcolor{cyan!25}\textbf{86.5} & \cellcolor{cyan!25}\textbf{92.3} & \cellcolor{cyan!25}\textbf{88.7}
  & \cellcolor{cyan!25}\textbf{61.7} & \cellcolor{cyan!25}\textbf{56.7} & \cellcolor{cyan!25}\textbf{55.6}
  & \cellcolor{cyan!25}79.8          & \cellcolor{cyan!25}\textbf{75.9} & \cellcolor{cyan!25}\textbf{74.8}
  & \cellcolor{cyan!25}\textbf{96.3} & \cellcolor{cyan!25}\textbf{96.3} & \cellcolor{cyan!25}\textbf{95.7}
  & \cellcolor{cyan!25}\textbf{95.7} & \cellcolor{cyan!25}\textbf{96.2} & \cellcolor{cyan!25}\textbf{97.3}
  & \cellcolor{cyan!25}\textbf{84.5} & \cellcolor{cyan!25}\textbf{84.2} & \cellcolor{cyan!25}\textbf{84.7} \\
\bottomrule
\end{tabular}
    }

\end{table*}

\begin{table}[!ht]
\vspace{-2.35mm}
    \centering
    \caption{Results of intervention for enhancing contextual knowledge, following the same convention as Tab.~\ref{tab:main_intervention_unfiltered}.  
    }
    \label{tab:main_intervention_context}
    
    \resizebox{\linewidth}{!}{
       \begin{tabular}{@{}llcccccc@{}}
\toprule
\multicolumn{1}{c}{\textbf{Model}} 
 & \multicolumn{1}{c}{\textbf{Method}} 
 & \textbf{\begin{tabular}[c]{@{}c@{}}NQ\\ Swap\end{tabular}} 
 & \textbf{\begin{tabular}[c]{@{}c@{}}Hate Spe-\\ ech Ending\end{tabular}} 
 & \textbf{\begin{tabular}[c]{@{}c@{}}History of \\ Science qa\end{tabular}} 
 & \textbf{\begin{tabular}[c]{@{}c@{}}Proverb\\ Ending\end{tabular}} 
 & \textbf{\begin{tabular}[c]{@{}c@{}}Proverb\\ Translation\end{tabular}} 
 & \textbf{Average} \\ 
\midrule

\multirow{7}{*}{Gemma}
 & \cellcolor{gray!18}Original 
   & \cellcolor{gray!18}38.7 
   & \cellcolor{gray!18}70.7 
   & \cellcolor{gray!18}29.9 
   & \cellcolor{gray!18}26.5 
   & \cellcolor{gray!18}59.0 
   & \cellcolor{gray!18}45.0 
   \\

 & \cellcolor{cyan!5}Prompt
   & \cellcolor{cyan!5}40.9 
   & \cellcolor{cyan!5}73.2 
   & \cellcolor{cyan!5}38.0 
   & \cellcolor{cyan!5}26.6 
   & \cellcolor{cyan!5}58.4 
   & \cellcolor{cyan!5}47.4 
   \\

 & \cellcolor{cyan!5}CAD
   & \cellcolor{cyan!5}56.9 
   & \cellcolor{cyan!5}81.7 
   & \cellcolor{cyan!5}16.9 
   & \cellcolor{cyan!5}37.1 
   & \cellcolor{cyan!5}62.9 
   & \cellcolor{cyan!5}51.1 
   \\

 & \cellcolor{cyan!5}\phl
   & \cellcolor{cyan!5}51.0 
   & \cellcolor{cyan!5}82.8 
   & \cellcolor{cyan!5}46.5 
   & \cellcolor{cyan!5}57.8 
   & \cellcolor{cyan!5}62.0 
   & \cellcolor{cyan!5}60.0 
   \\

 & \cellcolor{cyan!5}\phs
   & \cellcolor{cyan!5}50.2 
   & \cellcolor{cyan!5}80.2 
   & \cellcolor{cyan!5}35.2 
   & \cellcolor{cyan!5}50.1 
   & \cellcolor{cyan!5}63.2 
   & \cellcolor{cyan!5}55.8 
   \\

 & \cellcolor{cyan!14}\jro (Ours)
   & \cellcolor{cyan!14}38.7 
   & \cellcolor{cyan!14}79.3 
   & \cellcolor{cyan!14}\textbf{50.1} 
   & \cellcolor{cyan!14}26.8 
   & \cellcolor{cyan!14}67.1 
   & \cellcolor{cyan!14}52.4 
   \\

 & \cellcolor{cyan!25}\jrt (Ours)
   & \cellcolor{cyan!25}\textbf{58.4} 
   & \cellcolor{cyan!25}\textbf{84.1} 
   & \cellcolor{cyan!25}47.0 
   & \cellcolor{cyan!25}\textbf{74.6} 
   & \cellcolor{cyan!25}\textbf{66.8} 
   & \cellcolor{cyan!25}\textbf{66.2} 
   \\
\midrule

\multirow{7}{*}{Llama2}
 & \cellcolor{gray!18}Original
   & \cellcolor{gray!18}24.5 
   & \cellcolor{gray!18}57.3 
   & \cellcolor{gray!18}13.3 
   & \cellcolor{gray!18}26.6 
   & \cellcolor{gray!18}52.8 
   & \cellcolor{gray!18}34.9 
   \\

 & \cellcolor{cyan!5}Prompt
   & \cellcolor{cyan!5}39.6 
   & \cellcolor{cyan!5}58.5 
   & \cellcolor{cyan!5}21.3 
   & \cellcolor{cyan!5}25.7 
   & \cellcolor{cyan!5}52.5 
   & \cellcolor{cyan!5}39.5 
   \\

 & \cellcolor{cyan!5}CAD
   & \cellcolor{cyan!5}29.8 
   & \cellcolor{cyan!5}65.4 
   & \cellcolor{cyan!5}20.2 
   & \cellcolor{cyan!5}28.6 
   & \cellcolor{cyan!5}54.2 
   & \cellcolor{cyan!5}41.4 
   \\

 & \cellcolor{cyan!5}\phl
   & \cellcolor{cyan!5}48.2 
   & \cellcolor{cyan!5}63.4 
   & \cellcolor{cyan!5}20.4 
   & \cellcolor{cyan!5}68.7 
   & \cellcolor{cyan!5}58.8 
   & \cellcolor{cyan!5}51.9 
   \\

 & \cellcolor{cyan!5}\phs
   & \cellcolor{cyan!5}25.3 
   & \cellcolor{cyan!5}62.2 
   & \cellcolor{cyan!5}16.5 
   & \cellcolor{cyan!5}26.5 
   & \cellcolor{cyan!5}55.2 
   & \cellcolor{cyan!5}37.1 
   \\

 & \cellcolor{cyan!14}\jro (Ours)
   & \cellcolor{cyan!14}29.7 
   & \cellcolor{cyan!14}76.8 
   & \cellcolor{cyan!14}49.3 
   & \cellcolor{cyan!14}34.3 
   & \cellcolor{cyan!14}52.8 
   & \cellcolor{cyan!14}48.6 
   \\

 & \cellcolor{cyan!25}\jrt (Ours)
   & \cellcolor{cyan!25}\textbf{49.5} 
   & \cellcolor{cyan!25}\textbf{93.9} 
   & \cellcolor{cyan!25}\textbf{50.2} 
   & \cellcolor{cyan!25}\textbf{77.1} 
   & \cellcolor{cyan!25}\textbf{62.6} 
   & \cellcolor{cyan!25}\textbf{66.6} 
   \\
\midrule

\multirow{7}{*}{Llama3}
 & \cellcolor{gray!18}Original
   & \cellcolor{gray!18}18.5 
   & \cellcolor{gray!18}51.2 
   & \cellcolor{gray!18}72.9 
   & \cellcolor{gray!18}24.5 
   & \cellcolor{gray!18}50.1 
   & \cellcolor{gray!18}43.4 
   \\

 & \cellcolor{cyan!5}Prompt
   & \cellcolor{cyan!5}33.4 
   & \cellcolor{cyan!5}53.7 
   & \cellcolor{cyan!5}71.7 
   & \cellcolor{cyan!5}23.9 
   & \cellcolor{cyan!5}51.8 
   & \cellcolor{cyan!5}46.9 
   \\

 & \cellcolor{cyan!5}CAD
   & \cellcolor{cyan!5}34.7 
   & \cellcolor{cyan!5}60.8 
   & \cellcolor{cyan!5}73.1 
   & \cellcolor{cyan!5}33.1 
   & \cellcolor{cyan!5}54.1 
   & \cellcolor{cyan!5}51.2
   \\

 & \cellcolor{cyan!5}\phl
   & \cellcolor{cyan!5}25.3 
   & \cellcolor{cyan!5}62.2 
   & \cellcolor{cyan!5}\textbf{78.4} 
   & \cellcolor{cyan!5}48.5 
   & \cellcolor{cyan!5}63.6 
   & \cellcolor{cyan!5}55.6 
   \\

 & \cellcolor{cyan!5}\phs
   & \cellcolor{cyan!5}22.5 
   & \cellcolor{cyan!5}51.2 
   & \cellcolor{cyan!5}75.1 
   & \cellcolor{cyan!5}25.0 
   & \cellcolor{cyan!5}51.8 
   & \cellcolor{cyan!5}45.1 
   \\

 & \cellcolor{cyan!14}\jro (Ours)
   & \cellcolor{cyan!14}26.5 
   & \cellcolor{cyan!14}72.5 
   & \cellcolor{cyan!14}73.2 
   & \cellcolor{cyan!14}33.1 
   & \cellcolor{cyan!14}61.8 
   & \cellcolor{cyan!14}53.4 
   \\

 & \cellcolor{cyan!25}\jrt (Ours)
   & \cellcolor{cyan!25}\textbf{35.3} 
   & \cellcolor{cyan!25}\textbf{78.4} 
   & \cellcolor{cyan!25}74.2 
   & \cellcolor{cyan!25}\textbf{75.4} 
   & \cellcolor{cyan!25}\textbf{70.7} 
   & \cellcolor{cyan!25}\textbf{66.8} 
   \\
\bottomrule
\end{tabular}
    }
    \vspace{-5mm}
\end{table}

\newcommand{\githubrepo}{\url{https://github.com/tensor-gales/HeteLinkPred}}

\subsection{Enhancing Parametric Beliefs}
\label{subsec:intervention}

\textbf{Setups.} 
We use the datasets and evaluation metric detailed in Sec. \ref{subsec:para_setup}. 
Notably, we have three different conflict types: No Conflict (Type 1), 
Substitution Conflict (Type 2), and Coherent Conflict (Type 3).
For presentation clarity, we use the number to represent 
these conflict types in Tab.~\ref{tab:main_intervention_unfiltered}.

\textbf{Baselines.} We compare our methods against the following baselines:
(1) \textbf{Prompt:} We instruct the LM to generate answers solely based on internal memory; 
(2) \textbf{PH3:} \cite{jin2024cutting} leverages patching-based 
methods to identify and prune ``context'' and ``memory'' heads, demonstrating strong performance in substitution conflicts. 
We note that the original PH3 requires a \emph{development} set of 200 samples for 
head identification. For a fair comparison, we include two versions of PH3: \textbf{\phl}, 
the original version, and \textbf{\phs}, which uses the same amount of 
samples as \jrt for head identifications (\emph{i.e.,} 4 samples).
(3) \textbf{\jro (Just Run Once):} an \emph{ablated} variant of \jrt that only omits the dual-run design, whose details can be found in Appen. \ref{appen:algorithm}.

\textbf{Results.} Tab.~\ref{tab:main_intervention_unfiltered} presents the results of these intervention methods across 
different models. 
Key observations include:

\begin{enumerate}
    \setlength\itemsep{0em}
    \item \jrt consistently and significantly outperforms all baselines in most cases. 
    Experimental results indicate that \jrt can almost completely reverse the model's tendency 
    to produce contextual knowledge, even in the most challenging (coherent conflict, Type 3) scenarios. 
    
    \item 
    \jrt achieves improvements on zero-shot clean prompts, enhancing the factuality of the model.
    \item While PH3 and Prompt demonstrate notable improvements in substitution conflicts under certain conditions, they fail to effectively address coherent conflict scenarios.
    Importantly, there is a clear performance difference when PH3 has a small 
    set of head identification sets. \jrt can achieve better performance with a significantly 
    smaller head identification set. 
    \item \jrt outperforms \jro on average in almost all cases. In particular, 
    the gap is about 20\% with the Gemma model. This ablation further illustrates 
    the effectiveness of the dual-run design of \jrt. 

    \item While PH3 bears an appealing ability to identify ``cross-relation heads''~\citep{jin2024cutting}, its transferability is largely limited to closely related datasets (\emph{i.e.,} heads identified from the world capital dataset are effective for the official language dataset but not for the company headquarters dataset). In contrast, our method achieves high performance across diverse domains, with heads only being selected from the world capital domain. 
\end{enumerate}

\subsection{Enhancing Contextual Reliance}
\label{subsec:intervention_context}

\begin{figure*}[th!]
  \centering
  \subfloat[Head Identification Set Size]{%
    \includegraphics[width=0.32\linewidth]{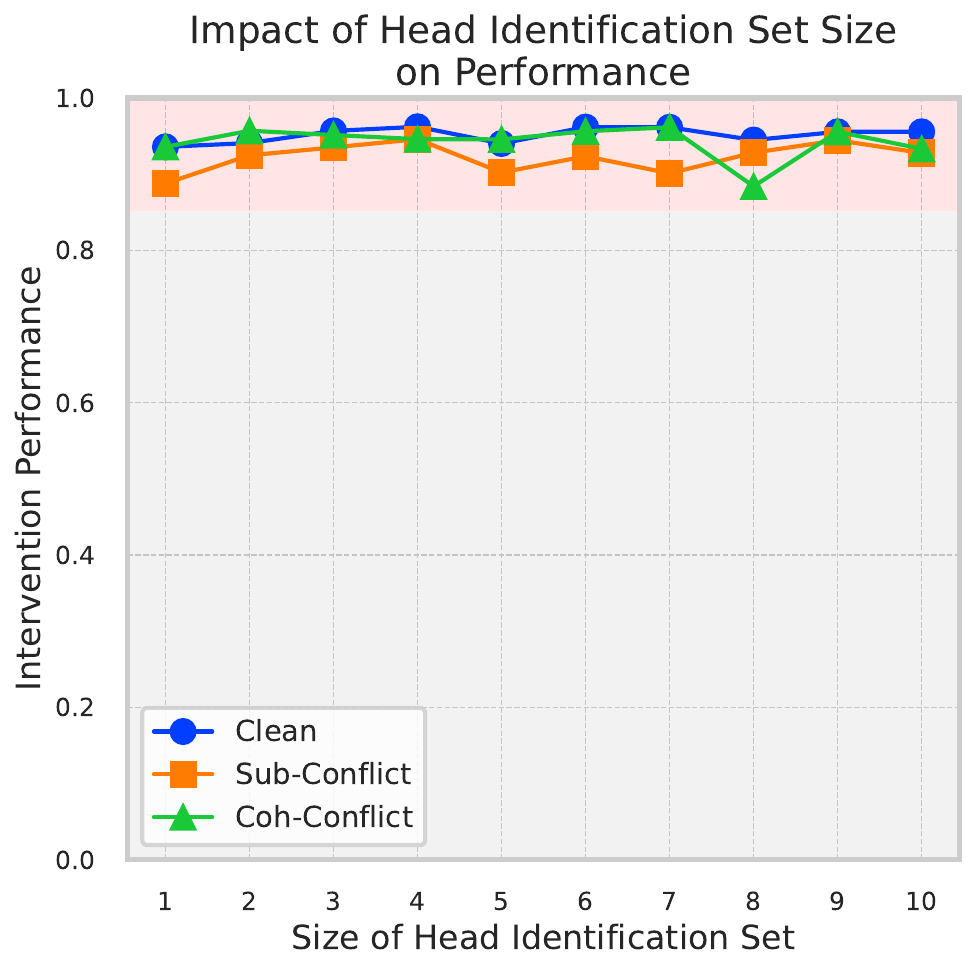}%
    \label{fig:head_identification_set}%
  }\hfill
  \subfloat[Number of Heads Intervened]{%
    \includegraphics[width=0.32\linewidth]{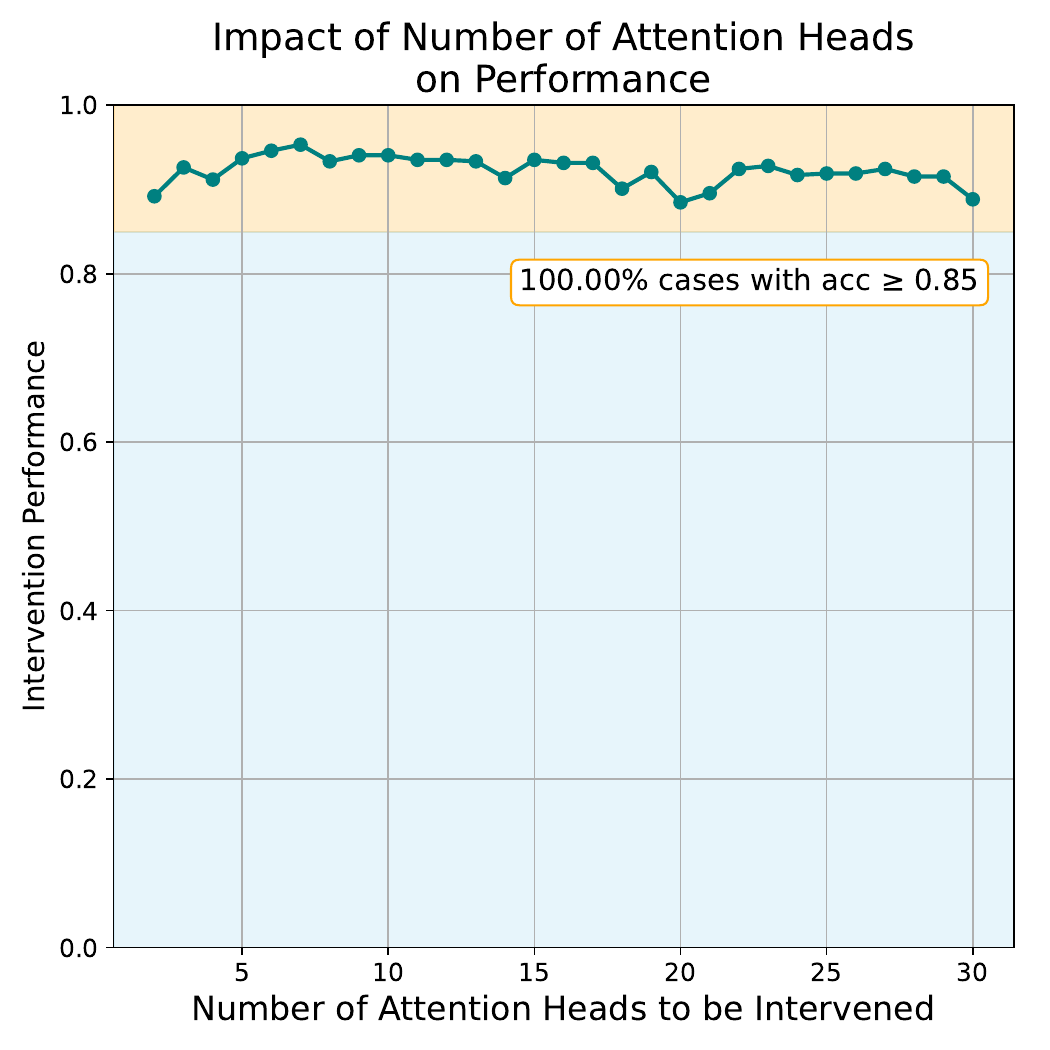}%
    \label{fig:num_head}%
  }
  \subfloat[Scaling Factor]{%
    \includegraphics[width=0.32\linewidth]{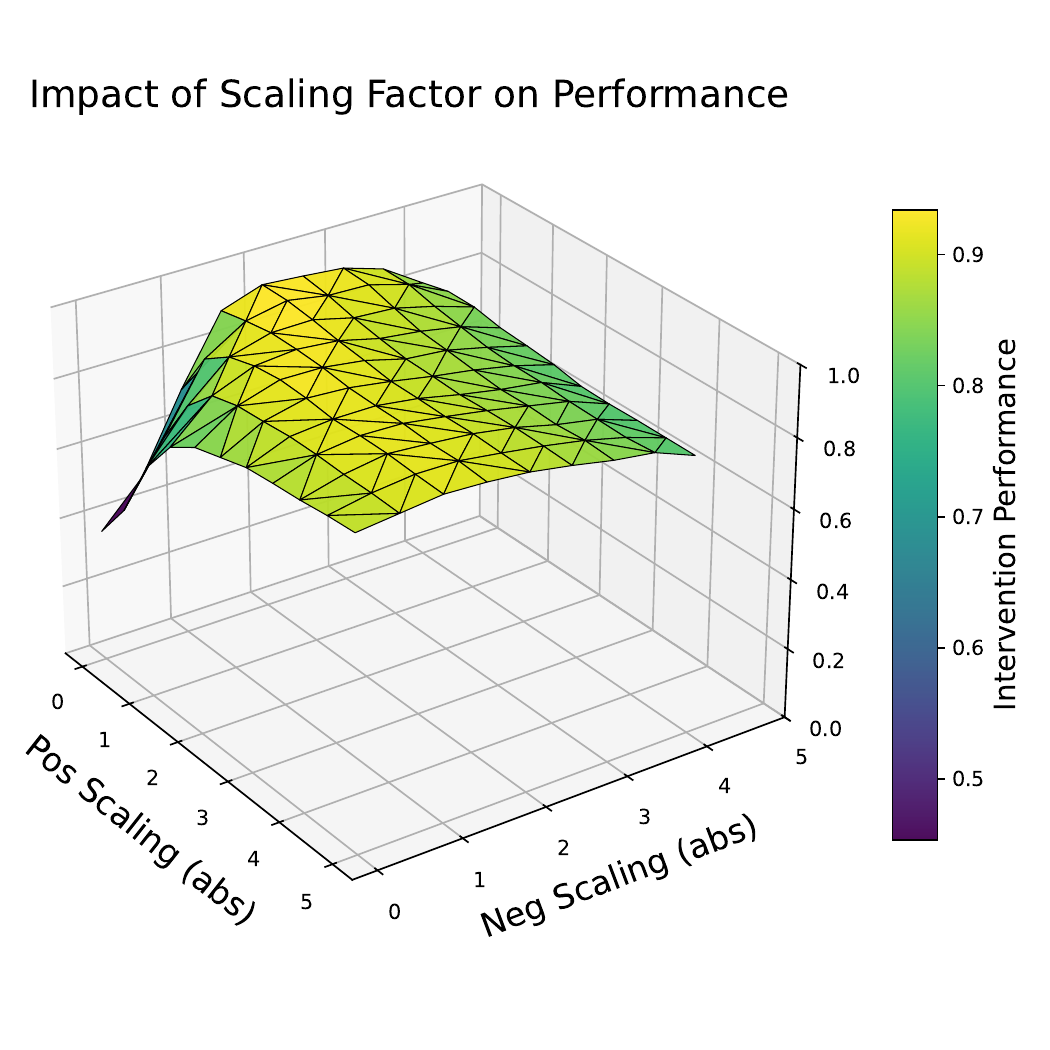}%
    \label{fig:scaling}%
  }
  \caption{Robustness analysis of \jrt across key hyperparameters. We observe consistent intervention performance as we vary the head identification set size, the number of heads intervened, and the scaling factor magnitudes, underscoring the robustness and adaptability.
  }
  \label{fig:robustness_plot}
\end{figure*}

\textbf{Setups and Baselines.} We use the datasets and evaluation metric detailed in Sec. \ref{subsec:context_setup}.  
We compare our methods against the previously 
mentioned baselines and an additional one: \textbf{CAD}~\citep{shi2023trusting}, 
a decoding-based method that leverages contrastive decoding~\citep{li2022contrastive} 
to encourage the language model to attend to its context.

\textbf{Results.} Tab.~\ref{tab:main_intervention_context} presents the 
results of these intervention methods across the models. The main conclusions 
from the prior subsection are still valid. \jrt consistently outperforms 
all baselines on average and is versatile in promoting contextual knowledge as well.

\subsection{Robustness of \jrt}
\label{subsec:robustness}




In this section, we examine the robustness of \jrt against variations in key hyperparameters and paraphrased prompts. Using Gemma as our backbone model, we systematically vary one hyperparameter at a time to isolate its effects on performance. Specifically, we evaluate the impact of three hyperparameters: the size of the head identification set $|D|$, the number of intervened attention heads $K$, and the magnitude of the scaling factors. Additionally, we investigate robustness to paraphrased prompts by employing multiple curated templates for each conflict type, selecting one at random during evaluation. Detailed experimental setups and additional analyses are provided in Appendix~\ref{appen_sub:robustness}.

Figure~\ref{fig:robustness_plot} illustrates the robustness of \jrt{} across these hyperparameters. The results demonstrate that \jrt{} maintains consistently high performance across a wide range of hyperparameter values, highlighting its stability and effectiveness.

Tab.~\ref{tab:appen_paraphrase} in Appendix~\ref {appen_sub:robustness} presents the results of \jrt{} when applied to paraphrased prompts. Our findings show that \jrt{} is highly robust to variations in input prompt formats, consistently maintaining its effectiveness across diverse templates. Notably, \jrt{} continues to demonstrate superior performance, effectively shifting the model's reliance from context to parametric memory.



\subsection{\jro vs. \jrt: Effect of Running Twice}
\label{subsec:direct_ablation}

\begin{figure}[h!]
    \centering
    \includegraphics[width=\linewidth]{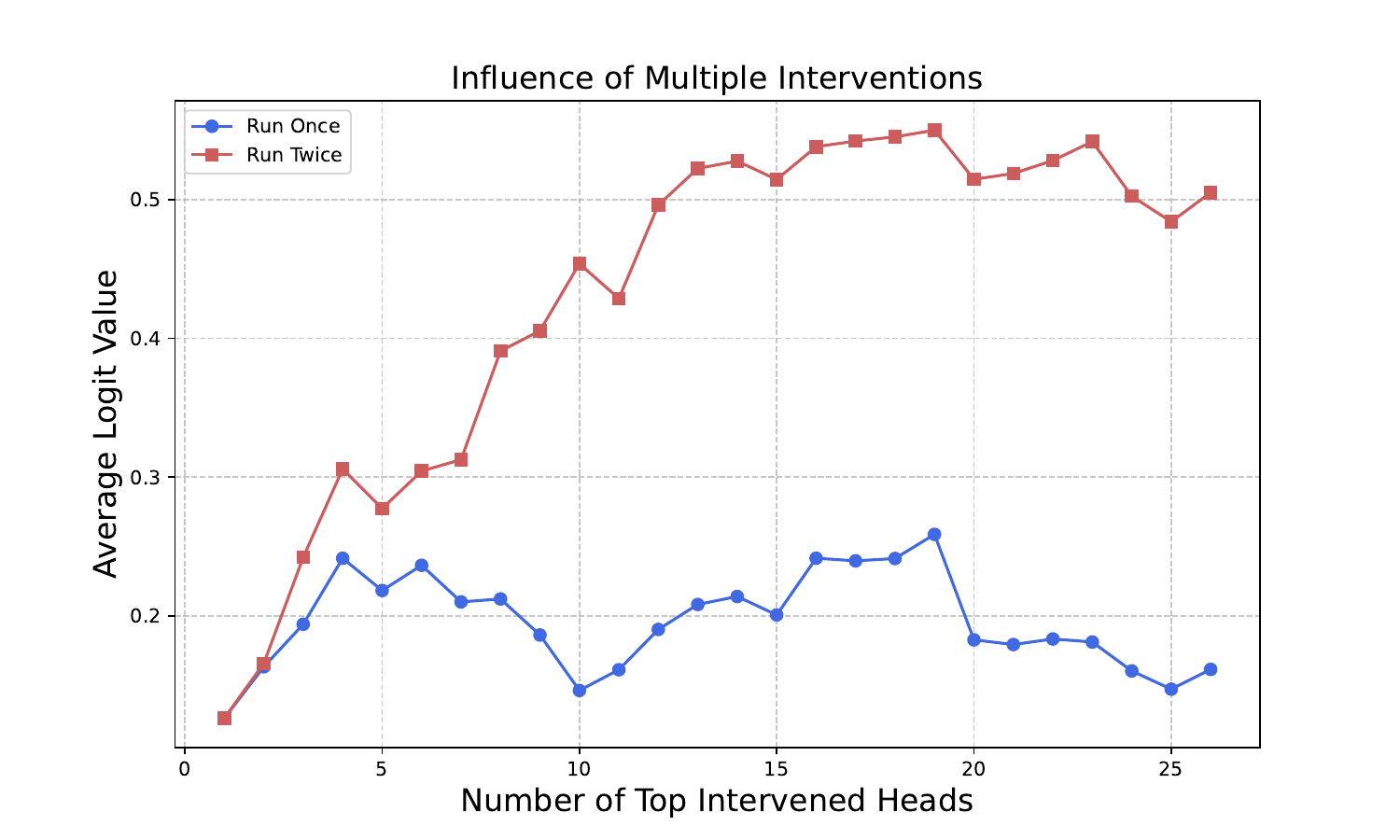}
  \caption{
  Effect of Running Twice: Mitigating Counteracting Effects of Multiple Interventions. All presented heads contribute to individual gains, starting from a baseline logit value of 0.03. The results show that naive single-pass interventions are unstable and prone to degradation. In contrast, the dual-run design ensures consistent and effective interventions.
  }
  \label{fig:multiple_interventions2}
\end{figure}

We conduct an additional experiment to demonstrate the effectiveness of the dual-run design. Following the same setup as in 
Tab.~\ref{tab:multiple_influence}, we compare the intervened logit value of Run Once versus Run Twice when combining multiple individually effective interventions. As shown in Fig.~\ref{fig:multiple_interventions2}, single-pass interventions are unstable and prone to performance degradation. In contrast, the dual-run design delivers consistently effective interventions.

\section{Theoretical Analysis}
\label{sec:theory_jrt}
In the previous sections, we have conducted a comprehensive empirical analysis to identify the phenomenon of \emph{CP superposition} and demonstrated the effectiveness of \jrt across a variety of setups.
In this section, we aim to formalize our observations and understand the underlying mechanisms behind both observations.
Specifically, we conceptualize knowledge conflicts as arising naturally within the weight matrices of the attention module, shaped through the training process via gradient descent.
Under such conditions, we elucidate that \jrt provides a superior approach compared to naive single-pass interventions. A more detailed theoretical analysis can be found in Appen.~\ref{appen:theory}. We first provide a brief overview of the model and task setup.

\begin{figure*}[h!]
    \centering
    \includegraphics[width=0.9\linewidth]{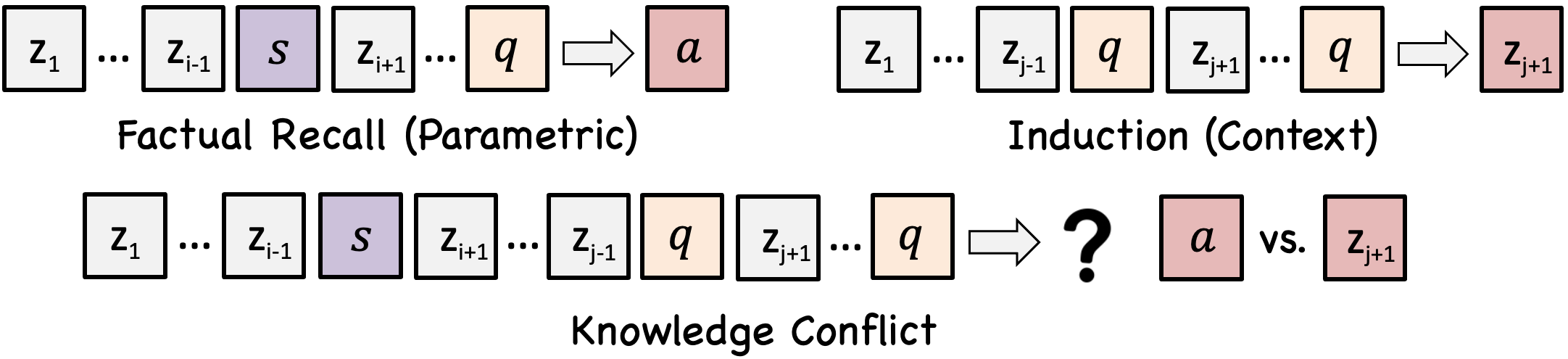}
  \caption{
  Illustration of the theoretical task setup. The top row shows two distinct tasks that a two-layer transformer learns during training; the bottom row depicts the conflicting task encountered at inference. Here, $z_j$ denotes noisy tokens, $s$ is the subject token, $a$ is the answer token associated with $s$, and $q$ is the trigger and fixed query (EOS) token.
  }
  \label{fig:theory_task_setup}
\end{figure*}

\textbf{Model Setup.} 
We use a two-layer Transformer with one attention head per layer, absolute positional encoding, and residual connections. 
The input is a sequence of tokens \(z_{1:T} \in [N]^T\), where \(T\) is the sequence length, and \(N\) is the vocabulary size. 
Each token \(z_t\) is mapped to a \(d\)-dimensional embedding \(\phi(z_t)\), and a positional embedding \(p_t \in \mathbb{R}^d\) is added. The input to the model is:
\(x_T \coloneqq \phi(z_t) + p_t\) for \(t = 1, \ldots, T\). 
We denote \(X^{(l)} = [x_1, \ldots, x_T]\) as the representation of the embeddings at layer \(l\). These embeddings are updated through two layers as follows:

\begin{equation*}
    X^{(l+1)} = X^{(l)} + \Wov^{(l)} X^{(l)} \sft\left( \text{MSK} \odot \left( X^{(l)} W_{KQ}^{(l)} X^{(l)} \right) \right)
\end{equation*}

where \(\sft\) is the column-wise softmax function. 
Finally, the embeddings are mapped back to the vocabulary space through a linear layer parameterized by \(\Wlin \in \mathbb{R}^{d \times N}\). 
The \(i\)-th column vector is denoted as \(\mu(i)\).

\paragraph{Task Setup.} 
We consider two tasks in parallel: \textbf{Factual Recalls} and \textbf{Induction}. \lgtnote{They correspond to parametric and contextual tasks, respectively.}
A diagram illustration of the whole theoretical task setup can be found in Fig.~\ref{fig:theory_task_setup}.

In the \textbf{factual recall} task~\citep{nichani2024understanding}, the goal is to learn associations between the subject token space \(\Sc\) and the answer token space \(\Ac\), based on a bijective ground truth mapping \(\Gc^* \colon \Sc \to \Ac\). 
This models knowledge triples like \emph{(China, capital, Beijing)}, where the subject token \emph{(China, capital)} maps to the answer token \emph{(Beijing)}. 
Non-critical tokens like ``the'' and ``of'' also constitute part of a factual sentence, and we assume these tokens are from the noise token space $\Nc$.
Sequences \(z_{1:T+1} \in [N]^{T+1}\) are generated as follows:

\vspace{-0.3cm}
\begin{enumerate}
    \setlength\itemsep{0em}
    \item Sample a fact \(s \in \Sc\) and index \(i \in [T - 1]\) uniformly at random, and set \(z_i = s\). 
    \item For all \(k \in [T - 1] \backslash \{i\}\), sample \(z_k\) uniformly from \(\Nc\) 
    without replacement.
    \item Set \(z_T = q\), the query token
    and \(z_{T+1} = \Gc^*(s)\).
\end{enumerate}

In the \textbf{induction} task~\citep{olsson2022context}, the goal is to predict a token \(b \in \Nc\)
following the second occurence 
of a trigger word \(q\) (\emph{e.g.} $...qb … q \to b$). Sequences \(z_{1:T+1} \in [N]^{T+1}\) are generated as follows:

\vspace{-0.3cm}
\begin{enumerate}
    \setlength\itemsep{0em}
    \item Sample \(j \in [T - 2] \backslash \{1\}\) uniformly, set \(z_j = q\), and sample \(z_{j+1}\) from \(\Nc\).
    \item For all other token \(z_k\), sample uniformly at random from \(\Nc \backslash \{z_{j+1}\}\) without replacement.
    \item Set \(z_{T} = q\) and \(z_{T+1} = z_{j+1}\).
\end{enumerate}

In summary, the vocabulary space consists of \(\Vc = \Sc \cup \Ac \cup \{q\} \cup \Nc\). We remark that we use the same trigger 
token \(q\) as the fixed query token in the factual recall task to induce knowledge conflicts. 

\begin{assumption}[Near-orthogonal Initialization] 
    All embedding, unembedding, and positional vectors are initialized randomly. 
\end{assumption}

This ensures near-orthogonality among all embeddings and unembeddings, such that \(\langle \phi(z_i), \phi(z_j) \rangle \approx \delta_{ij} (\mathbbm{1}[i = j])\) 
when the embedding dimension \(d\) is large. 
Our setting is similar to recent works~\cite{bietti2024birth, ghosal2024understanding, jiang2024llms, nichani2024understanding}. 

\subsection{CP Superposition}

We first examine how knowledge conflict arises in our simplified model, starting by demonstrating its existence.

\begin{proposition}[Existence of a Perfect Solver]
    There exists a two-layer transformer that can solve both \emph{induction} and \emph{factual recall} 
    tasks with the perfect accuracy. 
    \label{prop:two_layer_task}
\end{proposition}

The construction can be achieved as follows. By setting \(\Wov^{(1)}\) as a random matrix and defining

\vspace{-0.1cm}
\begin{align}
    \Wkq^{(1)} &= C \sum_{t=1}^{T-1} p_{t}p_{t+1}\T,  \\ 
    \Wkq^{(2)} &= C_1 \left( \Wov^1 \phi(q)\right) \phi(q)\T + C_2 \sum_{s \in \Sc} \phi(s) \phi(q)\T, \label{equa:wkq_2}\\ 
    \Wov^{(2)} &= C_3\sum_{k\in \Nc} \mu(k) \phi(k)\T + C_4\sum_{s \in S} \mu \left( \Gc^* \left( s \right) \right) \phi(s)\T,  \label{equa:wov_2} 
\end{align}

where \(C_1, C_2, C_3, C_4\) are appropriate scaling factors and $C$ is a large constant. In this setup, the first layer implements a ``copy from previous embedding''
behavior, while the second layer learns the critical tokens and associated memory required for the tasks. 
Notably, the construction of the second layer inherently forms a superposition, which leads to knowledge conflicts.

Next, we analyze how this construction could naturally emerge from training via gradient descent with a cross-entropy loss over the two tasks. 
We assume a perfectly learned first layer and focus on the dynamics of the second layer, as it suffices to illustrate the core idea. 
For simplicity, we assume a linear attention model and strictly orthogonal embeddings (i.e., all initialized vectors are orthogonal), 
which are common in the existing literature~\cite{li2023transformers, ahn2023transformers, zhang2023trained, mahankali2023one}.

\begin{proposition}[Learning the Second Superposition Layer via Gradient Descent, Informal]
    In a simplified setup using one-layer attention only transformer, 
    the superposition head as constructed in Eq.\ref{equa:wkq_2} and Eq.\ref{equa:wov_2} can be trained via gradient descent from 
    zero initialization using the cross-entropy loss. 
    \label{prop:learning_dynamics_main}
\end{proposition}
\vspace{-4.5mm}

We defer the proof to Appen.~\ref{appen:theory}.
This proposition tells us that the standard training objectives of language models encourages superposition. 
In practice, the first layer may also learn associative memories required by different tasks.
Such formulation of the weight matrices naturally results in knowledge conflicts at the inference time.

\subsection{Knowledge Conflict}

We now define and analyze the \textbf{knowledge conflict task}:  

\vspace{-1.5mm}

\begin{enumerate}
    \setlength\itemsep{0em}
    \item Sample an index \(j \in [T - 2] \backslash \{1\}\), set \(z_j = q\), and sample \(z_{j + 1}\) from \(\Nc\). 
    \item Sample an index \(i \in [T - 1] \backslash \{j, j + 1\}\) and \(s \in \Sc\). Set \(z_i = s\). 
    \item Set \(z_{T} = q\).
\end{enumerate}

\begin{corollary}[Knowledge Conflict]
    Under the knowledge conflict inference setting, the model capable of solving both \emph{factual recall} and 
    \emph{induction} from \Cref{prop:two_layer_task} may output either the inductive token or the 
    factual token. More specifically, 
    if \(\exp(C_1) C_3 < \exp(C_2) C_4\), then the model outputs the factual recall answer \(\Gc^*(s)\); otherwise, the model 
    outputs the induction answer \(z_{j+1}\). 
\end{corollary}

This corollary highlights how distinct, well-defined training tasks can overlap at inference.
The conflict arises naturally due to the associative memory structure of the weight matrices tied to specific tokens. The model's output preference depends on the relative strengths of coefficients \(C_1, \ldots, C_4\), which are influenced by factors like the learning rate and the number of (task) samples. Notably, the coefficient $C_i$ should be sample-dependent in practice.  
\cite{yu2023characterizing} found that models are more likely to generate the parametric answer when the corresponding fact appears frequently in the pretraining data, aligning with our results.

Finally, we manifest the effectiveness of the dual-run design over single-pass intervention.

\begin{proposition}[Effectiveness of \jrt]
    Consider the model from Prop.~\ref{prop:two_layer_task} and the case when its \emph{inductive} part dominates (\emph{i.e.,} \(\exp(C_1)C_3 >> \exp(C_2) C_4\)), then 
    the intervention by \jro/PH3 of deleting the two attention heads is not as effective as \jrt. 
    In particular, in this case \jro/PH3 does not result in the parametric answer, while \jrt does. 
\end{proposition}

Both attention heads from Prop.~\ref{prop:two_layer_task} can be identified as ``influential context heads'' in the above setting. However, 
when the first head is removed, the second head no longer functions for the induction task but instead transitions into a factual memorizer.
A single-pass intervention method may still remove the second head, as it was initially classified as a ``context head". 
By instead deleting activations from the original run—arguably a more reliable source—\jrt achieves more precise control over the model's behavior 
and steers it as desired.

\section{Conclusion}
\label{sec:conclusion}
This work presents a unified and principled study of knowledge conflicts in language models, revealing the phenomenon of superposition of contextual information and parametric memory. We propose Just Run Twice (\jrt), a simple yet effective test-time intervention that reliably steers models toward either parametric beliefs or contextual information
without requiring fine-tuning. \jrt consistently and significantly achieves effective intervention performance across different datasets under various conflict types. Our theoretical analysis further reveals the underlying mechanisms of knowledge conflict and the effectiveness of \jrt. These findings not only enhance our fundamental understanding of LMs’ knowledge representation mechanism
but also offer a practical method for improving model controllability in real-world applications. We discuss possible limitations and future works in Appen.~\ref{appen:limitation}.

\section*{Acknowledgement}

We appreciate Ruizhong Qiu for the early discussion about the work.
This work is partially supported by NSF (2416070). 
The content of the information in this document does not necessarily reflect the position or the policy of the Government, and no official endorsement should be inferred.  The U.S. Government is authorized to reproduce and distribute reprints for Government purposes notwithstanding any copyright notation here on.
This research used the Delta advanced computing and data resource which is supported by the National Science Foundation (award OAC 2005572) and the State of Illinois. Delta is a joint effort of the University of Illinois Urbana-Champaign and its National Center for Supercomputing Applications.
This work used the Delta system at the National Center for Supercomputing Applications through allocation CIS250054 from the Advanced Cyberinfrastructure Coordination Ecosystem: Services \& Support (ACCESS) program, which is supported by National Science Foundation grants \#2138259, \#2138286, \#2138307, \#2137603, and \#2138296.

\section*{Impact Statement}

This paper presents work whose goal is to advance the field of Machine Learning. There are many potential societal consequences of our work, none of which we feel must be specifically highlighted here.

\bibliography{reference}

\begin{thebibliography}{63}
\providecommand{\natexlab}[1]{#1}
\providecommand{\url}[1]{\texttt{#1}}
\expandafter\ifx\csname urlstyle\endcsname\relax
  \providecommand{\doi}[1]{doi: #1}\else
  \providecommand{\doi}{doi: \begingroup \urlstyle{rm}\Url}\fi

\bibitem[Ahn et~al.(2023)Ahn, Cheng, Daneshmand, and Sra]{ahn2023transformers}
Ahn, K., Cheng, X., Daneshmand, H., and Sra, S.
\newblock Transformers learn to implement preconditioned gradient descent for in-context learning.
\newblock \emph{Advances in Neural Information Processing Systems}, 36:\penalty0 45614--45650, 2023.

\bibitem[Arora et~al.(2018)Arora, Li, Liang, Ma, and Risteski]{arora2018linear}
Arora, S., Li, Y., Liang, Y., Ma, T., and Risteski, A.
\newblock Linear algebraic structure of word senses, with applications to polysemy.
\newblock \emph{Transactions of the Association for Computational Linguistics}, 6:\penalty0 483--495, 2018.

\bibitem[Bellagente et~al.(2024)Bellagente, Tow, Mahan, Phung, Zhuravinskyi, Adithyan, Baicoianu, Brooks, Cooper, Datta, et~al.]{bellagente2024stable}
Bellagente, M., Tow, J., Mahan, D., Phung, D., Zhuravinskyi, M., Adithyan, R., Baicoianu, J., Brooks, B., Cooper, N., Datta, A., et~al.
\newblock Stable lm 2 1.6 b technical report.
\newblock \emph{arXiv preprint arXiv:2402.17834}, 2024.

\bibitem[Bietti et~al.(2024)Bietti, Cabannes, Bouchacourt, Jegou, and Bottou]{bietti2024birth}
Bietti, A., Cabannes, V., Bouchacourt, D., Jegou, H., and Bottou, L.
\newblock Birth of a transformer: A memory viewpoint.
\newblock \emph{Advances in Neural Information Processing Systems}, 36, 2024.

\bibitem[Cabannes et~al.(2023)Cabannes, Dohmatob, and Bietti]{cabannes2023scaling}
Cabannes, V., Dohmatob, E., and Bietti, A.
\newblock Scaling laws for associative memories.
\newblock \emph{arXiv preprint arXiv:2310.02984}, 2023.

\bibitem[Cammarata et~al.(2020)Cammarata, Carter, Goh, Olah, Petrov, Schubert, Voss, Egan, and Lim]{cammarata2020thread}
Cammarata, N., Carter, S., Goh, G., Olah, C., Petrov, M., Schubert, L., Voss, C., Egan, B., and Lim, S.~K.
\newblock Thread: circuits.
\newblock \emph{Distill}, 5\penalty0 (3):\penalty0 e24, 2020.

\bibitem[Chang \& Bergen(2024)Chang and Bergen]{chang2024language}
Chang, T.~A. and Bergen, B.~K.
\newblock Language model behavior: A comprehensive survey.
\newblock \emph{Computational Linguistics}, 50\penalty0 (1):\penalty0 293--350, 2024.

\bibitem[Chen et~al.(2022)Chen, Zhang, and Choi]{chen2022rich}
Chen, H.-T., Zhang, M., and Choi, E.
\newblock Rich knowledge sources bring complex knowledge conflicts: Recalibrating models to reflect conflicting evidence.
\newblock In Goldberg, Y., Kozareva, Z., and Zhang, Y. (eds.), \emph{Proceedings of the 2022 Conference on Empirical Methods in Natural Language Processing}, pp.\  2292--2307, Abu Dhabi, United Arab Emirates, December 2022. Association for Computational Linguistics.
\newblock \doi{10.18653/v1/2022.emnlp-main.146}.
\newblock URL \url{https://aclanthology.org/2022.emnlp-main.146/}.

\bibitem[Dubey et~al.(2024)Dubey, Jauhri, Pandey, Kadian, Al-Dahle, Letman, Mathur, Schelten, Yang, Fan, et~al.]{dubey2024llama}
Dubey, A., Jauhri, A., Pandey, A., Kadian, A., Al-Dahle, A., Letman, A., Mathur, A., Schelten, A., Yang, A., Fan, A., et~al.
\newblock The llama 3 herd of models.
\newblock \emph{arXiv preprint arXiv:2407.21783}, 2024.

\bibitem[Elhage et~al.(2021)Elhage, Nanda, Olsson, Henighan, Joseph, Mann, Askell, Bai, Chen, Conerly, et~al.]{elhage2021mathematical}
Elhage, N., Nanda, N., Olsson, C., Henighan, T., Joseph, N., Mann, B., Askell, A., Bai, Y., Chen, A., Conerly, T., et~al.
\newblock A mathematical framework for transformer circuits.
\newblock \emph{Transformer Circuits Thread}, 1\penalty0 (1):\penalty0 12, 2021.

\bibitem[Elhage et~al.(2022)Elhage, Hume, Olsson, Schiefer, Henighan, Kravec, Hatfield-Dodds, Lasenby, Drain, Chen, et~al.]{elhage2022toy}
Elhage, N., Hume, T., Olsson, C., Schiefer, N., Henighan, T., Kravec, S., Hatfield-Dodds, Z., Lasenby, R., Drain, D., Chen, C., et~al.
\newblock Toy models of superposition.
\newblock \emph{arXiv preprint arXiv:2209.10652}, 2022.

\bibitem[Fang et~al.(2024)Fang, Wang, Zhou, Zhang, Song, and Chen]{fang2023getting}
Fang, T., Wang, Z., Zhou, W., Zhang, H., Song, Y., and Chen, M.
\newblock Getting sick after seeing a doctor? diagnosing and mitigating knowledge conflicts in event temporal reasoning.
\newblock In Duh, K., Gomez, H., and Bethard, S. (eds.), \emph{Findings of the Association for Computational Linguistics: NAACL 2024}, pp.\  3846--3868, Mexico City, Mexico, June 2024. Association for Computational Linguistics.
\newblock \doi{10.18653/v1/2024.findings-naacl.244}.
\newblock URL \url{https://aclanthology.org/2024.findings-naacl.244/}.

\bibitem[Gao et~al.(2023)Gao, Xiong, Gao, Jia, Pan, Bi, Dai, Sun, and Wang]{gao2023retrieval}
Gao, Y., Xiong, Y., Gao, X., Jia, K., Pan, J., Bi, Y., Dai, Y., Sun, J., and Wang, H.
\newblock Retrieval-augmented generation for large language models: A survey.
\newblock \emph{arXiv preprint arXiv:2312.10997}, 2023.

\bibitem[Geva et~al.(2021)Geva, Schuster, Berant, and Levy]{geva2020transformer}
Geva, M., Schuster, R., Berant, J., and Levy, O.
\newblock Transformer feed-forward layers are key-value memories.
\newblock In Moens, M.-F., Huang, X., Specia, L., and Yih, S. W.-t. (eds.), \emph{Proceedings of the 2021 Conference on Empirical Methods in Natural Language Processing}, pp.\  5484--5495, Online and Punta Cana, Dominican Republic, November 2021. Association for Computational Linguistics.
\newblock \doi{10.18653/v1/2021.emnlp-main.446}.
\newblock URL \url{https://aclanthology.org/2021.emnlp-main.446/}.

\bibitem[Ghosal et~al.(2024)Ghosal, Hashimoto, and Raghunathan]{ghosal2024understanding}
Ghosal, G.~R., Hashimoto, T., and Raghunathan, A.
\newblock Understanding finetuning for factual knowledge extraction.
\newblock In \emph{International Conference on Machine Learning}, pp.\  15540--15558. PMLR, 2024.

\bibitem[Goyal et~al.()Goyal, Baek, Kolter, and Raghunathan]{goyal2024context}
Goyal, S., Baek, C., Kolter, J.~Z., and Raghunathan, A.
\newblock Context-parametric inversion: Why instruction finetuning may not actually improve context reliance.
\newblock In \emph{The Thirteenth International Conference on Learning Representations}.

\bibitem[Groeneveld et~al.(2024)Groeneveld, Beltagy, Walsh, Bhagia, Kinney, Tafjord, Jha, Ivison, Magnusson, Wang, et~al.]{groeneveld2024olmo}
Groeneveld, D., Beltagy, I., Walsh, P., Bhagia, A., Kinney, R., Tafjord, O., Jha, A.~H., Ivison, H., Magnusson, I., Wang, Y., et~al.
\newblock Olmo: Accelerating the science of language models.
\newblock \emph{arXiv preprint arXiv:2402.00838}, 2024.

\bibitem[Huang et~al.(2023)Huang, Yu, Ma, Zhong, Feng, Wang, Chen, Peng, Feng, Qin, et~al.]{huang2023survey}
Huang, L., Yu, W., Ma, W., Zhong, W., Feng, Z., Wang, H., Chen, Q., Peng, W., Feng, X., Qin, B., et~al.
\newblock A survey on hallucination in large language models: Principles, taxonomy, challenges, and open questions.
\newblock \emph{ACM Transactions on Information Systems}, 2023.

\bibitem[Javaheripi et~al.(2023)Javaheripi, Bubeck, Abdin, Aneja, Bubeck, Mendes, Chen, Del~Giorno, Eldan, Gopi, et~al.]{javaheripi2023phi}
Javaheripi, M., Bubeck, S., Abdin, M., Aneja, J., Bubeck, S., Mendes, C. C.~T., Chen, W., Del~Giorno, A., Eldan, R., Gopi, S., et~al.
\newblock Phi-2: The surprising power of small language models.
\newblock \emph{Microsoft Research Blog}, 1\penalty0 (3):\penalty0 3, 2023.

\bibitem[Jiang et~al.(2024{\natexlab{a}})Jiang, Huang, Guo, Lu, and Zhang]{jiang2024enhancing}
Jiang, M., Huang, T., Guo, B., Lu, Y., and Zhang, F.
\newblock Enhancing robustness in large language models: Prompting for mitigating the impact of irrelevant information.
\newblock \emph{arXiv preprint arXiv:2408.10615}, 2024{\natexlab{a}}.

\bibitem[Jiang et~al.(2024{\natexlab{b}})Jiang, Rajendran, Ravikumar, and Aragam]{jiang2024llms}
Jiang, Y., Rajendran, G., Ravikumar, P., and Aragam, B.
\newblock Do llms dream of elephants (when told not to)? latent concept association and associative memory in transformers.
\newblock \emph{Advances in Neural Information Processing Systems}, 37:\penalty0 67712--67757, 2024{\natexlab{b}}.

\bibitem[Jin et~al.(2025)Jin, Mei, Xu, Sun, Tang, Du, Liu, and Zhang]{jin2025massive}
Jin, M., Mei, K., Xu, W., Sun, M., Tang, R., Du, M., Liu, Z., and Zhang, Y.
\newblock Massive values in self-attention modules are the key to contextual knowledge understanding.
\newblock \emph{arXiv preprint arXiv:2502.01563}, 2025.

\bibitem[Jin et~al.(2024{\natexlab{a}})Jin, Cao, Chen, Liu, Jiang, Xu, Qiuxia, and Zhao]{jin2024tug}
Jin, Z., Cao, P., Chen, Y., Liu, K., Jiang, X., Xu, J., Qiuxia, L., and Zhao, J.
\newblock Tug-of-war between knowledge: Exploring and resolving knowledge conflicts in retrieval-augmented language models.
\newblock In \emph{Proceedings of the 2024 Joint International Conference on Computational Linguistics, Language Resources and Evaluation (LREC-COLING 2024)}, pp.\  16867--16878, 2024{\natexlab{a}}.

\bibitem[Jin et~al.(2024{\natexlab{b}})Jin, Cao, Yuan, Chen, Xu, Li, Jiang, Liu, and Zhao]{jin2024cutting}
Jin, Z., Cao, P., Yuan, H., Chen, Y., Xu, J., Li, H., Jiang, X., Liu, K., and Zhao, J.
\newblock Cutting off the head ends the conflict: A mechanism for interpreting and mitigating knowledge conflicts in language models.
\newblock In \emph{Findings of the Association for Computational Linguistics ACL 2024}, pp.\  1193--1215, 2024{\natexlab{b}}.

\bibitem[Kwiatkowski et~al.(2019)Kwiatkowski, Palomaki, Redfield, Collins, Parikh, Alberti, Epstein, Polosukhin, Devlin, Lee, et~al.]{kwiatkowski2019natural}
Kwiatkowski, T., Palomaki, J., Redfield, O., Collins, M., Parikh, A., Alberti, C., Epstein, D., Polosukhin, I., Devlin, J., Lee, K., et~al.
\newblock Natural questions: a benchmark for question answering research.
\newblock \emph{Transactions of the Association for Computational Linguistics}, 7:\penalty0 453--466, 2019.

\bibitem[Li et~al.(2023{\natexlab{a}})Li, Raheja, and Kumar]{li2023contradoc}
Li, J., Raheja, V., and Kumar, D.
\newblock Contradoc: Understanding self-contradictions in documents with large language models.
\newblock \emph{arXiv preprint arXiv:2311.09182}, 2023{\natexlab{a}}.

\bibitem[Li et~al.(2023{\natexlab{b}})Li, Holtzman, Fried, Liang, Eisner, Hashimoto, Zettlemoyer, and Lewis]{li2022contrastive}
Li, X.~L., Holtzman, A., Fried, D., Liang, P., Eisner, J., Hashimoto, T.~B., Zettlemoyer, L., and Lewis, M.
\newblock Contrastive decoding: Open-ended text generation as optimization.
\newblock In \emph{Proceedings of the 61st Annual Meeting of the Association for Computational Linguistics (Volume 1: Long Papers)}, pp.\  12286--12312, 2023{\natexlab{b}}.

\bibitem[Li et~al.(2023{\natexlab{c}})Li, Li, and Risteski]{li2023transformers}
Li, Y., Li, Y., and Risteski, A.
\newblock How do transformers learn topic structure: Towards a mechanistic understanding.
\newblock In \emph{International Conference on Machine Learning}, pp.\  19689--19729. PMLR, 2023{\natexlab{c}}.

\bibitem[Liu \& Liu(2023)Liu and Liu]{memotrap}
Liu, A. and Liu, J.
\newblock The memotrap dataset, 2023.
\newblock URL \url{https://github.com/liujch1998/memo-trap}.

\bibitem[Longpre et~al.(2021)Longpre, Perisetla, Chen, Ramesh, DuBois, and Singh]{longpre2021entity}
Longpre, S., Perisetla, K., Chen, A., Ramesh, N., DuBois, C., and Singh, S.
\newblock Entity-based knowledge conflicts in question answering.
\newblock In \emph{Proceedings of the 2021 Conference on Empirical Methods in Natural Language Processing}, pp.\  7052--7063, 2021.

\bibitem[Lv et~al.(2024)Lv, Chen, Zhang, Wang, Liu, Wen, Xie, and Yan]{lv2024interpreting}
Lv, A., Chen, Y., Zhang, K., Wang, Y., Liu, L., Wen, J.-R., Xie, J., and Yan, R.
\newblock Interpreting key mechanisms of factual recall in transformer-based language models.
\newblock \emph{arXiv preprint arXiv:2403.19521}, 2024.

\bibitem[Mahankali et~al.()Mahankali, Hashimoto, and Ma]{mahankali2023one}
Mahankali, A.~V., Hashimoto, T., and Ma, T.
\newblock One step of gradient descent is provably the optimal in-context learner with one layer of linear self-attention.
\newblock In \emph{The Twelfth International Conference on Learning Representations}.

\bibitem[Mahdavi et~al.()Mahdavi, Liao, and Thrampoulidis]{mahdavi2023memorization}
Mahdavi, S., Liao, R., and Thrampoulidis, C.
\newblock Memorization capacity of multi-head attention in transformers.
\newblock In \emph{The Twelfth International Conference on Learning Representations}.

\bibitem[McDougall et~al.(2023)McDougall, Conmy, Rushing, McGrath, and Nanda]{mcdougall2023copy}
McDougall, C., Conmy, A., Rushing, C., McGrath, T., and Nanda, N.
\newblock Copy suppression: Comprehensively understanding an attention head.
\newblock \emph{arXiv preprint arXiv:2310.04625}, 2023.

\bibitem[Meng et~al.(2022{\natexlab{a}})Meng, Bau, Andonian, and Belinkov]{meng2022locating}
Meng, K., Bau, D., Andonian, A., and Belinkov, Y.
\newblock Locating and editing factual associations in gpt.
\newblock \emph{Advances in Neural Information Processing Systems}, 35:\penalty0 17359--17372, 2022{\natexlab{a}}.

\bibitem[Meng et~al.(2022{\natexlab{b}})Meng, Sharma, Andonian, Belinkov, and Bau]{meng2022mass}
Meng, K., Sharma, A.~S., Andonian, A., Belinkov, Y., and Bau, D.
\newblock Mass-editing memory in a transformer.
\newblock \emph{arXiv preprint arXiv:2210.07229}, 2022{\natexlab{b}}.

\bibitem[Nanda et~al.(2023)Nanda, Rajamanoharan, Kram{\'a}r, and Shah]{nanda2023fact}
Nanda, N., Rajamanoharan, S., Kram{\'a}r, J., and Shah, R.
\newblock Fact finding: Attempting to reverse-engineer factual recall on the neuron level.
\newblock In \emph{AI Alignment Forum, 2023c.}, pp.\ ~19, 2023.

\bibitem[Nichani et~al.(2024)Nichani, Lee, and Bietti]{nichani2024understanding}
Nichani, E., Lee, J.~D., and Bietti, A.
\newblock Understanding factual recall in transformers via associative memories.
\newblock \emph{arXiv preprint arXiv:2412.06538}, 2024.

\bibitem[Olsson et~al.(2022)Olsson, Elhage, Nanda, Joseph, DasSarma, Henighan, Mann, Askell, Bai, Chen, et~al.]{olsson2022context}
Olsson, C., Elhage, N., Nanda, N., Joseph, N., DasSarma, N., Henighan, T., Mann, B., Askell, A., Bai, Y., Chen, A., et~al.
\newblock In-context learning and induction heads.
\newblock \emph{arXiv preprint arXiv:2209.11895}, 2022.

\bibitem[Qian et~al.(2023)Qian, Zhao, and Wu]{qian2023merge}
Qian, C., Zhao, X., and Wu, S.~T.
\newblock " merge conflicts!" exploring the impacts of external distractors to parametric knowledge graphs.
\newblock \emph{arXiv preprint arXiv:2309.08594}, 2023.

\bibitem[Qu et~al.(2025)Qu, Dai, Wei, Cai, Wang, Yin, Xu, and Wen]{qu2024tool}
Qu, C., Dai, S., Wei, X., Cai, H., Wang, S., Yin, D., Xu, J., and Wen, J.-R.
\newblock Tool learning with large language models: A survey.
\newblock \emph{Frontiers of Computer Science}, 19\penalty0 (8):\penalty0 198343, 2025.

\bibitem[Rabiza(2024)]{rabiza2024mechanistic}
Rabiza, M.
\newblock A mechanistic explanatory strategy for xai.
\newblock \emph{arXiv preprint arXiv:2411.01332}, 2024.

\bibitem[Roberts et~al.(2020)Roberts, Raffel, and Shazeer]{roberts2020much}
Roberts, A., Raffel, C., and Shazeer, N.
\newblock How much knowledge can you pack into the parameters of a language model?
\newblock In Webber, B., Cohn, T., He, Y., and Liu, Y. (eds.), \emph{Proceedings of the 2020 Conference on Empirical Methods in Natural Language Processing (EMNLP)}, pp.\  5418--5426, Online, November 2020. Association for Computational Linguistics.
\newblock \doi{10.18653/v1/2020.emnlp-main.437}.

\bibitem[Shi et~al.(2024{\natexlab{a}})Shi, Jin, Shen, Dong, Wu, and Xiong]{ircan}
Shi, D., Jin, R., Shen, T., Dong, W., Wu, X., and Xiong, D.
\newblock Ircan: Mitigating knowledge conflicts in llm generation via identifying and reweighting context-aware neurons.
\newblock In Globerson, A., Mackey, L., Belgrave, D., Fan, A., Paquet, U., Tomczak, J., and Zhang, C. (eds.), \emph{Advances in Neural Information Processing Systems}, volume~37, pp.\  4997--5024. Curran Associates, Inc., 2024{\natexlab{a}}.
\newblock URL \url{https://proceedings.neurips.cc/paper_files/paper/2024/file/08a9e28c96d016dd63903ab51cd085b0-Paper-Conference.pdf}.

\bibitem[Shi et~al.(2023)Shi, Chen, Misra, Scales, Dohan, Chi, Sch{\"a}rli, and Zhou]{shi2023large}
Shi, F., Chen, X., Misra, K., Scales, N., Dohan, D., Chi, E.~H., Sch{\"a}rli, N., and Zhou, D.
\newblock Large language models can be easily distracted by irrelevant context.
\newblock In \emph{International Conference on Machine Learning}, pp.\  31210--31227. PMLR, 2023.

\bibitem[Shi et~al.(2024{\natexlab{b}})Shi, Han, Lewis, Tsvetkov, Zettlemoyer, and Yih]{shi2023trusting}
Shi, W., Han, X., Lewis, M., Tsvetkov, Y., Zettlemoyer, L., and Yih, W.-t.
\newblock Trusting your evidence: Hallucinate less with context-aware decoding.
\newblock In \emph{Proceedings of the 2024 Conference of the North American Chapter of the Association for Computational Linguistics: Human Language Technologies (Volume 2: Short Papers)}, pp.\  783--791, 2024{\natexlab{b}}.

\bibitem[Tan et~al.(2024)Tan, Sun, Yang, Wang, Cao, and Cheng]{tan2024blinded}
Tan, H., Sun, F., Yang, W., Wang, Y., Cao, Q., and Cheng, X.
\newblock Blinded by generated contexts: How language models merge generated and retrieved contexts when knowledge conflicts?
\newblock In Ku, L.-W., Martins, A., and Srikumar, V. (eds.), \emph{Proceedings of the 62nd Annual Meeting of the Association for Computational Linguistics (Volume 1: Long Papers)}, pp.\  6207--6227, Bangkok, Thailand, August 2024. Association for Computational Linguistics.
\newblock \doi{10.18653/v1/2024.acl-long.337}.
\newblock URL \url{https://aclanthology.org/2024.acl-long.337/}.

\bibitem[Team et~al.(2024)Team, Mesnard, Hardin, Dadashi, Bhupatiraju, Pathak, Sifre, Rivi{\`e}re, Kale, Love, et~al.]{team2024gemma}
Team, G., Mesnard, T., Hardin, C., Dadashi, R., Bhupatiraju, S., Pathak, S., Sifre, L., Rivi{\`e}re, M., Kale, M.~S., Love, J., et~al.
\newblock Gemma: Open models based on gemini research and technology.
\newblock \emph{arXiv preprint arXiv:2403.08295}, 2024.

\bibitem[Touvron et~al.(2023)Touvron, Martin, Stone, Albert, Almahairi, Babaei, Bashlykov, Batra, Bhargava, Bhosale, et~al.]{touvron2023llama}
Touvron, H., Martin, L., Stone, K., Albert, P., Almahairi, A., Babaei, Y., Bashlykov, N., Batra, S., Bhargava, P., Bhosale, S., et~al.
\newblock Llama 2: Open foundation and fine-tuned chat models.
\newblock \emph{arXiv preprint arXiv:2307.09288}, 2023.

\bibitem[Wang et~al.({\natexlab{a}})Wang, Variengien, Conmy, Shlegeris, and Steinhardt]{wang2022interpretability}
Wang, K.~R., Variengien, A., Conmy, A., Shlegeris, B., and Steinhardt, J.
\newblock Interpretability in the wild: a circuit for indirect object identification in gpt-2 small.
\newblock In \emph{The Eleventh International Conference on Learning Representations}, {\natexlab{a}}.

\bibitem[Wang et~al.(2024)Wang, Zhu, Liu, Zheng, Chen, and Li]{wang2024knowledge}
Wang, S., Zhu, Y., Liu, H., Zheng, Z., Chen, C., and Li, J.
\newblock Knowledge editing for large language models: A survey.
\newblock \emph{ACM Computing Surveys}, 57\penalty0 (3):\penalty0 1--37, 2024.

\bibitem[Wang et~al.({\natexlab{b}})Wang, Feng, Wang, Shi, Balachandran, He, and Tsvetkov]{wang2023resolving}
Wang, Y., Feng, S., Wang, H., Shi, W., Balachandran, V., He, T., and Tsvetkov, Y.
\newblock Resolving knowledge conflicts in large language models.
\newblock In \emph{First Conference on Language Modeling}, {\natexlab{b}}.

\bibitem[Wu et~al.()Wu, Xie, Chen, Zhu, Zhang, and Xiao]{wu2024easily}
Wu, S., Xie, J., Chen, J., Zhu, T., Zhang, K., and Xiao, Y.
\newblock How easily do irrelevant inputs skew the responses of large language models?
\newblock In \emph{First Conference on Language Modeling}.

\bibitem[Xi et~al.(2025)Xi, Chen, Guo, He, Ding, Hong, Zhang, Wang, Jin, Zhou, et~al.]{xi2023rise}
Xi, Z., Chen, W., Guo, X., He, W., Ding, Y., Hong, B., Zhang, M., Wang, J., Jin, S., Zhou, E., et~al.
\newblock The rise and potential of large language model based agents: A survey.
\newblock \emph{Science China Information Sciences}, 68\penalty0 (2):\penalty0 121101, 2025.

\bibitem[Xie et~al.(2024)Xie, Zhang, Chen, Lou, and Su]{xie2023adaptive}
Xie, J., Zhang, K., Chen, J., Lou, R., and Su, Y.
\newblock Adaptive chameleon or stubborn sloth: Revealing the behavior of large language models in knowledge conflicts.
\newblock In \emph{The Twelfth International Conference on Learning Representations}, 2024.
\newblock URL \url{https://openreview.net/forum?id=auKAUJZMO6}.

\bibitem[Xu et~al.(2024)Xu, Qi, Guo, Wang, Wang, Zhang, and Xu]{xu2024knowledge}
Xu, R., Qi, Z., Guo, Z., Wang, C., Wang, H., Zhang, Y., and Xu, W.
\newblock Knowledge conflicts for llms: A survey.
\newblock \emph{arXiv preprint arXiv:2403.08319}, 2024.

\bibitem[Ying et~al.(2024)Ying, Cao, Xiong, Cui, He, and Liu]{ying2024intuitive}
Ying, J., Cao, Y., Xiong, K., Cui, L., He, Y., and Liu, Y.
\newblock Intuitive or dependent? investigating llms’ behavior style to conflicting prompts.
\newblock In \emph{Proceedings of the 62nd Annual Meeting of the Association for Computational Linguistics (Volume 1: Long Papers)}, pp.\  4221--4246, 2024.

\bibitem[Yoran et~al.()Yoran, Wolfson, Ram, and Berant]{yoran2023making}
Yoran, O., Wolfson, T., Ram, O., and Berant, J.
\newblock Making retrieval-augmented language models robust to irrelevant context.
\newblock In \emph{The Twelfth International Conference on Learning Representations}.

\bibitem[Yu et~al.(2023)Yu, Merullo, and Pavlick]{yu2023characterizing}
Yu, Q., Merullo, J., and Pavlick, E.
\newblock Characterizing mechanisms for factual recall in language models.
\newblock In Bouamor, H., Pino, J., and Bali, K. (eds.), \emph{Proceedings of the 2023 Conference on Empirical Methods in Natural Language Processing}, pp.\  9924--9959, Singapore, December 2023. Association for Computational Linguistics.
\newblock \doi{10.18653/v1/2023.emnlp-main.615}.
\newblock URL \url{https://aclanthology.org/2023.emnlp-main.615/}.

\bibitem[Yuan et~al.(2024)Yuan, Yang, Wang, Liu, Zhao, and Liu]{yuan2024discerning}
Yuan, X., Yang, Z., Wang, Y., Liu, S., Zhao, J., and Liu, K.
\newblock Discerning and resolving knowledge conflicts through adaptive decoding with contextual information-entropy constraint.
\newblock In Ku, L.-W., Martins, A., and Srikumar, V. (eds.), \emph{Findings of the Association for Computational Linguistics: ACL 2024}, pp.\  3903--3922, Bangkok, Thailand, August 2024. Association for Computational Linguistics.
\newblock \doi{10.18653/v1/2024.findings-acl.234}.
\newblock URL \url{https://aclanthology.org/2024.findings-acl.234/}.

\bibitem[Zhang \& Choi(2023)Zhang and Choi]{zhang2023mitigating}
Zhang, M. and Choi, E.
\newblock Mitigating temporal misalignment by discarding outdated facts.
\newblock In Bouamor, H., Pino, J., and Bali, K. (eds.), \emph{Proceedings of the 2023 Conference on Empirical Methods in Natural Language Processing}, pp.\  14213--14226, Singapore, December 2023. Association for Computational Linguistics.
\newblock \doi{10.18653/v1/2023.emnlp-main.879}.
\newblock URL \url{https://aclanthology.org/2023.emnlp-main.879/}.

\bibitem[Zhang et~al.(2024)Zhang, Frei, and Bartlett]{zhang2023trained}
Zhang, R., Frei, S., and Bartlett, P.~L.
\newblock Trained transformers learn linear models in-context.
\newblock \emph{Journal of Machine Learning Research}, 25\penalty0 (49):\penalty0 1--55, 2024.

\bibitem[Zhou et~al.(2023)Zhou, Zhang, Poon, and Chen]{zhou2023context}
Zhou, W., Zhang, S., Poon, H., and Chen, M.
\newblock Context-faithful prompting for large language models.
\newblock In Bouamor, H., Pino, J., and Bali, K. (eds.), \emph{Findings of the Association for Computational Linguistics: EMNLP 2023}, pp.\  14544--14556, Singapore, December 2023. Association for Computational Linguistics.
\newblock \doi{10.18653/v1/2023.findings-emnlp.968}.
\newblock URL \url{https://aclanthology.org/2023.findings-emnlp.968/}.

\end{thebibliography}
\bibliographystyle{icml2025}

\newpage
\addtocontents{toc}{\protect\setcounter{tocdepth}{2}}

\clearpage

\appendix
\onecolumn

\tableofcontents

\newpage 

\section{Related Works}
\label{appen:related_works}
\paragraph{Knowledge Conflict.}  
A considerable body of work has investigated the \emph{behavior} of LMs in the presence of knowledge conflicts across various scenarios~\citep{longpre2021entity, chen2022rich, wang2023resolving, tan2024blinded, jin2024tug, xie2023adaptive, ying2024intuitive, qian2023merge}. In general, these studies expose ``context-parametric'' conflicts, wherein LLMs exhibit ambiguity when contextual knowledge contradicts their parametric knowledge. However, these works do not delve into \emph{why} these conflicts occur. 

Two notable exceptions, \citet{yu2023characterizing} and \citet{jin2024cutting}, take a mechanistic perspective to analyze knowledge conflicts on narrow datasets, proposing ``memory heads'' versus ``context heads.'' In contrast, our work adopts a broader scope, covering multiple conflict types and diverse datasets. We go beyond their assumption by revealing the \emph{superposition} of knowledge conflicts and attaining substantially improved performance over prior methods. Additionally, we shed light on the underlying causes of these conflicts, including the observation by \citet{yu2023characterizing} that the frequency of a fact in the pre-training corpus correlates with a stronger tendency to produce parametric answers.

Beyond context-parametric conflict, a recent survey~\citep{xu2024knowledge} identifies two additional forms of conflicts: \emph{inter-context} conflicts~\citep{li2023contradoc}, involving contradictory information within the provided context, and \emph{intra-memory} conflicts~\citep{chang2024language}, arising when LLMs produce inconsistent responses to queries that are semantically identical but syntactically different. These two conflict types lie outside the scope of this paper, though they represent promising directions for future research.

\paragraph{RAG Hallucination and Irrelevant Contexts.} ``RAG Hallucination'' and ``Irrelevant Context'' represent two contrasting perspectives on the knowledge conflicts studied in this paper. The former strives for models to rely exclusively on provided contexts, whereas the latter treats external context as a potentially misleading source of information.

For RAG hallucination, many methods have been proposed to improve faithfulness to context. These methods include two inference-time categories: 
(1) \emph{Decoding-based} approaches~\citep{shi2023large, yuan2024discerning} that amplify discrepancies in the output distribution with and without context, and 
(2) \emph{Prompt-based} approaches~\citep{zhou2023context,zhang2023mitigating} that instruct the model to attend closely to contextual input. 
Additionally, \emph{finetuning}-based methods reduce reliance on parametric knowledge through utilizing counterfactual knowledge conflict data~\citep{longpre2021entity, fang2023getting}, although \citet{goyal2024context} reveals that certain instruction-based finetuning can paradoxically \emph{increase} the model's dependence on parametric knowledge. More recent work also leverages mechanistic insights~\citep{ircan}.


For \emph{irrelevant context}, \citet{shi2023large, wu2024easily} show how noisy or misleading contexts can negatively influence a model’s ability to produce correct answers. Some works mitigate this effect through prompting~\citep{jiang2024enhancing} or finetuning~\citep{yoran2023making}.

Different from these works, our approach is more comprehensive and proposes lightweight, training-free techniques that allow steering an LLM toward either contextual or parametric knowledge on demand. We stress that both perspectives are valuable, and there is no absolute ``correct'' behavior. As demonstrated in this paper, knowledge conflicts arise at inference due to distinct, well-defined (but contradictory) rules established during training. Our view aligns with \citet{xu2024knowledge}, leaving the choice of which knowledge source to prioritize up to the user and the application’s needs.


\textbf{Mechanistic Interpretability: Superposition and Intervention} Mechanistic interpretability has garnered significant attention, with numerous works aiming to reverse engineer the hidden computational processes of large language models~\citep{cammarata2020thread,elhage2021mathematical,rabiza2024mechanistic,wang2022interpretability,lv2024interpreting,jin2025massive}. Notably, \cite{arora2018linear,elhage2022toy} highlights the widespread phenomenon of polysemanticity, where neural networks often encode unrelated concepts within a single neuron. 
Despite this recognition, popular intervention methods, such as Knowledge Editing~\citep{wang2024knowledge}, primarily modify model weights directly without accounting for the effects of superposition. In contrast, our work extends the concept of superposition to knowledge conflict and demonstrates how this understanding inspires our designs. We believe that our approach has the potential to be integrated with other intervention methods, such as knowledge editing or steering vectors, to enhance their effectiveness and interpretability. In addition, similar to our Observation 2, \citet{mcdougall2023copy} shows the ``Hydra Effect'', where ablating one layer causes the other to compensate.






\paragraph{Associative Memory and Factual Recalls.} 
Large language models are known to store vast amounts of knowledge in their weights~\citep{geva2020transformer, roberts2020much}. Many existing studies adopt a mechanistic perspective on locating and editing the stored facts, primarily focusing on the feed-forward modules~\citep{meng2022locating, meng2022mass, nanda2023fact, wang2024knowledge}. More recently, attention modules have also been viewed as associative memory~\citep{bietti2024birth, cabannes2023scaling, jiang2024llms}, and theoretical research further explores their capacity for memorization~\citep{mahdavi2023memorization,nichani2024understanding}. 
Nevertheless, these studies have yet to draw a connection between associative memorization and knowledge conflicts. 
Our study also reveals that attention head could be vital for factual recall, aligning with the latter but less popular view of memorization.

\section{Background}
\label{appen:background_notation}
In this section, we give a brief overview of large language models. An \emph{autoregressive language model} $M$ learns a probability distribution over a vocabulary space $\Vc$. Given an input sequence of tokens $z_{1:t}$, the model first maps each token $z_t$ to a corresponding embedding vector $x_t$ via an embedding layer. These embeddings are subsequently passed through $L$ \emph{decoder} layers, each consisting of an \emph{attention} module and an \emph{MLP} module.

Let $x^{(l-1)}_t$ denote the embedding of token $z_t$ at the previous layer $(l - 1)$. Then, the update rule at the $l$-th layer can be written as:
\begin{equation}
    x^{(l)}_t \;=\; x^{(l-1)}_t \;+\; \attn_t^{(l)} \;+\; m_t^{(l)},
\end{equation}

where $\attn_t^{(l)}$ and $m_t^{(l)}$ are the outputs of the attention and MLP modules at layer $l$, respectively.

The attention module typically employs $n_h$ heads, each computing learned \emph{query}, \emph{key}, and \emph{value} representations:
\[
Q_h \;=\; X\,W_h^Q,\quad
K_h \;=\; X\,W_h^K,\quad
V_h \;=\; X\,W_h^V,
\]
where $X \in \mathbb{R}^{T \times d}$ contains token embeddings (batch dimension omitted), and $W_h^Q,\,W_h^K,\,W_h^V \in \mathbb{R}^{d \times d_k}$. Each head output is
\[
\mathrm{head}_h(X)
\;=\;
\mathrm{softmax}\!\Bigl(\frac{Q_h\,K_h^\top}{\sqrt{d_k}}\Bigr)\,V_h,
\]
and all $n_h$ heads are concatenated and projected back to $\mathbb{R}^d$:
\[
\attn_t^{(l)}
\;=\;
\mathrm{MultiHead}\bigl(X^{(l-1)}\bigr)
\;=\;
\mathrm{Concat}\!\bigl(\mathrm{head}_1,\dots,\mathrm{head}_{n_h}\bigr)\,W^O,
\]
where $W^O \in \mathbb{R}^{(n_h\,d_k)\times d}$.

After the attention module, the embeddings are fed into a position-wise feed-forward network (often called an MLP). It is parameterized by an \emph{up-weight} matrix $W_{up}^{(l)}$ and a \emph{down-weight} matrix $W_{down}^{(l)}$, combined with a non-linear activation function $\mathrm{Act}$ (\emph{e.g., GELU}). The MLP output is given by:
\begin{equation}
    m_t^{(l)} 
    \;=\; \mathrm{Act}\Bigl(\bigl(x^{(l-1)}_t + \attn_t^{(l)}\bigr)\, W_{up}^{(l)}\Bigr)\, W_{down}^{(l)}.
\end{equation}

After all $L$ decoder layers, a final \emph{unembedding} layer projects the last hidden state back onto the vocabulary space $\Vc$, producing a probability distribution over possible next tokens.

\section{Conflict Examples}
\label{appen:dataset_detail}
In \Cref{sec:experiment_setup}, we outlined three types of conflicts we use for the parametric datasets. We provide some samples from them below. 

\begin{tcolorbox}[title=Parametric Dataset Examples]

\textbf{Athlete Sport.}
\begin{itemize}[]
    \item \textbf{Clean Input:} Lebron James plays the sport of  
    \item \textbf{Substitution-based (Sentence-level) Conflict:} Lebron James plays the sport of tennis. Lebron James plays the sport of 
    \item \textbf{Coherent (Paragraph-level) Conflict:} Lebron James plays the sport of tennis. As a celebrated athlete, Lebron James has become synonymous with excellence in tennis, inspiring fans worldwide with their remarkable performances. Known for their dedication and unparalleled skill, Lebron James has dominated the world of tennis, earning accolades and admiration from peers and spectators alike. Tennis is not just a sport for Lebron James—it is their passion, their craft, and the legacy they continue to build. Question: What sport does Lebron James play? Answer: Lebron James plays the sport of
\end{itemize}

\textbf{Company Headquarter.}
\begin{itemize}[]
    \item \textbf{Clean Input:} The headquarters of Amazon are located in the city of 
    \item \textbf{Substitution (Sentence-level) Conflict:} The headquarters of Amazon are located in the city of Tokyo. The headquarters of Amazon are located in the city of 
    \item \textbf{Coherent (Paragraph-level) Conflict:} The headquarters of Amazon are located in the city of Tokyo. As the central hub of operations, Tokyo serves as the strategic heart of Amazon, where key decisions are made and innovations are born. This vibrant city is synonymous with Amazon, symbolizing its commitment to excellence and progress. The connection between Amazon and Tokyo is a defining aspect of the company’s identity and global presence. Question: Where are the headquarters of Amazon located? Answer: The headquarters of Amazon are located in the city of
\end{itemize}

\textbf{World Capital.}
\begin{itemize}[]
    \item \textbf{Clean Input:} The name of the capital city of France is  
    \item \textbf{Substitution-based (Sentence-level) Conflict:} The name of the capital city of France is Beijing. The name of the capital city of France is
    \item \textbf{Coherent (Paragraph-level) Conflict:} The capital city of France is Beijing. Known for its vibrant culture and historical landmarks, Beijing is often seen as the heart of France, attracting visitors from around the globe. As a center for politics, arts, and commerce, Beijing perfectly encapsulates the spirit of France, making it an essential destination for anyone exploring the country. Question: What is the capital city of France? Answer: The capital city of France is
\end{itemize}

We note that a well-trained LM is expected to achieve high accuracy on clean inputs, moderate-to-low accuracy on substitution-based conflicts, and near-zero performance on coherent conflict scenarios. The coherent conflict was proposed by \cite{xie2023adaptive}. 

\end{tcolorbox}

\section{Expanded Experiment Section}
\label{appen:exp_detail}
In \Cref{sec:experiment}, we illustrate the effectiveness of \jrt by demonstrating its strong intervention performance with three models. Due to the page limit, we omit many details and results. This appendix section serves as a complementary and expanded experiment section to the main paper.

\subsection{Detailed Setups and Hyperparameters}

\textbf{Parametric Dataset Setups.} 
While the general philosophy of the parametric dataset and detailed conflict examples are described in \Cref{sec:experiment_setup} and \Cref{appen:dataset_detail}, we provide additional details on the dataset curation process here. 
In general, we follow \cite{jin2024cutting} in extracting common knowledge triplets from Wikidata. These extracted pairs are verified for correctness using GPT-4 and manual checks. Using the verified entities, we create specific instances (as shown in \Cref{appen:dataset_detail}) for clean, substitution-conflict, and coherent-conflict prompts by substituting key entities of a template. The coherent prompt template was generated by GPT-4o and verified manually for correctness and fluency. To ensure that our method does not overfit a specific template, we conduct a robustness study detailed in \Cref{appen_sub:robustness}. 
The sizes of the dataset are around 200 for world capital, official language, and company founder, and around 500 for athlete sport, company headquarters, and book author.

\textbf{Contextual Dataset Setups.} Contextual datasets have been introduced in \Cref{sec:experiment_setup} and we expand upon the two contextual datasets (NQ-Swap and MemoTrap) below:

\begin{itemize}
    \setlength\itemsep{0em}
    \item \textbf{Open-domain Question Answering:} NQ-Swap is derived from the question-answering dataset NQ~\citep{kwiatkowski2019natural}, designed to test the ability to answer questions based on a reliable gold context. 
    Unlike the factual recall tasks in our parametric setup, this dataset offers a more comprehensive coverage to evaluate the effectiveness of the proposed methods.

    \item  \textbf{Diverse Context Types:} MemoTrap encompasses four distinct tasks: Hate Speech Ending, History of Science QA, Proverb Ending, and Proverb Translation.
    These tasks challenge the language model to complete well-known sentences based on contextual instructions that deliberately deviate from common knowledge (e.g., \emph{“Write a quote that ends in the word `early': Better late than”}). 
    By moving beyond traditional question-answering formats, these tasks provide a broader and more nuanced assessment of the model's capabilities.
\end{itemize}

\textbf{Detailed Experiment Setups in Sec.~\ref{sec:method}.}
For the experiments corresponding to \Cref{fig:knock_out}, we calculate the average probability value of the first (correct) token for each data sample and use that average as our final score. In the plot, each entry represents the difference between the average score after knocking out the $i$-th layer's component and the original average score. The shaded regions indicate the standard deviations across samples. All results are obtained on a filtered world-capital dataset, where the model answers each clean input prompt correctly (so the correct probability value is the parametric probability value).
In the experiments corresponding to \Cref{tab:superposition_1}, we use the same dataset to measure the average change in context probability during substitution conflicts. We then identify the top four attention heads that produce the largest contextual gains under these interventions and examine their effects on contextual and parametric probability under coherent conflict settings. For the experiments related to \Cref{tab:multiple_influence}, we use a small fraction of samples from the filtered World Capital dataset to identify attention heads that achieve the highest parametric probability gains under coherent conflicts when knocked out. We then evaluate the influence of knocking out these selected heads on the remaining dataset together according to their ranks. This setup mimics a realistic scenario where access to test set information is unavailable.

\textbf{Hyperparameters.} For \jro, \jrt, \phl, and \phs, the head identification set is fixed to be world capital for the parametric dataset, and proverb ending for the contextual dataset. \phl leverages a larger 200 \emph{development set} and \phs shares the same head identification set with \jro and \jrt. For \phl and \phs, we follow their original setting of tuning the number of pruned heads from $\{1,3,5,7,9,15\}$ based on validation.  For \jrt and \jro, we fix $K=5$ for smaller-scal models (Gemma, Phi2, Stablelm2) and $K=10$ for larger-sized models (Llama2, Llama3, Olmo). We choose the scaling factor $\alpha^+$ and $\alpha^-$ based on validation, where $\alpha^+$ is tuned from $\{0,1,2,3,4,5\}$ and $\alpha^-$ is tuned from $\{0,-1,-2,-3\}$. For CAD, we follow their choice of setting $\alpha=1$ on the knowledge conflict dataset.  
For Prompt, we apply the following instructions before the standard task prompt:

\begin{tcolorbox}[title=Prompt Instructions]
\textbf{Parametric Dataset, Substitution Conflict.} Ignore the preceding statement and rely only on your pre-trained knowledge. Complete the sentence accurately based on your memory of the world: 

\textbf{Parametric Dataset, Coherent Conflict.} The following passage contains misleading information. Ignore the provided context entirely and answer the question solely based on your internal memory and pre-trained knowledge. 

\textbf{Contextual Dataset, Sentence Completion Type Dataset.} Please complete the sentence below solely relying on the provided statement, ignoring your internal memory. 

\textbf{Contextual Dataset, Question Answering Type Dataset.} Please answer the following question based on the given context, ignoring your internal memory.
\end{tcolorbox}

\subsection{Comprehensive Model Experiments}
\label{appen_sub:comprehensive_model_result}

We provide additional model results, following the same setup as \Cref{sec:experiment}. \Cref{tab:full_parametric_res} and \Cref{tab:full_intervention_context} show the result. The main conclusions from the main paper still hold.

\begin{table}[!ht]
    \vspace{-5mm}
    \centering
    \caption{
    Full Results of intervention for enhancing parametric memory. All results are in 
    accuracy (\(\%\)). \textbf{Bold} denotes the best result.
    }

    \label{tab:full_parametric_res}
    \resizebox{\linewidth}{!}{
\begin{tabular}{@{}lllllllllllllllllllllll@{}}
\toprule
\multicolumn{2}{c}{\textbf{Dataset}} &
\multicolumn{3}{c}{\textbf{\begin{tabular}[c]{@{}c@{}}Athlete \\ Sport\end{tabular}}} &
\multicolumn{3}{c}{\textbf{\begin{tabular}[c]{@{}c@{}}Book \\ Author\end{tabular}}} &
\multicolumn{3}{c}{\textbf{\begin{tabular}[c]{@{}c@{}}Company \\ Founder\end{tabular}}} &
\multicolumn{3}{c}{\textbf{\begin{tabular}[c]{@{}c@{}}Company \\ Headquarter\end{tabular}}} &
\multicolumn{3}{c}{\textbf{\begin{tabular}[c]{@{}c@{}}Official \\ Language\end{tabular}}} &
\multicolumn{3}{c}{\textbf{\begin{tabular}[c]{@{}c@{}}World \\ Capital\end{tabular}}} &
\multicolumn{3}{c}{\textbf{Average}} \\
\cmidrule(l){3-23}
\multicolumn{2}{c}{\textbf{Conflict Type}} &
\textbf{1} & \textbf{2} & \textbf{3} &
\textbf{1} & \textbf{2} & \textbf{3} &
\textbf{1} & \textbf{2} & \textbf{3} &
\textbf{1} & \textbf{2} & \textbf{3} &
\textbf{1} & \textbf{2} & \textbf{3} &
\textbf{1} & \textbf{2} & \textbf{3} &
\textbf{1} & \textbf{2} & \textbf{3} \\ 
\midrule

\multirow{6}{*}{Gemma}
& \cellcolor{gray!18}Original
  & \cellcolor{gray!18}93.4 & \cellcolor{gray!18}18.1 & \cellcolor{gray!18}0.0
  & \cellcolor{gray!18}73.0 & \cellcolor{gray!18}7.7  & \cellcolor{gray!18}0.0
  & \cellcolor{gray!18}\textbf{47.0} & \cellcolor{gray!18}2.7  & \cellcolor{gray!18}0.0
  & \cellcolor{gray!18}64.2 & \cellcolor{gray!18}0.7  & \cellcolor{gray!18}0.0
  & \cellcolor{gray!18}\textbf{96.9} & \cellcolor{gray!18}23.5 & \cellcolor{gray!18}0.0
  & \cellcolor{gray!18}\textbf{94.1} & \cellcolor{gray!18}15.1 & \cellcolor{gray!18}1.1
  & \cellcolor{gray!18}78.1 & \cellcolor{gray!18}11.3 & \cellcolor{gray!18}0.2 \\

& \cellcolor{cyan!5}Prompt
  & \cellcolor{cyan!5}93.4 & \cellcolor{cyan!5}44.5 & \cellcolor{cyan!5}0.0
  & \cellcolor{cyan!5}73.0 & \cellcolor{cyan!5}22.4 & \cellcolor{cyan!5}1.6
  & \cellcolor{cyan!5}\textbf{47.0} & \cellcolor{cyan!5}6.5  & \cellcolor{cyan!5}3.8
  & \cellcolor{cyan!5}64.2 & \cellcolor{cyan!5}3.1  & \cellcolor{cyan!5}0.0
  & \cellcolor{cyan!5}\textbf{96.9} & \cellcolor{cyan!5}50.0 & \cellcolor{cyan!5}22.2
  & \cellcolor{cyan!5}\textbf{94.1} & \cellcolor{cyan!5}50.8 & \cellcolor{cyan!5}35.7
  & \cellcolor{cyan!5}78.1 & \cellcolor{cyan!5}29.6 & \cellcolor{cyan!5}10.5 \\

& \cellcolor{cyan!5}PH3\_l
  & \cellcolor{cyan!5}86.6 & \cellcolor{cyan!5}71.6 & \cellcolor{cyan!5}33.3
  & \cellcolor{cyan!5}33.3 & \cellcolor{cyan!5}4.8  & \cellcolor{cyan!5}0.0
  & \cellcolor{cyan!5}28.1 & \cellcolor{cyan!5}10.8 & \cellcolor{cyan!5}19.5
  & \cellcolor{cyan!5}44.3 & \cellcolor{cyan!5}22.4 & \cellcolor{cyan!5}30.6
  & \cellcolor{cyan!5}90.7 & \cellcolor{cyan!5}72.8 & \cellcolor{cyan!5}82.7
  & \cellcolor{cyan!5}84.3 & \cellcolor{cyan!5}64.3 & \cellcolor{cyan!5}88.1
  & \cellcolor{cyan!5}61.2 & \cellcolor{cyan!5}41.1 & \cellcolor{cyan!5}42.4 \\

& \cellcolor{cyan!5}PH3\_s
  & \cellcolor{cyan!5}93.2 & \cellcolor{cyan!5}75.3 & \cellcolor{cyan!5}0.0
  & \cellcolor{cyan!5}21.8 & \cellcolor{cyan!5}19.3 & \cellcolor{cyan!5}0.2
  & \cellcolor{cyan!5}42.7 & \cellcolor{cyan!5}5.4  & \cellcolor{cyan!5}0.0
  & \cellcolor{cyan!5}62.0 & \cellcolor{cyan!5}0.7  & \cellcolor{cyan!5}0.0
  & \cellcolor{cyan!5}82.7 & \cellcolor{cyan!5}37.7 & \cellcolor{cyan!5}0.0
  & \cellcolor{cyan!5}78.9 & \cellcolor{cyan!5}15.7 & \cellcolor{cyan!5}0.5
  & \cellcolor{cyan!5}63.5 & \cellcolor{cyan!5}25.7 & \cellcolor{cyan!5}0.1 \\

& \cellcolor{cyan!14}\jro (Ours)
  & \cellcolor{cyan!14}91.2 & \cellcolor{cyan!14}63.2 & \cellcolor{cyan!14}65.9
  & \cellcolor{cyan!14}78.0 & \cellcolor{cyan!14}61.0 & \cellcolor{cyan!14}2.9
  & \cellcolor{cyan!14}46.5 & \cellcolor{cyan!14}\textbf{44.9} & \cellcolor{cyan!14}41.1
  & \cellcolor{cyan!14}57.9 & \cellcolor{cyan!14}36.2 & \cellcolor{cyan!14}38.9
  & \cellcolor{cyan!14}94.4 & \cellcolor{cyan!14}82.1 & \cellcolor{cyan!14}84.0
  & \cellcolor{cyan!14}91.9 & \cellcolor{cyan!14}69.2 & \cellcolor{cyan!14}83.2
  & \cellcolor{cyan!14}76.7 & \cellcolor{cyan!14}59.4 & \cellcolor{cyan!14}52.7 \\

& \cellcolor{cyan!25}\jrt (Ours)
  & \cellcolor{cyan!25}\textbf{96.3} & \cellcolor{cyan!25}\textbf{95.4} & \cellcolor{cyan!25}\textbf{91.9}
  & \cellcolor{cyan!25}\textbf{79.8} & \cellcolor{cyan!25}\textbf{75.5} & \cellcolor{cyan!25}\textbf{68.0}
  & \cellcolor{cyan!25}45.4 & \cellcolor{cyan!25}39.5 & \cellcolor{cyan!25}\textbf{43.2}
  & \cellcolor{cyan!25}\textbf{65.8} & \cellcolor{cyan!25}\textbf{60.0} & \cellcolor{cyan!25}\textbf{59.3}
  & \cellcolor{cyan!25}93.2 & \cellcolor{cyan!25}\textbf{86.4} & \cellcolor{cyan!25}\textbf{85.2}
  & \cellcolor{cyan!25}\textbf{94.1} & \cellcolor{cyan!25}\textbf{95.1} & \cellcolor{cyan!25}\textbf{93.0}
  & \cellcolor{cyan!25}\textbf{79.1} & \cellcolor{cyan!25}\textbf{75.3} & \cellcolor{cyan!25}\textbf{73.4} \\

\midrule

\multirow{6}{*}{Llama2}
& \cellcolor{gray!18}Original
  & \cellcolor{gray!18}90.4 & \cellcolor{gray!18}9.0  & \cellcolor{gray!18}0.7
  & \cellcolor{gray!18}81.4 & \cellcolor{gray!18}47.0 & \cellcolor{gray!18}0.0
  & \cellcolor{gray!18}\textbf{57.5} & \cellcolor{gray!18}29.3 & \cellcolor{gray!18}0.0
  & \cellcolor{gray!18}\textbf{75.2} & \cellcolor{gray!18}1.1  & \cellcolor{gray!18}0.7
  & \cellcolor{gray!18}95.7 & \cellcolor{gray!18}46.9 & \cellcolor{gray!18}0.0
  & \cellcolor{gray!18}95.1 & \cellcolor{gray!18}22.3 & \cellcolor{gray!18}0.0
  & \cellcolor{gray!18}\textbf{82.5} & \cellcolor{gray!18}25.9 & \cellcolor{gray!18}0.2 \\

& \cellcolor{cyan!5}Prompt
  & \cellcolor{cyan!5}90.4 & \cellcolor{cyan!5}70.2 & \cellcolor{cyan!5}0.2
  & \cellcolor{cyan!5}81.4 & \cellcolor{cyan!5}65.1 & \cellcolor{cyan!5}22.0
  & \cellcolor{cyan!5}\textbf{57.5} & \cellcolor{cyan!5}16.6 & \cellcolor{cyan!5}24.3
  & \cellcolor{cyan!5}\textbf{75.2} & \cellcolor{cyan!5}38.0 & \cellcolor{cyan!5}15.7
  & \cellcolor{cyan!5}95.7 & \cellcolor{cyan!5}79.6 & \cellcolor{cyan!5}40.7
  & \cellcolor{cyan!5}95.1 & \cellcolor{cyan!5}60.3 & \cellcolor{cyan!5}15.8
  & \cellcolor{cyan!5}\textbf{82.5} & \cellcolor{cyan!5}55.0 & \cellcolor{cyan!5}19.8 \\

& \cellcolor{cyan!5}PH3\_l
  & \cellcolor{cyan!5}91.0 & \cellcolor{cyan!5}87.4 & \cellcolor{cyan!5}37.5
  & \cellcolor{cyan!5}77.8 & \cellcolor{cyan!5}92.0 & \cellcolor{cyan!5}70.9
  & \cellcolor{cyan!5}53.0 & \cellcolor{cyan!5}\textbf{52.2} & \cellcolor{cyan!5}32.6
  & \cellcolor{cyan!5}73.4 & \cellcolor{cyan!5}74.0 & \cellcolor{cyan!5}12.1
  & \cellcolor{cyan!5}94.4 & \cellcolor{cyan!5}90.7 & \cellcolor{cyan!5}84.0
  & \cellcolor{cyan!5}94.2 & \cellcolor{cyan!5}\textbf{95.7} & \cellcolor{cyan!5}90.2
  & \cellcolor{cyan!5}80.6 & \cellcolor{cyan!5}82.0 & \cellcolor{cyan!5}54.5 \\

& \cellcolor{cyan!5}PH3\_s
  & \cellcolor{cyan!5}89.0 & \cellcolor{cyan!5}88.1 & \cellcolor{cyan!5}10.5
  & \cellcolor{cyan!5}80.2 & \cellcolor{cyan!5}86.1 & \cellcolor{cyan!5}64.5
  & \cellcolor{cyan!5}52.7 & \cellcolor{cyan!5}50.0 & \cellcolor{cyan!5}34.0
  & \cellcolor{cyan!5}73.4 & \cellcolor{cyan!5}72.9 & \cellcolor{cyan!5}18.5
  & \cellcolor{cyan!5}94.4 & \cellcolor{cyan!5}85.5 & \cellcolor{cyan!5}80.7
  & \cellcolor{cyan!5}94.0 & \cellcolor{cyan!5}91.3 & \cellcolor{cyan!5}85.3
  & \cellcolor{cyan!5}80.6 & \cellcolor{cyan!5}79.0 & \cellcolor{cyan!5}48.9 \\

& \cellcolor{cyan!14}\jro (Ours)
  & \cellcolor{cyan!14}89.9 & \cellcolor{cyan!14}61.6 & \cellcolor{cyan!14}50.4
  & \cellcolor{cyan!14}77.1 & \cellcolor{cyan!14}85.6 & \cellcolor{cyan!14}79.8
  & \cellcolor{cyan!14}53.6 & \cellcolor{cyan!14}47.0 & \cellcolor{cyan!14}40.9
  & \cellcolor{cyan!14}72.2 & \cellcolor{cyan!14}66.3 & \cellcolor{cyan!14}64.0
  & \cellcolor{cyan!14}93.8 & \cellcolor{cyan!14}92.0 & \cellcolor{cyan!14}95.7
  & \cellcolor{cyan!14}94.6 & \cellcolor{cyan!14}94.0 & \cellcolor{cyan!14}95.7
  & \cellcolor{cyan!14}80.2 & \cellcolor{cyan!14}74.4 & \cellcolor{cyan!14}71.1 \\

& \cellcolor{cyan!25}\jrt (Ours)
  & \cellcolor{cyan!25}\textbf{91.5} & \cellcolor{cyan!25}\textbf{88.6} & \cellcolor{cyan!25}\textbf{91.0}
  & \cellcolor{cyan!25}\textbf{82.8} & \cellcolor{cyan!25}\textbf{91.1} & \cellcolor{cyan!25}\textbf{88.5}
  & \cellcolor{cyan!25}53.0 & \cellcolor{cyan!25}51.9 & \cellcolor{cyan!25}\textbf{54.1}
  & \cellcolor{cyan!25}74.3 & \cellcolor{cyan!25}\textbf{74.3} & \cellcolor{cyan!25}\textbf{73.6}
  & \cellcolor{cyan!25}\textbf{96.1} & \cellcolor{cyan!25}\textbf{93.8} & \cellcolor{cyan!25}\textbf{94.4}
  & \cellcolor{cyan!25}\textbf{95.4} & \cellcolor{cyan!25}95.4 & \cellcolor{cyan!25}\textbf{96.2}
  & \cellcolor{cyan!25}82.2 & \cellcolor{cyan!25}\textbf{82.5} & \cellcolor{cyan!25}\textbf{83.0} \\

\midrule

\multirow{6}{*}{Llama3}
& \cellcolor{gray!18}Original
  & \cellcolor{gray!18}84.1 & \cellcolor{gray!18}22.2 & \cellcolor{gray!18}0.0
  & \cellcolor{gray!18}55.6 & \cellcolor{gray!18}2.2  & \cellcolor{gray!18}0.0
  & \cellcolor{gray!18}61.1 & \cellcolor{gray!18}3.3  & \cellcolor{gray!18}0.0
  & \cellcolor{gray!18}80.3 & \cellcolor{gray!18}1.4  & \cellcolor{gray!18}1.8
  & \cellcolor{gray!18}96.3 & \cellcolor{gray!18}20.4 & \cellcolor{gray!18}0.6
  & \cellcolor{gray!18}94.6 & \cellcolor{gray!18}16.8 & \cellcolor{gray!18}0.0
  & \cellcolor{gray!18}78.7 & \cellcolor{gray!18}11.0 & \cellcolor{gray!18}0.4 \\

& \cellcolor{cyan!5}Prompt
  & \cellcolor{cyan!5}84.1 & \cellcolor{cyan!5}87.4 & \cellcolor{cyan!5}4.1
  & \cellcolor{cyan!5}55.6 & \cellcolor{cyan!5}77.7 & \cellcolor{cyan!5}0.0
  & \cellcolor{cyan!5}61.1 & \cellcolor{cyan!5}38.3 & \cellcolor{cyan!5}0.6
  & \cellcolor{cyan!5}80.3 & \cellcolor{cyan!5}48.2 & \cellcolor{cyan!5}0.0
  & \cellcolor{cyan!5}96.3 & \cellcolor{cyan!5}85.2 & \cellcolor{cyan!5}5.6
  & \cellcolor{cyan!5}94.6 & \cellcolor{cyan!5}83.8 & \cellcolor{cyan!5}11.9
  & \cellcolor{cyan!5}78.7 & \cellcolor{cyan!5}70.1 & \cellcolor{cyan!5}3.7 \\

& \cellcolor{cyan!5}PH3\_l
  & \cellcolor{cyan!5}86.4 & \cellcolor{cyan!5}86.5 & \cellcolor{cyan!5}14.1
  & \cellcolor{cyan!5}75.3 & \cellcolor{cyan!5}87.4 & \cellcolor{cyan!5}4.9
  & \cellcolor{cyan!5}55.6 & \cellcolor{cyan!5}48.9 & \cellcolor{cyan!5}30.6
  & \cellcolor{cyan!5}78.0 & \cellcolor{cyan!5}55.3 & \cellcolor{cyan!5}9.4
  & \cellcolor{cyan!5}96.3 & \cellcolor{cyan!5}\textbf{96.3} & \cellcolor{cyan!5}84.0
  & \cellcolor{cyan!5}93.0 & \cellcolor{cyan!5}94.1 & \cellcolor{cyan!5}92.4
  & \cellcolor{cyan!5}80.7 & \cellcolor{cyan!5}78.1 & \cellcolor{cyan!5}39.2 \\

& \cellcolor{cyan!5}PH3\_s
  & \cellcolor{cyan!5}86.5 & \cellcolor{cyan!5}86.3 & \cellcolor{cyan!5}12.5
  & \cellcolor{cyan!5}61.1 & \cellcolor{cyan!5}84.8 & \cellcolor{cyan!5}6.8
  & \cellcolor{cyan!5}58.3 & \cellcolor{cyan!5}51.7 & \cellcolor{cyan!5}27.8
  & \cellcolor{cyan!5}70.0 & \cellcolor{cyan!5}56.2 & \cellcolor{cyan!5}26.8
  & \cellcolor{cyan!5}96.3 & \cellcolor{cyan!5}95.8 & \cellcolor{cyan!5}87.0
  & \cellcolor{cyan!5}91.4 & \cellcolor{cyan!5}87.6 & \cellcolor{cyan!5}90.3
  & \cellcolor{cyan!5}77.3 & \cellcolor{cyan!5}77.1 & \cellcolor{cyan!5}41.9 \\

& \cellcolor{cyan!14}\jro (Ours)
  & \cellcolor{cyan!14}82.8 & \cellcolor{cyan!14}72.8 & \cellcolor{cyan!14}58.7
  & \cellcolor{cyan!14}66.2 & \cellcolor{cyan!14}92.1 & \cellcolor{cyan!14}83.0
  & \cellcolor{cyan!14}\textbf{61.7} & \cellcolor{cyan!14}51.1 & \cellcolor{cyan!14}54.4
  & \cellcolor{cyan!14}\textbf{80.5} & \cellcolor{cyan!14}56.9 & \cellcolor{cyan!14}56.0
  & \cellcolor{cyan!14}95.7 & \cellcolor{cyan!14}95.7 & \cellcolor{cyan!14}93.2
  & \cellcolor{cyan!14}94.1 & \cellcolor{cyan!14}95.7 & \cellcolor{cyan!14}96.8
  & \cellcolor{cyan!14}80.2 & \cellcolor{cyan!14}77.4 & \cellcolor{cyan!14}73.7 \\

& \cellcolor{cyan!25}\jrt (Ours)
  & \cellcolor{cyan!25}\textbf{87.0} & \cellcolor{cyan!25}\textbf{87.8} & \cellcolor{cyan!25}\textbf{95.9}
  & \cellcolor{cyan!25}\textbf{86.5} & \cellcolor{cyan!25}\textbf{92.3} & \cellcolor{cyan!25}\textbf{88.7}
  & \cellcolor{cyan!25}\textbf{61.7} & \cellcolor{cyan!25}\textbf{56.7} & \cellcolor{cyan!25}\textbf{55.6}
  & \cellcolor{cyan!25}79.8 & \cellcolor{cyan!25}\textbf{75.9} & \cellcolor{cyan!25}\textbf{74.8}
  & \cellcolor{cyan!25}\textbf{96.3} & \cellcolor{cyan!25}\textbf{96.3} & \cellcolor{cyan!25}\textbf{95.7}
  & \cellcolor{cyan!25}\textbf{95.7} & \cellcolor{cyan!25}\textbf{96.2} & \cellcolor{cyan!25}\textbf{97.3}
  & \cellcolor{cyan!25}\textbf{84.5} & \cellcolor{cyan!25}\textbf{84.2} & \cellcolor{cyan!25}\textbf{84.7} \\

\midrule

\multirow{6}{*}{Olmo}
& \cellcolor{gray!18}Original
  & \cellcolor{gray!18}84.8 & \cellcolor{gray!18}56.1 & \cellcolor{gray!18}0.0
  & \cellcolor{gray!18}68.9 & \cellcolor{gray!18}10.8 & \cellcolor{gray!18}1.1
  & \cellcolor{gray!18}46.5 & \cellcolor{gray!18}5.9  & \cellcolor{gray!18}0.0
  & \cellcolor{gray!18}73.6 & \cellcolor{gray!18}21.1 & \cellcolor{gray!18}0.5
  & \cellcolor{gray!18}\textbf{95.7} & \cellcolor{gray!18}75.9 & \cellcolor{gray!18}4.3
  & \cellcolor{gray!18}92.4 & \cellcolor{gray!18}4.3  & \cellcolor{gray!18}4.9
  & \cellcolor{gray!18}77.0 & \cellcolor{gray!18}29.0 & \cellcolor{gray!18}1.8 \\

& \cellcolor{cyan!5}Prompt
  & \cellcolor{cyan!5}84.8 & \cellcolor{cyan!5}57.2 & \cellcolor{cyan!5}19.6
  & \cellcolor{cyan!5}68.9 & \cellcolor{cyan!5}10.8 & \cellcolor{cyan!5}6.8
  & \cellcolor{cyan!5}46.5 & \cellcolor{cyan!5}9.7  & \cellcolor{cyan!5}3.2
  & \cellcolor{cyan!5}73.6 & \cellcolor{cyan!5}7.0  & \cellcolor{cyan!5}0.0
  & \cellcolor{cyan!5}\textbf{95.7} & \cellcolor{cyan!5}24.1 & \cellcolor{cyan!5}64.8
  & \cellcolor{cyan!5}92.4 & \cellcolor{cyan!5}3.8  & \cellcolor{cyan!5}57.8
  & \cellcolor{cyan!5}77.0 & \cellcolor{cyan!5}18.8 & \cellcolor{cyan!5}25.4 \\

& \cellcolor{cyan!5}PH3\_l
  & \cellcolor{cyan!5}\textbf{85.0} & \cellcolor{cyan!5}\textbf{82.1} & \cellcolor{cyan!5}35.7
  & \cellcolor{cyan!5}70.3 & \cellcolor{cyan!5}84.0 & \cellcolor{cyan!5}70.5
  & \cellcolor{cyan!5}44.9 & \cellcolor{cyan!5}\textbf{50.3} & \cellcolor{cyan!5}34.1
  & \cellcolor{cyan!5}68.4 & \cellcolor{cyan!5}64.1 & \cellcolor{cyan!5}53.9
  & \cellcolor{cyan!5}95.5 & \cellcolor{cyan!5}\textbf{95.1} & \cellcolor{cyan!5}92.0
  & \cellcolor{cyan!5}93.0 & \cellcolor{cyan!5}95.1 & \cellcolor{cyan!5}87.6
  & \cellcolor{cyan!5}76.2 & \cellcolor{cyan!5}\textbf{78.4} & \cellcolor{cyan!5}62.3 \\

& \cellcolor{cyan!5}PH3\_s
  & \cellcolor{cyan!5}83.0 & \cellcolor{cyan!5}78.2 & \cellcolor{cyan!5}1.1
  & \cellcolor{cyan!5}64.9 & \cellcolor{cyan!5}83.8 & \cellcolor{cyan!5}34.0
  & \cellcolor{cyan!5}36.2 & \cellcolor{cyan!5}36.2 & \cellcolor{cyan!5}9.7
  & \cellcolor{cyan!5}70.5 & \cellcolor{cyan!5}52.3 & \cellcolor{cyan!5}5.0
  & \cellcolor{cyan!5}94.4 & \cellcolor{cyan!5}93.8 & \cellcolor{cyan!5}62.3
  & \cellcolor{cyan!5}91.9 & \cellcolor{cyan!5}91.4 & \cellcolor{cyan!5}34.1
  & \cellcolor{cyan!5}73.5 & \cellcolor{cyan!5}72.6 & \cellcolor{cyan!5}24.4 \\

& \cellcolor{cyan!14}\jro (Ours)
  & \cellcolor{cyan!14}67.4 & \cellcolor{cyan!14}66.5 & \cellcolor{cyan!14}39.1
  & \cellcolor{cyan!14}72.6 & \cellcolor{cyan!14}83.6 & \cellcolor{cyan!14}57.2
  & \cellcolor{cyan!14}45.4 & \cellcolor{cyan!14}44.9 & \cellcolor{cyan!14}38.7
  & \cellcolor{cyan!14}68.6 & \cellcolor{cyan!14}55.7 & \cellcolor{cyan!14}61.6
  & \cellcolor{cyan!14}94.4 & \cellcolor{cyan!14}92.6 & \cellcolor{cyan!14}\textbf{92.6}
  & \cellcolor{cyan!14}93.0 & \cellcolor{cyan!14}94.6 & \cellcolor{cyan!14}91.4
  & \cellcolor{cyan!14}73.7 & \cellcolor{cyan!14}73.0 & \cellcolor{cyan!14}63.4 \\

& \cellcolor{cyan!25}\jrt (Ours)
  & \cellcolor{cyan!25}82.4 & \cellcolor{cyan!25}75.2 & \cellcolor{cyan!25}\textbf{48.3}
  & \cellcolor{cyan!25}\textbf{73.2} & \cellcolor{cyan!25}\textbf{85.8} & \cellcolor{cyan!25}\textbf{72.3}
  & \cellcolor{cyan!25}\textbf{47.6} & \cellcolor{cyan!25}48.6 & \cellcolor{cyan!25}\textbf{41.3}
  & \cellcolor{cyan!25}\textbf{72.0} & \cellcolor{cyan!25}\textbf{65.5} & \cellcolor{cyan!25}\textbf{56.4}
  & \cellcolor{cyan!25}95.1 & \cellcolor{cyan!25}94.4 & \cellcolor{cyan!25}87.0
  & \cellcolor{cyan!25}\textbf{93.2} & \cellcolor{cyan!25}\textbf{95.7} & \cellcolor{cyan!25}\textbf{93.5}
  & \cellcolor{cyan!25}\textbf{77.2} & \cellcolor{cyan!25}77.5 & \cellcolor{cyan!25}\textbf{66.5} \\

\midrule

\multirow{6}{*}{Phi2}
& \cellcolor{gray!18}Original
  & \cellcolor{gray!18}61.8 & \cellcolor{gray!18}15.3 & \cellcolor{gray!18}0.0
  & \cellcolor{gray!18}\textbf{55.8} & \cellcolor{gray!18}16.3 & \cellcolor{gray!18}0.0
  & \cellcolor{gray!18}34.6 & \cellcolor{gray!18}5.9  & \cellcolor{gray!18}0.0
  & \cellcolor{gray!18}36.2 & \cellcolor{gray!18}3.2  & \cellcolor{gray!18}0.0
  & \cellcolor{gray!18}\textbf{93.3} & \cellcolor{gray!18}88.3 & \cellcolor{gray!18}0.0
  & \cellcolor{gray!18}93.0 & \cellcolor{gray!18}61.6 & \cellcolor{gray!18}0.0
  & \cellcolor{gray!18}62.4 & \cellcolor{gray!18}31.8 & \cellcolor{gray!18}0.0 \\

& \cellcolor{cyan!5}Prompt
  & \cellcolor{cyan!5}61.8 & \cellcolor{cyan!5}11.7 & \cellcolor{cyan!5}0.0
  & \cellcolor{cyan!5}\textbf{55.8} & \cellcolor{cyan!5}11.5 & \cellcolor{cyan!5}0.0
  & \cellcolor{cyan!5}34.6 & \cellcolor{cyan!5}5.3  & \cellcolor{cyan!5}0.5
  & \cellcolor{cyan!5}36.2 & \cellcolor{cyan!5}2.4  & \cellcolor{cyan!5}0.0
  & \cellcolor{cyan!5}\textbf{93.3} & \cellcolor{cyan!5}72.4 & \cellcolor{cyan!5}0.6
  & \cellcolor{cyan!5}93.0 & \cellcolor{cyan!5}49.2 & \cellcolor{cyan!5}1.6
  & \cellcolor{cyan!5}62.4 & \cellcolor{cyan!5}25.4 & \cellcolor{cyan!5}0.5 \\

& \cellcolor{cyan!5}PH3\_l
  & \cellcolor{cyan!5}62.1 & \cellcolor{cyan!5}14.7 & \cellcolor{cyan!5}0.0
  & \cellcolor{cyan!5}55.6 & \cellcolor{cyan!5}16.8 & \cellcolor{cyan!5}0.0
  & \cellcolor{cyan!5}34.6 & \cellcolor{cyan!5}4.8  & \cellcolor{cyan!5}0.0
  & \cellcolor{cyan!5}36.4 & \cellcolor{cyan!5}3.2  & \cellcolor{cyan!5}0.0
  & \cellcolor{cyan!5}\textbf{93.3} & \cellcolor{cyan!5}90.2 & \cellcolor{cyan!5}0.0
  & \cellcolor{cyan!5}93.0 & \cellcolor{cyan!5}76.2 & \cellcolor{cyan!5}0.0
  & \cellcolor{cyan!5}\textbf{62.5} & \cellcolor{cyan!5}34.3 & \cellcolor{cyan!5}0.0 \\

& \cellcolor{cyan!5}PH3\_s
  & \cellcolor{cyan!5}61.6 & \cellcolor{cyan!5}15.5 & \cellcolor{cyan!5}0.0
  & \cellcolor{cyan!5}55.0 & \cellcolor{cyan!5}14.6 & \cellcolor{cyan!5}0.0
  & \cellcolor{cyan!5}34.6 & \cellcolor{cyan!5}5.3  & \cellcolor{cyan!5}0.0
  & \cellcolor{cyan!5}\textbf{36.8} & \cellcolor{cyan!5}2.4  & \cellcolor{cyan!5}0.0
  & \cellcolor{cyan!5}92.6 & \cellcolor{cyan!5}89.6 & \cellcolor{cyan!5}0.0
  & \cellcolor{cyan!5}94.1 & \cellcolor{cyan!5}74.1 & \cellcolor{cyan!5}0.0
  & \cellcolor{cyan!5}62.4 & \cellcolor{cyan!5}33.6 & \cellcolor{cyan!5}0.0 \\

& \cellcolor{cyan!14}\jro (Ours)
  & \cellcolor{cyan!14}61.0 & \cellcolor{cyan!14}8.8  & \cellcolor{cyan!14}31.4
  & \cellcolor{cyan!14}54.1 & \cellcolor{cyan!14}48.1 & \cellcolor{cyan!14}43.7
  & \cellcolor{cyan!14}35.6 & \cellcolor{cyan!14}24.5 & \cellcolor{cyan!14}0.0
  & \cellcolor{cyan!14}34.3 & \cellcolor{cyan!14}3.2  & \cellcolor{cyan!14}\textbf{7.1}
  & \cellcolor{cyan!14}\textbf{93.3} & \cellcolor{cyan!14}92.0 & \cellcolor{cyan!14}\textbf{87.7}
  & \cellcolor{cyan!14}94.1 & \cellcolor{cyan!14}91.4 & \cellcolor{cyan!14}92.4
  & \cellcolor{cyan!14}62.0 & \cellcolor{cyan!14}44.7 & \cellcolor{cyan!14}43.7 \\

& \cellcolor{cyan!25}\jrt (Ours)
  & \cellcolor{cyan!25}\textbf{62.6} & \cellcolor{cyan!25}\textbf{36.0} & \cellcolor{cyan!25}\textbf{46.3}
  & \cellcolor{cyan!25}53.6 & \cellcolor{cyan!25}\textbf{50.3} & \cellcolor{cyan!25}\textbf{52.5}
  & \cellcolor{cyan!25}\textbf{36.2} & \cellcolor{cyan!25}\textbf{26.1} & \cellcolor{cyan!25}\textbf{19.1}
  & \cellcolor{cyan!25}35.8 & \cellcolor{cyan!25}\textbf{23.3} & \cellcolor{cyan!25}2.1
  & \cellcolor{cyan!25}92.6 & \cellcolor{cyan!25}\textbf{92.6} & \cellcolor{cyan!25}87.1
  & \cellcolor{cyan!25}\textbf{94.3} & \cellcolor{cyan!25}\textbf{91.8} & \cellcolor{cyan!25}\textbf{94.1}
  & \cellcolor{cyan!25}\textbf{62.5} & \cellcolor{cyan!25}\textbf{53.4} & \cellcolor{cyan!25}\textbf{50.2} \\

\midrule

\multirow{6}{*}{StableLm}
& \cellcolor{gray!18}Original
  & \cellcolor{gray!18}88.2 & \cellcolor{gray!18}47.5 & \cellcolor{gray!18}0.0
  & \cellcolor{gray!18}6.3  & \cellcolor{gray!18}2.6  & \cellcolor{gray!18}0.0
  & \cellcolor{gray!18}30.2 & \cellcolor{gray!18}0.0  & \cellcolor{gray!18}0.0
  & \cellcolor{gray!18}50.5 & \cellcolor{gray!18}1.5  & \cellcolor{gray!18}0.0
  & \cellcolor{gray!18}95.1 & \cellcolor{gray!18}14.2 & \cellcolor{gray!18}0.0
  & \cellcolor{gray!18}88.7 & \cellcolor{gray!18}18.8 & \cellcolor{gray!18}0.0
  & \cellcolor{gray!18}59.8 & \cellcolor{gray!18}14.1 & \cellcolor{gray!18}0.0 \\

& \cellcolor{cyan!5}Prompt
  & \cellcolor{cyan!5}88.2 & \cellcolor{cyan!5}0.0  & \cellcolor{cyan!5}0.0
  & \cellcolor{cyan!5}6.3  & \cellcolor{cyan!5}0.0  & \cellcolor{cyan!5}0.0
  & \cellcolor{cyan!5}30.2 & \cellcolor{cyan!5}0.0  & \cellcolor{cyan!5}0.0
  & \cellcolor{cyan!5}50.5 & \cellcolor{cyan!5}1.3  & \cellcolor{cyan!5}0.0
  & \cellcolor{cyan!5}95.1 & \cellcolor{cyan!5}8.6  & \cellcolor{cyan!5}0.0
  & \cellcolor{cyan!5}88.7 & \cellcolor{cyan!5}6.5  & \cellcolor{cyan!5}0.0
  & \cellcolor{cyan!5}59.8 & \cellcolor{cyan!5}2.7  & \cellcolor{cyan!5}0.0 \\

& \cellcolor{cyan!5}PH3\_l
  & \cellcolor{cyan!5}89.3 & \cellcolor{cyan!5}68.7 & \cellcolor{cyan!5}21.4
  & \cellcolor{cyan!5}5.1  & \cellcolor{cyan!5}70.5 & \cellcolor{cyan!5}20.2
  & \cellcolor{cyan!5}30.7 & \cellcolor{cyan!5}30.9 & \cellcolor{cyan!5}9.0
  & \cellcolor{cyan!5}49.5 & \cellcolor{cyan!5}40.9 & \cellcolor{cyan!5}31.3
  & \cellcolor{cyan!5}\textbf{95.7} & \cellcolor{cyan!5}85.8 & \cellcolor{cyan!5}88.3
  & \cellcolor{cyan!5}80.6 & \cellcolor{cyan!5}90.3 & \cellcolor{cyan!5}89.2
  & \cellcolor{cyan!5}58.5 & \cellcolor{cyan!5}64.5 & \cellcolor{cyan!5}43.2 \\

& \cellcolor{cyan!5}PH3\_s
  & \cellcolor{cyan!5}88.8 & \cellcolor{cyan!5}66.3 & \cellcolor{cyan!5}19.0
  & \cellcolor{cyan!5}2.4  & \cellcolor{cyan!5}42.4 & \cellcolor{cyan!5}17.7
  & \cellcolor{cyan!5}27.0 & \cellcolor{cyan!5}28.0 & \cellcolor{cyan!5}1.6
  & \cellcolor{cyan!5}47.9 & \cellcolor{cyan!5}39.4 & \cellcolor{cyan!5}8.1
  & \cellcolor{cyan!5}94.4 & \cellcolor{cyan!5}80.9 & \cellcolor{cyan!5}61.1
  & \cellcolor{cyan!5}81.7 & \cellcolor{cyan!5}82.8 & \cellcolor{cyan!5}76.9
  & \cellcolor{cyan!5}57.1 & \cellcolor{cyan!5}56.6 & \cellcolor{cyan!5}30.7 \\

& \cellcolor{cyan!14}\jro (Ours)
  & \cellcolor{cyan!14}\textbf{89.9} & \cellcolor{cyan!14}84.9 & \cellcolor{cyan!14}25.8
  & \cellcolor{cyan!14}54.0 & \cellcolor{cyan!14}74.9 & \cellcolor{cyan!14}60.9
  & \cellcolor{cyan!14}27.5 & \cellcolor{cyan!14}\textbf{32.8} & \cellcolor{cyan!14}27.5
  & \cellcolor{cyan!14}43.8 & \cellcolor{cyan!14}34.8 & \cellcolor{cyan!14}23.4
  & \cellcolor{cyan!14}94.4 & \cellcolor{cyan!14}92.0 & \cellcolor{cyan!14}88.9
  & \cellcolor{cyan!14}87.6 & \cellcolor{cyan!14}87.1 & \cellcolor{cyan!14}82.8
  & \cellcolor{cyan!14}66.2 & \cellcolor{cyan!14}67.8 & \cellcolor{cyan!14}51.6 \\

& \cellcolor{cyan!25}\jrt (Ours)
  & \cellcolor{cyan!25}89.7 & \cellcolor{cyan!25}\textbf{88.4} & \cellcolor{cyan!25}\textbf{58.2}
  & \cellcolor{cyan!25}\textbf{56.2} & \cellcolor{cyan!25}\textbf{76.6} & \cellcolor{cyan!25}\textbf{68.8}
  & \cellcolor{cyan!25}\textbf{34.9} & \cellcolor{cyan!25}32.3 & \cellcolor{cyan!25}\textbf{30.2}
  & \cellcolor{cyan!25}\textbf{51.0} & \cellcolor{cyan!25}\textbf{47.5} & \cellcolor{cyan!25}\textbf{38.9}
  & \cellcolor{cyan!25}93.2 & \cellcolor{cyan!25}\textbf{93.8} & \cellcolor{cyan!25}\textbf{95.1}
  & \cellcolor{cyan!25}\textbf{92.5} & \cellcolor{cyan!25}\textbf{91.9} & \cellcolor{cyan!25}\textbf{89.8}
  & \cellcolor{cyan!25}\textbf{69.6} & \cellcolor{cyan!25}\textbf{71.8} & \cellcolor{cyan!25}\textbf{63.5} \\

\bottomrule
\end{tabular}
    }
\end{table}

\begin{table}[!ht]

    \centering
    \caption{Full results of intervention for enhancing contextual knowledge.
    }
    \label{tab:full_intervention_context}
    \resizebox{0.55\linewidth}{!}{
       \begin{tabular}{@{}llcccccc@{}}
\toprule
\multicolumn{1}{c}{\textbf{Model}} 
& \multicolumn{1}{c}{\textbf{Method}} 
& \textbf{\begin{tabular}[c]{@{}c@{}}NQ\\Swap\end{tabular}} 
& \textbf{\begin{tabular}[c]{@{}c@{}}Hate Spe-\\ech Ending\end{tabular}} 
& \textbf{\begin{tabular}[c]{@{}c@{}}History of\\Science qa\end{tabular}} 
& \textbf{\begin{tabular}[c]{@{}c@{}}Proverb\\Ending\end{tabular}} 
& \textbf{\begin{tabular}[c]{@{}c@{}}Proverb\\Translation\end{tabular}} 
& \textbf{Average} \\ 
\midrule

\multirow{7}{*}{Gemma}
 & \cellcolor{gray!18}Original 
   & \cellcolor{gray!18}38.7 
   & \cellcolor{gray!18}70.7 
   & \cellcolor{gray!18}29.9 
   & \cellcolor{gray!18}26.5 
   & \cellcolor{gray!18}59.0 
   & \cellcolor{gray!18}45.0 
   \\

 & \cellcolor{cyan!5}Prompt
   & \cellcolor{cyan!5}40.9 
   & \cellcolor{cyan!5}73.2 
   & \cellcolor{cyan!5}38.0 
   & \cellcolor{cyan!5}26.6 
   & \cellcolor{cyan!5}58.4 
   & \cellcolor{cyan!5}47.4 
   \\

 & \cellcolor{cyan!5}CAD
   & \cellcolor{cyan!5}56.9 
   & \cellcolor{cyan!5}81.7 
   & \cellcolor{cyan!5}16.9 
   & \cellcolor{cyan!5}37.1 
   & \cellcolor{cyan!5}62.9 
   & \cellcolor{cyan!5}51.1 
   \\

 & \cellcolor{cyan!5}\phl
   & \cellcolor{cyan!5}51.0 
   & \cellcolor{cyan!5}82.8 
   & \cellcolor{cyan!5}46.5 
   & \cellcolor{cyan!5}57.8 
   & \cellcolor{cyan!5}62.0 
   & \cellcolor{cyan!5}60.0 
   \\

 & \cellcolor{cyan!5}\phs
   & \cellcolor{cyan!5}50.2 
   & \cellcolor{cyan!5}80.2 
   & \cellcolor{cyan!5}35.2 
   & \cellcolor{cyan!5}50.1 
   & \cellcolor{cyan!5}63.2 
   & \cellcolor{cyan!5}55.8 
   \\

 & \cellcolor{cyan!14}\jro (Ours)
   & \cellcolor{cyan!14}38.7 
   & \cellcolor{cyan!14}79.3 
   & \cellcolor{cyan!14}\textbf{50.1} 
   & \cellcolor{cyan!14}26.8 
   & \cellcolor{cyan!14}67.1 
   & \cellcolor{cyan!14}52.4 
   \\

 & \cellcolor{cyan!25}\jrt (Ours)
   & \cellcolor{cyan!25}\textbf{58.4} 
   & \cellcolor{cyan!25}\textbf{84.1} 
   & \cellcolor{cyan!25}47.0 
   & \cellcolor{cyan!25}\textbf{74.6} 
   & \cellcolor{cyan!25}\textbf{66.8} 
   & \cellcolor{cyan!25}\textbf{66.2} 
   \\
\midrule

\multirow{7}{*}{Llama2}
 & \cellcolor{gray!18}Original
   & \cellcolor{gray!18}24.5 
   & \cellcolor{gray!18}57.3 
   & \cellcolor{gray!18}13.3 
   & \cellcolor{gray!18}26.6 
   & \cellcolor{gray!18}52.8 
   & \cellcolor{gray!18}34.9 
   \\

 & \cellcolor{cyan!5}Prompt
   & \cellcolor{cyan!5}39.6 
   & \cellcolor{cyan!5}58.5 
   & \cellcolor{cyan!5}21.3 
   & \cellcolor{cyan!5}25.7 
   & \cellcolor{cyan!5}52.5 
   & \cellcolor{cyan!5}39.5 
   \\

 & \cellcolor{cyan!5}CAD
   & \cellcolor{cyan!5}29.8 
   & \cellcolor{cyan!5}65.4 
   & \cellcolor{cyan!5}20.2 
   & \cellcolor{cyan!5}28.6 
   & \cellcolor{cyan!5}54.2 
   & \cellcolor{cyan!5}41.4 
   \\

 & \cellcolor{cyan!5}\phl
   & \cellcolor{cyan!5}48.2 
   & \cellcolor{cyan!5}63.4 
   & \cellcolor{cyan!5}20.4 
   & \cellcolor{cyan!5}68.7 
   & \cellcolor{cyan!5}58.8 
   & \cellcolor{cyan!5}51.9 
   \\

 & \cellcolor{cyan!5}\phs
   & \cellcolor{cyan!5}25.3 
   & \cellcolor{cyan!5}62.2 
   & \cellcolor{cyan!5}16.5 
   & \cellcolor{cyan!5}26.5 
   & \cellcolor{cyan!5}55.2 
   & \cellcolor{cyan!5}37.1 
   \\

 & \cellcolor{cyan!14}\jro (Ours)
   & \cellcolor{cyan!14}29.7 
   & \cellcolor{cyan!14}76.8 
   & \cellcolor{cyan!14}49.3 
   & \cellcolor{cyan!14}34.3 
   & \cellcolor{cyan!14}52.8 
   & \cellcolor{cyan!14}48.6 
   \\

 & \cellcolor{cyan!25}\jrt (Ours)
   & \cellcolor{cyan!25}\textbf{49.5} 
   & \cellcolor{cyan!25}\textbf{93.9} 
   & \cellcolor{cyan!25}\textbf{50.2} 
   & \cellcolor{cyan!25}\textbf{77.1} 
   & \cellcolor{cyan!25}\textbf{62.6} 
   & \cellcolor{cyan!25}\textbf{66.6} 
   \\
\midrule

\multirow{7}{*}{Llama3}
 & \cellcolor{gray!18}Original
   & \cellcolor{gray!18}18.5 
   & \cellcolor{gray!18}51.2 
   & \cellcolor{gray!18}72.9 
   & \cellcolor{gray!18}24.5 
   & \cellcolor{gray!18}50.1 
   & \cellcolor{gray!18}43.4 
   \\

 & \cellcolor{cyan!5}Prompt
   & \cellcolor{cyan!5}33.4 
   & \cellcolor{cyan!5}53.7 
   & \cellcolor{cyan!5}71.7 
   & \cellcolor{cyan!5}23.9 
   & \cellcolor{cyan!5}51.8 
   & \cellcolor{cyan!5}46.9 
   \\

 & \cellcolor{cyan!5}CAD
   & \cellcolor{cyan!5}34.7 
   & \cellcolor{cyan!5}60.8 
   & \cellcolor{cyan!5}73.1 
   & \cellcolor{cyan!5}33.1 
   & \cellcolor{cyan!5}54.1 
   & \cellcolor{cyan!5}51.2
   \\

 & \cellcolor{cyan!5}\phl
   & \cellcolor{cyan!5}25.3 
   & \cellcolor{cyan!5}62.2 
   & \cellcolor{cyan!5}\textbf{78.4} 
   & \cellcolor{cyan!5}48.5 
   & \cellcolor{cyan!5}63.6 
   & \cellcolor{cyan!5}55.6 
   \\

 & \cellcolor{cyan!5}\phs
   & \cellcolor{cyan!5}22.5 
   & \cellcolor{cyan!5}51.2 
   & \cellcolor{cyan!5}75.1 
   & \cellcolor{cyan!5}25.0 
   & \cellcolor{cyan!5}51.8 
   & \cellcolor{cyan!5}45.1 
   \\

 & \cellcolor{cyan!14}\jro (Ours)
   & \cellcolor{cyan!14}26.5 
   & \cellcolor{cyan!14}72.5 
   & \cellcolor{cyan!14}73.2 
   & \cellcolor{cyan!14}33.1 
   & \cellcolor{cyan!14}61.8 
   & \cellcolor{cyan!14}53.4 
   \\

 & \cellcolor{cyan!25}\jrt (Ours)
   & \cellcolor{cyan!25}\textbf{35.3} 
   & \cellcolor{cyan!25}\textbf{78.4} 
   & \cellcolor{cyan!25}74.2 
   & \cellcolor{cyan!25}\textbf{75.4} 
   & \cellcolor{cyan!25}\textbf{70.7} 
   & \cellcolor{cyan!25}\textbf{66.8} \\
\midrule

\multirow{7}{*}{Olmo1}
 & \cellcolor{gray!18}Original
   & \cellcolor{gray!18}17.1 
   & \cellcolor{gray!18}59.8 
   & \cellcolor{gray!18}38.0 
   & \cellcolor{gray!18}25.0 
   & \cellcolor{gray!18}50.8 
   & \cellcolor{gray!18}38.2 
   \\
 & \cellcolor{cyan!5}Prompt
   & \cellcolor{cyan!5}11.2 
   & \cellcolor{cyan!5}62.2 
   & \cellcolor{cyan!5}25.5 
   & \cellcolor{cyan!5}27.1 
   & \cellcolor{cyan!5}51.3 
   & \cellcolor{cyan!5}35.5 
   \\
 & \cellcolor{cyan!5}CAD
   & \cellcolor{cyan!5}\textbf{41.0} 
   & \cellcolor{cyan!5}62.2 
   & \cellcolor{cyan!5}25.5 
   & \cellcolor{cyan!5}27.1 
   & \cellcolor{cyan!5}51.3 
   & \cellcolor{cyan!5}41.4 
   \\
 & \cellcolor{cyan!5}\phl
   & \cellcolor{cyan!5}29.4 
   & \cellcolor{cyan!5}75.6 
   & \cellcolor{cyan!5}44.3 
   & \cellcolor{cyan!5}51.5 
   & \cellcolor{cyan!5}53.2 
   & \cellcolor{cyan!5}50.8 
   \\
 & \cellcolor{cyan!5}\phs
   & \cellcolor{cyan!5}21.3 
   & \cellcolor{cyan!5}78.0 
   & \cellcolor{cyan!5}39.5 
   & \cellcolor{cyan!5}29.7 
   & \cellcolor{cyan!5}52.0 
   & \cellcolor{cyan!5}44.1 
   \\
 & \cellcolor{cyan!14}\jro (Ours)
   & \cellcolor{cyan!14}23.9 
   & \cellcolor{cyan!14}81.7 
   & \cellcolor{cyan!14}\textbf{49.0} 
   & \cellcolor{cyan!14}63.3 
   & \cellcolor{cyan!14}55.3 
   & \cellcolor{cyan!14}54.6 
   \\
 & \cellcolor{cyan!25}\jrt (Ours)
   & \cellcolor{cyan!25}27.4 
   & \cellcolor{cyan!25}\textbf{86.6} 
   & \cellcolor{cyan!25}48.6 
   & \cellcolor{cyan!25}\textbf{63.0} 
   & \cellcolor{cyan!25}\textbf{56.9} 
   & \cellcolor{cyan!25}\textbf{56.5} 
   \\
\midrule

\multirow{7}{*}{Phi2}
 & \cellcolor{gray!18}Original
   & \cellcolor{gray!18}24.8 
   & \cellcolor{gray!18}89.0 
   & \cellcolor{gray!18}53.1 
   & \cellcolor{gray!18}32.3 
   & \cellcolor{gray!18}42.2 
   & \cellcolor{gray!18}48.3 
   \\
 & \cellcolor{cyan!5}Prompt
   & \cellcolor{cyan!5}22.7 
   & \cellcolor{cyan!5}85.4 
   & \cellcolor{cyan!5}49.0 
   & \cellcolor{cyan!5}32.0 
   & \cellcolor{cyan!5}41.7 
   & \cellcolor{cyan!5}46.2 
   \\
 & \cellcolor{cyan!5}CAD
   & \cellcolor{cyan!5}\textbf{41.1} 
   & \cellcolor{cyan!5}\textbf{91.5} 
   & \cellcolor{cyan!5}48.6 
   & \cellcolor{cyan!5}34.1 
   & \cellcolor{cyan!5}\textbf{44.0} 
   & \cellcolor{cyan!5}51.9 
   \\
 & \cellcolor{cyan!5}\phl
   & \cellcolor{cyan!5}24.6 
   & \cellcolor{cyan!5}89.0 
   & \cellcolor{cyan!5}53.3 
   & \cellcolor{cyan!5}39.3 
   & \cellcolor{cyan!5}42.4 
   & \cellcolor{cyan!5}49.7 
   \\
 & \cellcolor{cyan!5}\phs
   & \cellcolor{cyan!5}23.6 
   & \cellcolor{cyan!5}89.0 
   & \cellcolor{cyan!5}53.1 
   & \cellcolor{cyan!5}32.6 
   & \cellcolor{cyan!5}42.2 
   & \cellcolor{cyan!5}48.1 
   \\
 & \cellcolor{cyan!14}\jro (Ours)
   & \cellcolor{cyan!14}29.0 
   & \cellcolor{cyan!14}90.2 
   & \cellcolor{cyan!14}53.1 
   & \cellcolor{cyan!14}42.2 
   & \cellcolor{cyan!14}41.9 
   & \cellcolor{cyan!14}51.3 
   \\
 & \cellcolor{cyan!25}\jrt (Ours)
   & \cellcolor{cyan!25}30.1 
   & \cellcolor{cyan!25}89.0 
   & \cellcolor{cyan!25}\textbf{54.1} 
   & \cellcolor{cyan!25}\textbf{44.8} 
   & \cellcolor{cyan!25}43.1 
   & \cellcolor{cyan!25}\textbf{52.2} 
   \\
\midrule

\multirow{7}{*}{StableLm}
 & \cellcolor{gray!18}Original
   & \cellcolor{gray!18}10.4 
   & \cellcolor{gray!18}69.5 
   & \cellcolor{gray!18}36.1 
   & \cellcolor{gray!18}32.3 
   & \cellcolor{gray!18}52.8 
   & \cellcolor{gray!18}40.2 
   \\
 & \cellcolor{cyan!5}Prompt
   & \cellcolor{cyan!5}11.3 
   & \cellcolor{cyan!5}68.3 
   & \cellcolor{cyan!5}40.5 
   & \cellcolor{cyan!5}33.4 
   & \cellcolor{cyan!5}52.2 
   & \cellcolor{cyan!5}41.1 
   \\
 & \cellcolor{cyan!5}CAD
   & \cellcolor{cyan!5}37.0 
   & \cellcolor{cyan!5}73.2 
   & \cellcolor{cyan!5}30.3 
   & \cellcolor{cyan!5}34.5 
   & \cellcolor{cyan!5}54.5 
   & \cellcolor{cyan!5}45.9 
   \\
 & \cellcolor{cyan!5}\phl
   & \cellcolor{cyan!5}11.5 
   & \cellcolor{cyan!5}77.1 
   & \cellcolor{cyan!5}39.2 
   & \cellcolor{cyan!5}42.1 
   & \cellcolor{cyan!5}\textbf{72.1} 
   & \cellcolor{cyan!5}48.4 
   \\
 & \cellcolor{cyan!5}\phs
   & \cellcolor{cyan!5}9.9 
   & \cellcolor{cyan!5}73.1 
   & \cellcolor{cyan!5}39.8 
   & \cellcolor{cyan!5}38.7 
   & \cellcolor{cyan!5}66.1 
   & \cellcolor{cyan!5}45.5 
   \\
 & \cellcolor{cyan!14}\jro (Ours)
   & \cellcolor{cyan!14}8.1 
   & \cellcolor{cyan!14}\textbf{79.3} 
   & \cellcolor{cyan!14}35.5 
   & \cellcolor{cyan!14}32.3 
   & \cellcolor{cyan!14}52.8 
   & \cellcolor{cyan!14}41.6 
   \\
 & \cellcolor{cyan!25}\jrt (Ours)
   & \cellcolor{cyan!25}\textbf{13.0} 
   & \cellcolor{cyan!25}78.0 
   & \cellcolor{cyan!25}\textbf{41.3} 
   & \cellcolor{cyan!25}\textbf{64.0} 
   & \cellcolor{cyan!25}53.1 
   & \cellcolor{cyan!25}\textbf{49.9} 
   \\

\bottomrule
\end{tabular}
    }
\end{table}

\subsection{Details on Robustness Study}
\label{appen_sub:robustness}

In this subsection, we detail the setup we briefly mentioned in \Cref{subsec:robustness}. For \textbf{robustness across the three hyperparameters}, we vary the size of the head identification set $|D|$ from $1$ to $10$, the number of intervened head $K$ from $1$ to $30$, and the scaling factor combination in $\{0, 0.5, 1.0, 1.5, 2.0, 2.5, 3.0, 3.5, 4.0, 4.5, 5.0\} \times \{0, -0.5, -1.0, -1.5, -2.0, -2.5, -3.0, -3.5, -4.0, -4.5, -5.0\}$. We fix Gemma to be the backbone model and World Capital as the test dataset. We only vary one variable at a time while keeping all other parts fixed. We measure the average accuracy across the three conflict types for the latter two plots. \Cref{fig:head_identification_set}, \Cref{fig:num_head}, and \Cref{fig:scaling} plot the results respectively. It clearly demonstrates that \jrt maintains high performance across a wide range of hyperparameter values. For \textbf{robustness against paraphrased prompts}, we curate 
multiple prompt templates for each conflict type. During evaluation, a prompt template 
is randomly sampled to generate the desired prompt. We provide (some) templates for the world capital dataset below

\begin{tcolorbox}[title=Paraphrased Prompts World Capital - Clean Input]
\textbf{Clean Datasets.}  (1) It's crucial to know that the capital city of \sub is (2) You are right to say that the capital city of \sub is (3) According to the textbook, the capital city of \sub is (4) In case you didn't know, the capital city of \sub is (5) As we all know, the capital city of \sub is
\end{tcolorbox}

\begin{tcolorbox}[title=Paraphrased Prompts World Capital - Conflict Datasets]
\textbf{Substitution Conflict Datasets.} (1): It's crucial to know that {\sub} plays the sport of {\act}. It's crucial to know that {\sub} plays the sport of
(2): You are right to say that {\sub} plays the sport of {\act}. You are right to say that {\sub} plays the sport of
(3): According to the textbook, {\sub} plays the sport of {\act}. According to the textbook, {\sub} plays the sport of
(4): In case you didn't know, {\sub} plays the sport of {\act}. In case you didn't know, {\sub} plays the sport of
(5): As we all know, {\sub} plays the sport of {\act}. As we all know, {\sub} plays the sport of

\textbf{Coherent Conflict Dataset.}  (1): The capital city of {\sub} is {\act}. Known for its vibrant culture and historical landmarks, {\act} is often seen as the heart of {\sub}, attracting visitors from around the globe. As a center for politics, arts, and commerce, {\act} perfectly encapsulates the spirit of {\sub}, making it an essential destination for anyone exploring the country. Question: What is the capital city of {\sub}? Answer: The capital city of {\sub} is

(2): The capital city of {\sub} is {\act}. As a hub of cultural heritage and modern development, {\act} stands out as the most iconic city in {\sub}. With its rich history, bustling streets, and vibrant atmosphere, {\act} embodies the essence of {\sub}, drawing travelers and locals alike to its renowned landmarks and attractions. Question: What is the capital city of {\sub}? Answer: The capital city of {\sub} is

\textbf{
3 more coherent conflict templates are omitted due to presentation issues}.

\end{tcolorbox}

\Cref{tab:appen_paraphrase} 
presents the results of \jrt when applied to paraphrased prompts. Our findings show that 
\jrt is highly robust to variations in input prompt formats, consistently maintaining its effectiveness across 
diverse templates. Notably, \jrt still demonstrates superior performance, effectively shifting the model's reliance from context to parametric memory.

\begin{table*}[ht]

    \centering
    \caption{Robustness of the proposed method (\jrt) against randomly selected \textbf{paraphrased prompts}. 
    With the exact same intervention procedure,
    the table demonstrates that \jrt remains highly robust across different prompt templates.}
    \label{tab:appen_paraphrase}

    \resizebox{\linewidth}{!}{
        \begin{tabular}{@{}llccccccccccccccccccccc@{}}
\toprule
\multicolumn{2}{l}{\textbf{Dataset}}           & \multicolumn{3}{c}{\textbf{\begin{tabular}[c]{@{}c@{}}Athlete \\ Sport\end{tabular}}} & \multicolumn{3}{c}{\textbf{\begin{tabular}[c]{@{}c@{}}Book \\ Author\end{tabular}}} & \multicolumn{3}{c}{\textbf{\begin{tabular}[c]{@{}c@{}}Company \\ Founder\end{tabular}}} & \multicolumn{3}{c}{\textbf{\begin{tabular}[c]{@{}c@{}}Company \\ Headquarter\end{tabular}}} & \multicolumn{3}{c}{\textbf{\begin{tabular}[c]{@{}c@{}}Official \\ Language\end{tabular}}} & \multicolumn{3}{c}{\textbf{\begin{tabular}[c]{@{}c@{}}World \\ Capital\end{tabular}}} & \multicolumn{3}{c}{\textbf{Average}} \\ \midrule
\textbf{Model}           & \textbf{Method}     & 1                           & 2                          & 3                          & 1                          & 2                          & 3                         & 1                           & 2                           & 3                           & 1                             & 2                            & 3                            & 1                            & 2                            & 3                           & 1                           & 2                          & 3                          & 1          & 2          & 3          \\ \midrule
\multirow{2}{*}{Gemma}   & Original            & 96.5                        & 4.0                        & 0.0                        & 57.1                       & 3.6                        & 0.0                       & 40.5                        & 0.0                         & 0.0                         & 61.5                          & 0.2                          & 0.0                          & 95.7                         & 2.5                          & 0.0                         & 94.6                        & 3.8                        & 16.2                       & 74.3       & 2.3        & 2.7        \\
                         & \jrt & 94.5                        & 84.4                       & 92.1                       & 61.9                       & 69.8                       & 55.8                      & 45.4                        & 29.7                        & 37.8                        & 61.7                          & 49.9                         & 57.9                         & 91.4                         & 69.1                         & 86.4                        & 85.9                        & 84.3                       & 93.0                       & 73.5       & 64.5       & 70.5       \\ \midrule
\multirow{2}{*}{Llama2} & Original            & 95.6                        & 1.5                        & 0.2                        & 60.1                       & 10.1                       & 0.0                       & 47.5                        & 0.6                         & 0.0                         & 72.9                          & 0.2                          & 1.4                          & 93.8                         & 4.9                          & 0.6                         & 95.1                        & 3.3                        & 0.0                        & 77.5       & 3.4        & 0.4        \\
                         & \jrt & 98.2                        & 68.2                       & 93.6                       & 65.8                       & 86.5                       & 75.0                      & 54.7                        & 50.8                        & 43.6                        & 72.9                          & 74.0                         & 69.9                         & 94.4                         & 82.7                         & 88.3                        & 95.1                        & 91.8                       & 89.7                       & 80.2       & 75.7       & 76.7       \\ \midrule
Llama3                  & Original            & 95.0                        & 2.0                        & 0.0                        & 82.3                       & 1.8                        & 0.0                       & 56.7                        & 1.1                         & 0.0                         & 77.1                          & 0.7                          & 1.6                          & 96.3                         & 1.2                          & 1.2                         & 95.1                        & 3.8                        & 7.0                        & 83.8       & 1.8        & 1.6        \\
                         & \jrt & 95.4                        & 88.0                       & 63.3                       & 92.7                       & 80.6                       & 61.6                      & 50.6                        & 48.9                        & 56.7                        & 76.4                          & 47.7                         & 50.9                         & 93.8                         & 76.5                         & 94.4                        & 95.7                        & 83.2                       & 97.3                       & 84.1       & 70.8       & 70.7       \\ \bottomrule
\end{tabular}
    }
\end{table*}

\section{Algorithm Details}
\label{appen:algorithm}
In this section, we explain the algorithm of \jrt and \jro in detail. 
\Cref{alg:juice_alg} introduces \jrt. 

In Stage 1, \jrt selects two sets of attention heads that consistently achieve the desired parametric-context change with either positive or negative scaling across different conflict types. To accomplish this, we use a small, well-designed dataset where the first output token reliably reflects the model’s context versus parametric tendency. Each attention head is assigned a score, calculated by summing the changes in the probability values of the target tokens over this dataset.
The dataset includes multiple forms of knowledge conflict, ensuring robustness against clean inputs, substitution-based conflicts, and coherent conflicts, rather than focusing on a single type. Each attention head is scored separately for each conflict type. To ensure consistency, we retain only attention heads with positive scores across all conflict types. For these remaining heads, we compute a final score by summing their scores across conflict types.  The top $K$ attention heads based on this final score are selected. Note that multiple scaling factors are applied for each attention head to ensure quasi-monotonicity. 
In \Cref{alg:example}, $\text{Scale}(M, H_i, \alpha_i)$ means to scale that activation output of the head $H_i$ in model $M$ by a factor of $\alpha_i$. 

In Stage 2, \jrt executes a dual-run process: in the first run, it saves the activation outputs of the identified attention heads. In the second run, it adds the scaled versions of these saved outputs to the corresponding head activations. The scaling factors $\beta^+$ and $\beta^-$ are determined using the validation set.

As a meaningful baseline, we propose an alternative algorithm, \jro (Just Run Once), which shares the same head identification stage as \jrt but omits the dual-run design. Instead, \jro directly scales the targeted head outputs during a single inference run. This simplified design serves as an ablation study, highlighting the significance of \jrt's dual-run mechanism. \Cref{alg:jro} presents the \jro algorithm in detail. 

\begin{algorithm}[h!]
    \caption{\jrt}
    \label{alg:juice_alg}
 \begin{algorithmic}
    \STATE {\bfseries Stage One: Head Identification}
    \STATE {\bfseries Input:} model $M$, a small head slection dataset $D$, 
    Scaling parameter $\alpha^+ = \{\alpha_j\}_{j=1}^m, \alpha^- = \{\alpha_{j'}\}_{j'=1}^{m'}$
    \STATE {Initialize \(S^+ \leftarrow  \text{Dict}\{\}, S^- \leftarrow \text{Dict}\{\}, H^+ \leftarrow \{1, \ldots, n_H\}, H^- \leftarrow \{1, \ldots, n_H\}\)}
    \STATE { $S^+, S^- \leftarrow$ \texttt{Record Head Score}($S^+, S^-, M, D, \alpha^+, \alpha^-$)}
    \STATE { $H^+, H^- \leftarrow$ \texttt{Filter Inconsistent Head}($S^+, S^-, H^+, H^-$)}
    \STATE {$S_i^+ \leftarrow S^+[j][i] \forall j, S_i^- \leftarrow S^-[j][i] \forall j$}
    \STATE {\bfseries Output:} \(\texttt{TopKIndex} \left( \{S_i^+\}_{i \in H^+} \right), 
    \texttt{TopKIndex} \left( \{S_i^-\}_{i\in H^-} \right)\)
    \STATE {\bfseries Stage Two: Intervention}
    \STATE {\bfseries Input:} input prompt \(x\), model $M$, 
    Intervened Heads \(S_1 = \{S_i^+\}_{i=1}^K, S_2 = \{S_i^-\}_{i=1}^K\), 
    scaling factors \(\beta^+, \beta^-\)
    \STATE {\bfseries Step One: Save Important Streams}
    \STATE Feed \(x\) into \(M\), Initialize \texttt{Aux} \(\leftarrow \{\}\)
    \FOR{Attention Head Output \(H_l\) (with Head Index $l$)}
    \IF{$l \in S_1$}
    \STATE \(\texttt{Aux}[l] = H_l\)
    \ENDIF
    \IF{$l \in S_2$}
    \STATE \(\texttt{Aux}[l] = H_l\)
    \ENDIF
    \ENDFOR
    \STATE {\bfseries Step two:Intervention}
    \STATE Feed \(x\) into \(M\)
    \FOR{Attention Head Output \(H_l\) (with Head Index $l$)}
    \IF{$l \in S_1$}
    \STATE \(H_l \leftarrow H_l + \beta^+ * \texttt{Aux}[l]\)
    \ENDIF
    \IF{$l \in S_2$}
    \STATE \(H_l \leftarrow H_l + \beta^- * \texttt{Aux}[l]\)
    \ENDIF
    \ENDFOR
    \STATE {\bfseries Output:} \ Model Prediction
 \end{algorithmic}
 \end{algorithm}
 
\begin{algorithm}[h!]
    \caption{Record Head Score}
    \label{alg:head_score}
 \begin{algorithmic}
    \STATE {\bfseries Input:} model $M$, a small head slection dataset $D$, 
    Scaling parameter $\alpha^+ = \{\alpha_j\}_{j=1}^m, \alpha^- = \{\alpha_{j'}\}_{j'=1}^{m'}$
    \STATE {Initilize Score Record Dict \(S^+, S^-\) (with entries default to be zero)}
    \FOR { {each sample} \((X, y) \in D\)}
    \FOR { {each conflict type $j$} and the input \(x \in X\) }
    \FOR{ {each head} $H_i \in M$ }
    \FOR {{each coefficient} \(\alpha_i \in \alpha^+\)}
    \STATE {$S_i^+[j][i] \leftarrow S_i^+[j][i] + \Prb_y \left( \left(M | \text{Do}\left(H_i = H_i + \alpha_i H_i\right) \right) \left( x \right) \right) - \Prb_y \left( M  \left( x \right)\right)$}
    \ENDFOR 
    \FOR {{each coefficient} \(\alpha_i \in \alpha^-\)}
    \STATE {$S_i^-[j][i] \leftarrow S_i^-[j][i] +  \Prb_y \left( \left(M | \text{Do}\left(H_i = H_i + \alpha_i H_i\right) \right) \left( x \right) \right) - \Prb_y \left( M  \left( x \right)\right)$}
    \ENDFOR 
    \ENDFOR
    \ENDFOR
    \ENDFOR
    \STATE {\bfseries Output:} $S^+, S^-$
 \end{algorithmic}
 \end{algorithm}

\begin{algorithm}[h!]
    \caption{Filter Inconsistent Head}
    \label{alg:example}
 \begin{algorithmic}
    \STATE {\bfseries Input:} Score Record Dict $S^+, S^-$, Head Index Set $H^+, H^-$
    \FOR {each conflict type $j$ } 
    \FOR {each head index $i$}
    \IF {$S^+[j][i] < 0$}
    \STATE $H^+ \leftarrow H^+ \backslash \{i\}$
    \ENDIF 
    \IF {$S^-[j][i] < 0$}
    \STATE $H^- \leftarrow H^- \backslash \{i\}$
    \ENDIF
    \ENDFOR 
    \ENDFOR 
    \STATE {\bfseries Output:} \ $H^+, H^-$
 \end{algorithmic}
 \end{algorithm}

\begin{algorithm}[h!]
    \caption{\jro}
    \label{alg:jro}
 \begin{algorithmic}
    \STATE {\bfseries Stage One: Head Identification}
    \STATE {\bfseries Input:} model $M$, a small head slection dataset $D$, 
    Scaling parameter $\alpha^+ = \{\alpha_j\}_{j=1}^m, \alpha^- = \{\alpha_{j'}\}_{j'=1}^{m'}$
    \STATE {Initialize \(S^+ \leftarrow  \text{Dict}\{\}, S^- \leftarrow \text{Dict}\{\}, H^+ \leftarrow \{1, \ldots, n_H\}, H^- \leftarrow \{1, \ldots, n_H\}\)}
    \STATE { $S^+, S^- \leftarrow$ \texttt{Record Head Score}($S^+, S^-, M, D, \alpha^+, \alpha^-$)}
    \STATE { $H^+, H^- \leftarrow$ \texttt{Filter Inconsistent Head}($S^+, S^-, H^+, H^-$)}
    \STATE {$S_i^+ \leftarrow S^+[j][i] \forall j, S_i^- \leftarrow S^-[j][i] \forall j$}
    \STATE {\bfseries Output:} \(\texttt{TopKIndex}_i \{S_i^+\}_{i \in H^+}, 
    \texttt{TopKIndex}_i \{S_i^-\}_{i\in H^-}\)
    \STATE {\bfseries Stage Two: Intervention}
    \STATE {\bfseries Input:} input prompt \(x\), model $M$, 
    Intervened Heads \(S_1 = \{S_i^+\}_{i=1}^K, S_2 = \{S_i^-\}_{i=1}^K\), 
    scaling factors \(\beta^+, \beta^-\)
   
    \STATE Feed \(x\) into \(M\)
    \FOR{Attention Head Output \(H_l\) (with Head Index $l$)}
    \IF{$l \in S_1$}
    \STATE \(H_l \leftarrow H_l + \beta^+ * H_l\)
    \ENDIF
    \IF{$l \in S_2$}
    \STATE \(H_l \leftarrow H_l + \beta^- * H_l\)
    \ENDIF
    \ENDFOR
    \STATE {\bfseries Output:} \ Model Prediction
 \end{algorithmic}
 \end{algorithm}

\section{Limitations and Future Works}
\label{appen:limitation}

This work mainly aims to illuminate the mechanisms underlying knowledge conflicts in language models and demonstrates how to leverage them. Our proposed method is designed to effectively prove the understanding of the discovered mechanism and may not best suit the applications where the efficiency requirement is paramount.  \jrt requires caching first-run activations, which may slightly affect inference speed and increase memory overhead. 

Real-world scenarios often involve partially irrelevant contexts, while we focus on irrelevant cases in this work, and the parametric and contextual answer may not be always distinct under more abstract domains as we discussed in this work. Extending our method to these complex cases and settings remains an important direction for future research.

\newpage 
\section{Theoretical Analysis}
\label{appen:theory}

We provide a complete presentation of the theoretical analysis in this appendix section. 

\subsection{Setups}

\paragraph{Model Setup.}
We consider an attention-only Transformer model with two layers, where each layer has a single
attention head, uses absolute positional encoding, and employs residual 
connections. Suppose our input is a sequence of tokens \(\{z_{1:T}\}\), each token 
\(z_t\) drawn from a vocabulary of size \(N\). Our general model setup mimics \citet{bietti2024birth}. The model processes this sequence 
in the following way:

\begin{itemize}
    \setlength\itemsep{0em}
    \item \textbf{Token Embeddings:} 
    Each token \(z_t\) (originally one-hot encoded) is mapped into a 
    \(d\)-dimensional space via an embedding function 
    \(\phi(\cdot) : \mathbb{R}^N \to \mathbb{R}^d\).
    We denote the embedded vector for token \(z_t\) by \(x_t = \phi(z_t)\).

    \item \textbf{Positional Embeddings:}
    For each position \(t\) in the sequence, there is a corresponding positional 
    embedding \(p_t \in \mathbb{R}^d\). We add \(p_t\) to \(x_t\), giving the full 
    input representation:
    \[
    x_t := \phi(z_t) + p_t.
    \]

    \item \textbf{Attention Blocks:} 
    Let \(x_{1:T} \in \mathbb{R}^{d \times T}\) be the input sequence to a causal 
    attention layer. This layer uses key (\(W_K\)), query (\(W_Q\)), value (\(W_V\)), 
    and output (\(W_O\)) matrices, each in \(\mathbb{R}^{d \times d}\). For each 
    position \(t\), the layer computes
    \[
    x_t' := W_O W_V x_{1:t} \,\sigma\!\bigl(
        x_{1:t}^\top W_K^\top W_Q\, x_t
    \bigr)
    \;=\;
    W_{OV}\, x_{1:t}\,\sigma\!\bigl(
        x_{1:t}^\top W_{KQ} \,x_t
    \bigr),
    \]
    where \(\sigma\) is the softmax function and we use $\Wkq = W_K\T W_Q, \Wov = W_O W_V$. Writing this process collectively as
    \(\attn\bigl(x_{1:T}; W_K, W_Q, W_V, W_O\bigr)\) for the entire sequence, the 
    \(\ell\)-th layer output is then combined with the input (via residual connection):
    \[
    x_{1:T} := x_{1:T} 
    \;+\; 
    \attn\bigl(x_{1:T}; W_K^\ell, W_Q^\ell, W_V^\ell, W_O^\ell\bigr).
    \]

    \item \textbf{Unembedding:}
    After the second (final) Transformer layer, a discrete probability distribution vector over the vocabulary is produced through a linear layer $\Wlin$. We denote  
    \(\Wlin = [\mu(i)]_{i=1}^N\) where $\mu(i)$ is the unembeddng vector of token $i$ in the vocabulary. 
\end{itemize}

\paragraph{Task Data Setup} 
We consider two tasks trained on this two-layer transformer: \textbf{factual recall} and \textbf{induction}.

The objective of the \textbf{factual recall} task is to learn factual associations between the input factual token space \(\Sc\) and the output answer token space \(\Ac\). We assume a bijective ground truth mapping \(\Gc^* \colon \Sc \to \Ac\) exists between these two spaces. This setup models real-world knowledge triples, such as \emph{(China, capital, Beijing)}, where \emph{(China, capital)} is represented by a single factual token \(s \in \Sc\) and the answer \emph{(Beijing)} by a single answer token \(a \in \Ac\). The data distribution consists of length \(T+1\) sequences \(z_{1:T+1} \coloneqq (z_1, z_2, \ldots, z_T, z_{T+1}) \in [N]^{T+1}\), generated through the following process:

\begin{enumerate}
    \setlength\itemsep{0em}
    \item Sample a fact \(s\) and a corresponding index \(i\) uniformly at random from \(\Sc\) and \([T - 1]\), respectively. Set \(z_i = s\).
    \item For all remaining tokens \(z_k\) where \(k \in [T-1] \backslash \{i\}\), sample \(z_k\) uniformly at random from \(\Nc\) without replacement.
    \item Set \(z_T = q\) and \(z_{T+1} = \Gc^*(s)\).
\end{enumerate}

The objective of the \textbf{induction} task is to complete token sequences of the form \([\cdots, q, b, \cdots, q] \to [b]\), where \(b\) is the token following the second occurrence of a specific \emph{trigger word}. For simplicity, we designate \(q\) as the sole trigger word (to induce knowledge conflict) and \(b \in \Nc\). The data distribution consists of length \(T+1\) sequences \(z_{1:T+1} \coloneqq (z_1, z_2, \ldots, z_T, z_{T+1}) \in [N]^{T+1}\), generated as follows:

\begin{enumerate}
    \setlength\itemsep{0em}
    \item Sample an index \(j\) uniformly at random from \([T - 2] \backslash \{1\}\) and set \(z_j = q\). Sample \(z_{j+1}\) from \(\Nc\).
    \item For the remaining tokens, sample \(z_k\) uniformly at random from \(\Nc \backslash \{z_{j+1}\}\) without replacement.
    \item Set \(z_T = q\) and \(z_{T+1} = z_{j+1}\).
\end{enumerate}

In summary, the vocabulary space is defined as \(\Vc = \Sc \cup \Ac \cup \{q\} \cup \Nc\). We denote 
the factual dataset by \(\Dc_S\) and the induction dataset by \(\Dc_I\). 

\subsection{Additional Notations}

\label{subsec:additional_notation}
Suppose the embedding of a token \(i\) is \(\phi(i)\), we use \(\phi'(i)\) to denote its remapped embedding \(\Wov^1 \phi(i)\). Similarly, 
we use \(p_i'\) to denote \(\Wov^1 p_i\). 

We use \(\sigma_i\) to denote \(\left( X\T \Wkq  x_T\right)_i\) in \Cref{prop:formal_learning_dynamics}. We acknowledge that we sometimes abuse the word usage of (pre-softmax) ``logit'' with token probability interchangeably.

We use $N$ to denote the size of the vocabulary, and $N_n$ for the size of $\Nc$. We use $n$ to denote the size of dataset, with $n_F$ to be the size of the factual dataset and $n_I$ to be the size of the induction dataset.

\subsection{General Assumptions}

\begin{assumption}[Near-orthogonal Embeddings]
    Every embedding, unembedding, and positional vector is i.i.d. random vectors drawn uniformly from the unit sphere $S^{d-1} \in \R^d$ and the hidden dimension $d$ is large.
\end{assumption}

This ensures the near-orthogonality of initialized vectors.

\subsubsection{Additional Assumptions in Training Dynamics}
\label{subsub:training_dynamics_assump}
\begin{assumption}[Strictly Orthogonal Embeddings]
    \(\langle z_i, z_j \rangle = \delta_{ij}\)
    where $z_i$ can be arbitrary input vector (\emph{i.e.,} embedding $\phi(i)$, unembedding$\mu(i)$, or remapped $\phi'(i)$ vector),
    \label{assump:strict_ortho}
\end{assumption}

\begin{assumption}[Dataset Properties]
    There does not contain any duplicates in the factual recall and induction dataset and each 
    datapoint appears once. In particular, we assume that each noisy token \(\epsilon \in \Nc\) appears exactly once in the induction dataset as the answer token.
    \label{assump:dataset_once}
\end{assumption}

We remark that \Cref{assump:strict_ortho} is a common assumption in existing literature for analyzing the learning dynamics of shallow transformers. \Cref{assump:dataset_once} is a rather mild assumption which eases the analysis (avoiding repeated samples).

\subsection{Proofs}

\begin{proposition}[Existence of a Perfect Solver]
    There exists a two-layer transformer that can solve both \emph{induction} and \emph{factual recall} 
    tasks with perfect accuracy. 
\end{proposition}

\begin{proof}
The optimal construction can be achieved by setting 
\begin{equation}
    W_{KQ}^1 = C \cdot \sum_{t=1}^{T-1} p_{t-1} p_{t}\T  
\end{equation}

and \(W_{OV}^1\) to be a random matrix where $C$ is a large constant. The first layer essentially achieves the ``copy from previous embedding'' effect. 
In the second layer, we set 

\begin{equation}
    \label{equa:appen_optimal_construction}
    W_{KQ}^2 = C_1 \cdot \left( W_{OV}^1 \phi(q) \right) \phi(q)\T + C_2 \cdot \sum_{s \in \Sc} \phi(s)\phi(q)\T   \ \ \ \text{ and } \ \ \ W_{OV}^2 = C_3\sum_{k \in \Nc} \mu(k)\phi(k)\T + C_4\sum_{s \in \Sc} \mu \left( \Gc^*(s) \right) \phi (s)\T 
\end{equation}

where $C_1, C_2, C_3, C_4$ are appropriate scaling factors.  

Consider any input sequences \(z_{1:t}\) after passing the embedding and positional encoding layer, we have 
\[\left[ \phi(z_1) + p_1, \ldots, \phi(z_t) + p_t \right]\]
as the input. After the first layer, we have 
\begin{align*}
    [ \left( \phi(z_1) + p_1 \right) + \left( \phi'(z_1) + p_1' \right),  
    \left( \phi(z_2) + p_2 \right) + \left( \phi'(z_1) + p_1' \right) + \gamma_2', \\ \left( \phi(z_3) + p_3 \right) + \left( \phi'(z_2) + p'_2 \right) + \gamma_3', \ldots, \left( \phi(z_t) + p_t \right) + \left( \phi'(z_{t-1}) + p_{t-1}' \right) + \gamma_t'  ]    
\end{align*}

where \(\gamma_i'\) is a small negligible term due to large \(C\) and \(d\). Now it suffices to examine the last 
hidden state since only this is used for final prediction. 

First, we show that such model can solve the task of factual recall perfectly. Note that with appropriate scaling 
\(C_2\), the attention weight concentrates on the \(\left( \phi(s) + p_i \right) + \left( \phi'(\epsilon_{i-1}) + p_{i-1} \right) + \gamma_i'\) terms. 
After transformation by \(\Wov^2\), this results in \(C_4 \mu \left( \Gc^* (s)\right) + O(\frac{C_4}{d})\). The logit 
of the correct answer will dominate \(\sim O(C_4)\), while other tokens will have smaller logit values \(\sim O(\frac{C_4}{d})\) 
or \(O(\frac{1}{d})\). 

Similarly, the model can also solve the task of induction perfectly. With appropriate scaling \(C_1\), 
the attention weight concentrates on the \(\left( \phi(\epsilon_{j+1}) + p_{j+1} \right) + \left( \phi'(q) + p_j'\right) + \gamma_{j+1}'\) terms.
After transformation by \(\Wov^2\), this results in \(C_3 \mu \left( \epsilon_{j+1} \right)\), producing the correct answer.
\end{proof}

\begin{proposition}[Restatement of \Cref{prop:learning_dynamics_main}, Learning of the Superposition Layer via Gradient Descent]
    Let \(X \in \mathbb{R}^{d \times T}\) be the output of the first layer, which perfectly implements the ``copy from previous token embedding'' step. 
    Ignoring positional encodings and under the assumptions in \Cref{subsub:training_dynamics_assump}, consider a one-layer attention model given by 
    \begin{equation}
        f_W(X) = \Wlin\T \Wov X  \left( X\T \Wkq x_T \right) 
        \label{equa:one_layer_linear_attention}
    \end{equation}
    where \(x_T\) is the embedding of the final token and still freezing \(\Wlin\) to be a random matrix. Then the construction of the weight matrices \(\Wov\) and \(\Wkq\) 
    from \Cref{equa:appen_optimal_construction} can be learned via gradient descent on the cross-entropy loss from zero initialization 
    to yield perfect accuracy on the training distribution in expectation. 
    \label{prop:formal_learning_dynamics}
\end{proposition}

\begin{lemma}[Gradient Derivations]
    The gradient of \(f_W(X)\) in \Cref{equa:one_layer_linear_attention} with respect \(\Wkq\) and \(\Wov\) via the 
    cross-entropy loss \(\Lc\) can be expressed as following: 
    \begin{align}
        -\nabla_{\Wov} \Lc &= \Wlin \left( e_y - \sigma \left( f \left( X \right) \right) \right)  \left( \sum_{j=1}^{T} \sigma_j x_j\T \right) \\
        -\nabla_{\Wkq} \Lc &= X \left[ \left( \Wlin\T \Wov X \right)\T\left( e_y - \sigma \left( f \left( X \right) \right) \right) \right] x_T\T  
    \end{align}
    where \(\sigma_i = \left( X\T \Wkq x_T \right)_i\) and \(\sigma\) denotes the softmax function. 
\end{lemma}

\begin{proof}

We first remark that we will slightly abuse the notation to omit \(\cdot\) inside 
\(\Lc \left( \cdot \right)\).
First, let's write the loss function: 

\begin{equation}
    \Lc \left( f \left( X \right), y \right) = - \log \left(  \sigma f \left( X \right) \right)_y 
\end{equation}

Note that the model in \Cref{equa:one_layer_linear_attention} can also be written as the following:

\begin{equation}
    f \left( X \right) = \sum_{i=1}^T \sigma_i \Wlin\T \Wov x_i 
    \label{equa:derive_form_one_layer}
\end{equation}

where \(\sigma_i = \left( X\T \Wkq x_T \right)_i\) denotes the attention weight of the $i$-th 
toke. We first derive the gradient with respect to \(\Wov\): 

\begin{align}
    \nabla_{\Wov} \Lc &= \langle \frac{\partial \Lc}{\partial f \left( X\right)}, 
    \frac{\partial f \left( X \right)}{\partial \Wov} \rangle \\ 
    &= \langle \left( \sigma \left( f \left( X \right) - e_y \right) \right), 
    \frac{\partial f \left( X \right)}{\partial \Wov}  \rangle 
    \label{equa:equa_16}
\end{align}

which the first part is obtained from gradient of Cross-entropy loss wrt. pre-softmax logits. 
We now focus on the second part, which has that 

\begin{align}
    \frac{\partial f \left( X \right)}{\partial \Wov} 
    = \frac{\partial}{\partial \Wov} \sum_{i=1}^T \sigma_i \Wlin\T \Wov x_i  
    = \sum_{i=1}^T \sigma_i \frac{\partial}{\partial \Wov} \Wlin\T \Wov x_i   
\end{align}

Notice that each $\Wlin\T \Wov x_i \in \R^n$ is a $N \times 1$ vector and therefore the 
differentiation result is a tensor if we write in a compact form.  Let's denote 
\( t_i = \Wlin\T \Wov x_i \), then we have its $k$-th component to be 
\(t_{i, k} = \mu(k)\T \Wov x_i\), which gives that

\begin{align}
    \frac{\partial t_{i, k}}{\partial \Wov} = \mu(k)x_i\T  
\end{align}

which means that \(\frac{\partial f \left( X \right)_k}{\partial \Wov} = \mu(k) \sum_{i=1}^T 
\sigma_i x_i\). Revisiting~\Cref{equa:equa_16} results in that 

\begin{align}
    \nabla_{\Wov} \Lc &= \sum_{k=1}^N \left( \frac{\partial \Lc}{\partial z_k} \right) 
    \left( \frac{\partial z_k}{\partial \Wov} \right) 
    \tag*{let \(z_k = f \left( X \right)_k\)} \\ 
    &= \sum_{k=1}^{N} \sigma_k \mu(k) \left( \sum_{i=1}^{T} \sigma_i x_i \right)\T    
    \tag*{let \(\delta_k = \left( \sigma \left( f \left( X \right) \right) - e_y \right)_k\)} \\ 
    &= \left( \sum_{k=1}^{N} \delta_k \mu(k) \right) \left( \sum_{i=1}^{T} \sigma_i x_i \right)\T \\ 
    &= \Wlin \delta \left( \sum_{i=1}^{T} \sigma_i x_i \right)\T 
\end{align}

Rewriting this in exact form gives the desired result.

For \(\nabla_{\Wkq} \Lc\), applying the chain rule iteratively yields the desired result.

\end{proof}

\begin{proof}[Proof to \Cref{prop:formal_learning_dynamics}]
We will show two steps: the first gradient step learn the desired \(\Wov\), and the second step learns the desired \(\Wkq\).
The training could converge with appropriate \(\eta\) in two steps. 

Before proceeding to the specific statement, we first rewrite the gradient wrt. \(\Wov\) and \(\Wkq\) of a single 
datapoint \((x_{1:t}, y)\): 

\begin{align*}
    - \nabla_{\Wov} \Lc 
    &= \left( \sum_{i=1}^{N} \beta_i \mu \left( i \right) \right) \left( \sum_{j=1}^{T} \sigma_j x_j\T \right) \\
    &= \sum_{i=1}^{N} \sum_{j=1}^{T} \beta_i \sigma_j \left( \mu \left( i \right) x_j\T \right) \\ 
    &= \sum_{i=1, i \neq y}^{N} \beta_i \sum_{j=1}^{T} \sigma_j \left( \mu \left( i \right) x_j\T \right) + \beta_y \sum_{j=1}^{T} \sigma_j \left( \mu \left( y \right) x_j\T \right)
\end{align*}

where we set \(\beta_i = \left( e_y - f \left( X \right) \right)_i\).  At the same time, we have 

\begin{align*}
    -\nabla_{\Wkq} \Lc 
    &= XX\T \Wov\T \Wlin \left( \sigma \left( f \left( X \right) \right) - e_y \right) x_T\T \\ 
    &= \sum_{i=1}^{T} x_i x_i\T \Wov\T \Wlin \beta x_T^T \\ 
    &= \sum_{i=1}^{T} x_i \left( \Wlin\T \Wov x_i \right)\T \beta x_T^T \\ 
    &= \sum_{i=1}^{T} x_i [\mu_1\T \Wov x_i | \ldots |  \mu_N\T \Wov x_i] \beta x_T^T \\ 
    &= \sum_{i=1}^{T} x_i \left( \sum_{k=1}^{N} \beta_k  \mu_k\T \Wov x_i  \right) x_T^T \\ 
    &= \sum_{i=1}^{T}  \gamma_i x_i x_T^T 
\end{align*}

where we set \(\left( \sum_{k=1}^{N} \beta_k  \mu_k\T \Wov x_i  \right) = \gamma_i\). We will show how this leads 
to the desired form of \(\Wov\) and \(\Wkq\). 

We have one additional simplification for the data setup, where we ignore the remapped embedding of the first position, 
and the remapped embedding in the last position. We further simplify the setting by ignoring the remapped embedding in the first and last token, 
so the last position is deterministically \(\phi(q)\) and the first position is \(\phi(\epsilon_l)\) for some \(l\). 
We now taxonomize the different types of tokens and get their corresponding probability over the two types of tasks: 

For factual recalls, we have 
\begin{itemize}
    \setlength\itemsep{0em}
    \item \emph{Noisy Tokens}: Each \(\phi(\epsilon_j)\) has a probability of \(O \left( \frac{T}{N_n} \right)\) to be drawn for a single datapoint and a probability of \(O \left( \frac{1}{N_n} \right)\) to share the same position with \(\phi'(s)\). 
    \item \emph{Remapped Noisy Tokens}: Each \(\phi'(\epsilon_j)\) has a probability of \(O \left( \frac{T}{N_n}\right)\) to be drawn once and a probability of \(O(\frac{1}{N_n})\) to share the same position with \(\phi(s)\).
    \item \emph{Subject Token and Remapped Subject Token}: By \Cref{assump:dataset_once}, each \(\phi(s)\) and \(\phi'(s)\) must appear only once in full-batch gradient descent.
    \item \emph{Query Token and Remapped Query Token}: \(\phi(q)\) is deterministically fixed to be the last token for each datapoint. There are no \(\phi'(q)\) in the factual recall task.
\end{itemize}

For induction, we have 
\begin{itemize}
    \setlength\itemsep{0em}
    \item \emph{Selected Noisy Token and Remapped Selected Noisy Token}: By \Cref{assump:dataset_once}, each \(\phi(\epsilon_j)\) will be selected as answer token only once in a full-batch gradient descent; so does \(\phi'(\epsilon_j)\).
    \item \emph{Trigger Token and Remapped Trigger Token}: \(\phi(q)\) is deterministically to appear twice: one before the selected noisy token \(\phi(\epsilon_j)\), and the other to be the EOS token. \(\phi'(q)\) is guaranteed to share the same position with the answer token \(\phi(\epsilon_j)\). 
    \item \emph{Unselected Noisy Token and Remapped Unselected Noisy Token}: Each token \(\phi(\epsilon_k)\) has a probability of \(O(\frac{T}{N_n})\) to be drawn for datapoint that it is not the answer and a probaiblity of \(O \left( \frac{1}{N_n} \right)\) to share the same position with \(\phi'(\epsilon_j)\). Their remapped embedding \(\phi'(\epsilon_k)\) has a probaiblity of \(O \left( \frac{T}{N_n} \right)\) to be drawn for datapoint that \(\phi(\epsilon_k)\) is not the answer.
    \item \emph{Factual Token and Remapped Factual Token}: \(\phi(s)\) and \(\phi'(s)\) will not appear in the induction dataset.
\end{itemize}

We will examine the signal of each token after the gradient steps.

\textbf{In the first step,} since we initialize both weight matrices to be zero, we have 

\begin{equation}
    \sigma_j = \frac{1}{T} \ \  \forall j \text{  and  } 
    \beta_k = \begin{cases}
        -\frac{1}{N} &\text{ if } k \neq y \\ 
        \frac{N - 1}{N} &\text{ if } k = y 
    \end{cases}  \text{  and  } -\nabla_{\Wkq} \Lc = 0
\end{equation}

This means we are essentially only optimizing \(- \nabla_{\Wov} \Lc\). For each datapoint 
in the factual recall dataset, suppose the factual token and its answer are \(\left( s, y \right)\), we have that 

\begin{align}
    \frac{\eta}{n} \E \left[  \mu \left( y \right)\T  \left(-\nabla_{\Wov} \Lc \right) \phi(s) \right] &= O  \left( \frac{\eta}{n} \cdot \beta_y \cdot \sigma_j \right)
    = O \left( \frac{\eta}{n T} \right) \\ 
    \frac{\eta}{n} \E \left[  \mu \left( y \right)\T  \left(-\nabla_{\Wov} \Lc \right) \phi'(s) \right] 
    &= O \left( \frac{\eta}{n T} \right) \\ 
    \frac{\eta}{n} \E \left[  \mu \left( y \right)\T  \left(-\nabla_{\Wov} \Lc \right) \phi(\epsilon_k) \right] 
    &= O \left( \frac{\eta}{n T} \cdot \frac{T}{N_n}\right) - \underbrace{O \left( \frac{1}{N} \cdot \frac{1}{T} \cdot \frac{T}{N_n} \cdot \frac{\eta}{n} \cdot n_F \right)}_{\text{fact incorrect terms}}
     - \underbrace{O \left(  \frac{1}{N} \cdot \frac{1}{T} \cdot \frac{T}{N_n} \cdot \frac{\eta}{n} \cdot n_I \right)}_{\text{induction set}} \\
    &= O \left( \frac{\eta}{n N_n} \right) - O \left( \frac{\eta n_F}{N N_n n} \right) - O \left( \frac{\eta n_I}{N N_n n} \right) \\ 
    &= O \left( \frac{\eta}{n N_n} \right) - O \left( \frac{\eta}{N N_n} \right) \\
    \frac{\eta}{n} \E \left[  \mu \left( y \right)\T  \left(-\nabla_{\Wov} \Lc \right) \phi'(\epsilon_k) \right] 
    &= O \left( \frac{\eta}{n N_n} \right) - O \left( \frac{\eta}{N N_n} \right) \\ 
    \frac{\eta}{n} \E \left[  \mu \left( y \right)\T  \left(-\nabla_{\Wov} \Lc \right) \phi(q) \right] 
    &= O \left( \frac{\eta}{n T} \right) - \underbrace{O \left( \frac{\eta}{N T}  \right)}_{\text{fact incorrect terms}} - \underbrace{O \left(\frac{2\eta}{N T}\right)}_{\text{induction set}} 
\end{align}

where we can see that the most signal is absored in \(\phi(s)\) with spurious correlations learned with \(\phi'(s)\). 
The \(\Wov\) could act as associative-memory module for the factual recall dataset essentially in single gradient step.  
For other umembedding vector other than \(\mu(y)\) to dot product with \((-\nabla_{\Wov} \Lc) \phi(\cdot)\), we remark that 
the gradient update from the factual recall dataset gives a negative value.

Take an arbitrary point in the induction dataset, suppose the selected answer token is \(\epsilon_j\), we have 

\begin{align}
    \frac{\eta}{n} \E \left[ \mu \left( \epsilon_j \right)\T \left( -\nabla_{\Wov} \Lc \right) \phi \left( \epsilon_j \right) \right] 
    &= O \left( \frac{\eta}{n T} \right) - \underbrace{ O \left( \frac{\eta}{N N_n} \right)}_{\text{fact and induction}} \\ 
    \frac{\eta}{n} \E \left[ \mu \left( \epsilon_j \right)\T \left( -\nabla_{\Wov} \Lc \right) \phi' \left( \epsilon_j \right) \right] 
    &= O \left( \frac{\eta}{n T} \right) - \underbrace{ O \left( \frac{\eta}{N N_n} \right)}_{\text{fact and induction}} \\ 
    \frac{\eta}{n} \E \left[ \mu \left( \epsilon_j \right)\T \left( -\nabla_{\Wov} \Lc \right) \phi' \left( q \right) \right] 
    &= O \left( \frac{\eta}{n T} \right) - \underbrace{ O \left( \frac{\eta n_I}{N T n} \right)}_{\text{induction only}} \\ 
    \frac{\eta}{n} \E \left[ \mu \left( \epsilon_j \right)\T \left( -\nabla_{\Wov} \Lc \right) \phi \left( q \right) \right] 
    &= O \left( \frac{\eta}{n T} \right) - \underbrace{O \left( \frac{\eta}{N T} \right)}_{\text{fact and induction}} \\ 
    \frac{\eta}{n} \E \left[ \mu \left( \epsilon_j \right)\T \left( -\nabla_{\Wov} \Lc \right) \phi \left( \epsilon_k \right) \right] 
    &= O \left( \frac{\eta}{n N} \right) - \underbrace{O \left( \frac{\eta}{N N_n} \right)}_{\text{fact and induction}}\\ 
    \frac{\eta}{n} \E \left[ \mu \left( \epsilon_j \right)\T \left( -\nabla_{\Wov} \Lc \right) \phi' \left( \epsilon_k \right) \right] 
    &= O \left( \frac{\eta}{n N} \right) - \underbrace{O \left( \frac{\eta}{N N_n} \right)}_{\text{fact and induction}}
\end{align}

where we can see thaet the \(\Wov\) terms learns the correct association between each \(\mu(\epsilon_j)\) and \(\phi(\epsilon)\), 
with spurious correlation learned with \(\phi'(\epsilon_j)\). We further remark that this \(\Wov\) alone is able to make 
perfect predictions when loss is still high. However, as training progresses, the benign signal from \(\nabla_{\Wov} \Lc\) could also enable 
\(\Wkq\) to focus on the critical tokens. 

Nowe we focus on the \textbf{second gradient step}. Since now \(\Wkq\) is still a zero matrix, we have 

\begin{equation}
    \sigma_j = \frac{1}{T} \ \ \forall \ j 
\end{equation}

However, for now \(\beta_k\) doesn't have an order \(O(N)\) difference for \(k = y\) and \(k \neq y\). Here the relative update signal for \(\nabla_{\Wov} \Lc\) still follows from the analysis in the first step, where 
the relative update of the correct signal still dominates, but by a smaller margin. With a sufficiently large 
\(\eta\) in the second step, the training could converge. Now we focus on how the second step leads to the desired 
form of \(\Wkq\). 

For the induction task, we show that the model will concentrate on the correct term \(\phi'(q) + \phi(\epsilon_j)\). 
Let's recall the gradient with respect to \(\Wkq\): 

\begin{align}
    - \nabla_{\Wkq} \Lc 
    &= \sum_{i=1}^{T}  \gamma_i x_i x_T^T \tag*{\(\left( \sum_{k=1}^{N} \beta_k  \mu_k\T \Wov x_i  \right) = \gamma_i\)} 
    \label{equa:kq_update}
\end{align}

There are mainly ``six types of inputs'' in a single datapoint with selected answer token \(\epsilon_j\): 
(1) desired focused term \(\phi(\epsilon_j) + \phi'(q)\), (2) first occurrence of question \(\phi(q) + \phi'(\epsilon_{j-2})\), 
(3) last position \(\phi(q)\), (4) first position \(\phi(\epsilon_1)\), (5) remapped answer token with unrelated 
noise \(\phi(\epsilon_{j+1}) + \phi'(\epsilon_j)\), and (6) purely unrelated noise tokens \(\phi(\epsilon_k) + \phi'(\epsilon_{k-1})\).

We claim that 

\begin{equation}
    \E \left[ \gamma_j \right] > \E \left[ \gamma_k \right]
\end{equation}

where \(j\) is the coefficient for the desired term (1) \(\phi(\epsilon_j) + \phi'(q)\) and \(k\) is any other types of 
terms. We can decompose  

\begin{equation}
    \gamma_i = \underbrace{\sum_{k\neq y}^{N} \beta_k \mu (k)\T \Wov x_i}_{\text{Small}} + \underbrace{\left( \beta_y \mu (y)\T \Wov x_i \right)}_{\text{Large}}
\end{equation}

where now the subscript \(y\) refers to the token \(\epsilon_j\). We remark that the second term donimates the signal. 
From the analysis of the first gradient step, we know that 

\[ \E \left[ \mu (y)\T \Wov \phi'(q) \right] > \E \left[ \mu(y)\T \Wov \phi(\epsilon_j) \right] 
= \E \left[ \mu(y)\T \Wov \phi'(\epsilon_j) \right] > \E \left[ \mu(y)\T \Wov (\cdot) \right] \]

where \((\cdot)\) represents other terms (\emph{i.e.}, \(\phi(q), \phi(\epsilon_k), \phi'(\epsilon_k)\)). 
This means the term \(\phi(\epsilon_j) + \phi'(q)\) has the largest signal (\(\gamma_j\)) in expectation. To see this, 
as we know \(\beta_y > 0, \beta_k < 0\), consider substitute \(\phi'(q)\) with any other terms (e.g. \(\phi(q), 
\phi(\epsilon_k)\)), then \(\gamma_j\) is guaranteed to decrease. The same reasoning applies to \(\phi(\epsilon_j)\) 
as we fix \(\phi'(q)\). The only exception occurs with \(\phi'(\epsilon_j)\), but we know that this term is guaranteed 
to not share the same position with \(\phi'(q)\). Therefore, we finish our claim. 

A similar statement can be made for the factual recall task where \(\Wkq\) concentrates on the 
\(\phi(s) + \phi'(\epsilon_{i-1})\) and \(\phi'(s) + \phi(\epsilon_{i+1})\) terms. The second term could be 
regarded as ``benign spurious correlation'' under our setup. 
We can take a sufficiently large \(\eta\) in the second step to enable the convergence in expectation.  
As such, the \(\Wkq\) also takes in the form of \Cref{equa:appen_optimal_construction}.

\end{proof}

\begin{corollary}[Knowledge Conflict]
    Under the knowledge conflict inference setting, the model capable of solving both \emph{factual recall} and 
    \emph{induction} from \Cref{prop:two_layer_task} may output either the inductive token or the 
    factual token. More specifically, 
    if \(\exp(C_1) C_3 < \exp(C_2) C_4\), then the model outputs the factual recall answer \(\Gc^*(s)\); otherwise, the model 
    outputs the induction answer \(\epsilon_{j}\). 
\end{corollary}

\begin{proof}
The attention weight on the \(\phi'(q) + \phi(\epsilon_{j})\) is approximately \(\frac{\exp \left( C_1 \right)}{\exp(C_1) + \exp (C_2) + \left( T - 2 \right)}\);
The attention weight on the \(\phi(s) + \phi'(\epsilon_{i-1})\) is approximately \(\frac{\exp \left( C_2 \right)}{\exp(C_1) + \exp (C_2) + \left( T - 2 \right)}\). 

The raw logit value of \(\epsilon_j\) is \(C_3\frac{\exp \left( C_1 \right)}{\exp(C_1) + \exp (C_2) + \left( T - 2 \right)}\) and the raw 
logit value of \(\Gc^*(s)\) is \(C_4\frac{\exp \left( C_2 \right)}{\exp(C_1) + \exp (C_2) + \left( T - 2 \right)}\). Other terms 
have a small logit values. Therefore, if \(\exp(C_1) C_3 < \exp(C_2) C_4\), then the model outputs the factual recall answer \(\Gc^*(s)\). 
Otherwise, the model outputs the induction answer \(\epsilon_{j}\) 
\end{proof}

\begin{proposition}[Effectiveness of \jrt]
    Consider the model from \Cref{prop:two_layer_task} and the case when its \emph{inductive} part dominates (\emph{i.e.,} \(\exp(C_1)C_3 >> \exp(C_2) C_4\)), then 
    the intervention by \jro/PH3 of deleting the two attention heads is not as effective as \jrt. 
    In particular, in this case \jro/PH3 does not result in the parametric answer, while \jrt does. 
\end{proposition}

\begin{proof}
First, we remark that both attention heads (of the construction) are ``highly influential'' attention heads. 
As if one scales up or down the activation output of the two heads, the logit value of the corresponding parametric answer 
decreases or increases monotonically. 

We now choose the intervention method to be knocking out for simplicity (which is exactly PH3; for \jrt, \jro, it means adds 
a scaled version of the activation output by a factor of -1). 

If we were using a single-pass intervention method advocated by PH3 or \jro, then this simply means we delete the activation 
output from both heads, which gives an answer of random guessing among al elements in the vocabulary space \(\Vc\). 

If we use the dual-run design of \jrt, then we note that the activation outputs of the second layer from the first run has that 

\begin{equation}
    \text{Logit}_{fact}^{(1)} = \frac{C_4 \exp (C_2)}{\exp(C_1) + \exp(C_2) + (T - 2)} \ \ \ \ \ \ 
    \text{Logit}_{ind}^{(1)} = \frac{C_3 \exp (C_1)}{\exp(C_1) + \exp(C_2) + (T - 2)}
\end{equation}

In the second run we have 

\begin{equation}
    \text{Logit}_{fact}^{(2)} = \frac{C_4 \exp(C_2)}{\exp(C_2) + (T-1)} \ \ \ \ \ \ 
    \text{Logit}_{ind}^{(2)} = \frac{C_3}{\exp(C_2) + (T - 1)}
\end{equation}

By deleting the activation output from the first run, we have

\begin{equation}
    \text{Logit}_{fact}^{(2)*} > 0 \ \ \ \ \ \ 
    \text{Logit}_{ind}^{(2)*} < 0 \ \ \ \ \ \   
    \text{Logit}_{other}^{(2)*} \approx 0 
\end{equation}

This shows that \jrt results in the correct parametric answer.
\end{proof}


\end{document}